\documentclass[10pt,a4paper,oneside]{article}
\usepackage{amsmath,amssymb,amsthm,mathtools}
\usepackage{geometry}
\usepackage{graphicx}
\usepackage{bm}
\usepackage{float}      % allows [H] placement
\usepackage[section]{placeins}  % keeps floats within their section
\usepackage{caption}

\title{\textbf{Theoretical Compression Bounds for Wide Multilayer Perceptrons}}
\author{%
  Houssam El Cheairi\thanks{MIT. Email: houssamc@mit.edu} \and
  David Gamarnik\thanks{MIT. Email: gamarnik@mit.edu} \and
  Rahul Mazumder\thanks{MIT. Email: rahulmaz@mit.edu}
}
\date{\today}

\newtheorem{lemma}{Lemma}
\newtheorem{proposition}{Proposition}
\newtheorem{definition}{Definition}
\newtheorem{assumption}{Assumption}
\newtheorem{problem}{Problem}
\newtheorem{remark}{Remark}
\newtheorem{corollary}{Corollary}

\newcommand{\E}{\mathbb{E}}
\newcommand{\R}{\mathbb{R}}

\usepackage{amsmath}
\usepackage{mathtools}
\usepackage{amsthm}
\usepackage{algorithm}
\usepackage{algorithmic}
\usepackage{xcolor}         % colors
\usepackage[utf8]{inputenc} % allow utf-8 input
\usepackage[T1]{fontenc}    % use 8-bit T1 fonts
\usepackage{url}            % simple URL typesetting
\usepackage{booktabs}       % professional-quality tables
\usepackage{amsfonts}       % blackboard math symbols
\usepackage{nicefrac}       % compact symbols for 1/2, etc.
\usepackage{microtype}      % microtypography

\usepackage{xcolor} % for colors
\usepackage[
    colorlinks=true,
    linkcolor=blue,   % section links
    urlcolor=blue,     % url links
    citecolor=blue     % <--- citations in blue
]{hyperref}

\def \R{ \mathbb{R}}

\def \A{ \mathcal{A}}

\def \T{ \mathcal{T}}

\def \EE{\mathcal{E}}

\def \Z{\mathbf{Z}}

\def \l{\left}
\def \r{\right}
\def \b{\mathbf}

\def \p{\mathbb{P}}
\def \E{\mathbb{E}}
\def \Sm{\mathcal{S}}

\numberwithin{equation}{section}

\begin{document}
\maketitle

\begin{abstract}

Pruning and quantization techniques have been broadly successful in reducing the number of parameters needed for large neural networks, yet theoretical justification for their empirical success falls short. We consider a randomized greedy compression algorithm for pruning and quantization post-training and use it to rigorously show the existence of pruned/quantized subnetworks of multilayer perceptrons (MLPs) with competitive performance. We further extend our results to structured pruning of MLPs and convolutional neural networks (CNNs), thus providing a unified analysis of pruning in wide networks. Our results are free of data assumptions, and showcase a tradeoff between compressibility and  network width. The algorithm we consider bears some similarities with \textit{Optimal Brain Damage} (OBD) and can be viewed as a post-training randomized version of it. The theoretical results we derive bridge the gap between theory and application for pruning/quantization, and provide a justification for the empirical success of compression in wide multilayer perceptrons.
\end{abstract}

%\section*{comments}
%\begin{itemize}
%    \itam check the deep proof for quantization and adjust the assumptions on dimensions ratios
%    \item the data lower bound corollary only stated in unstructured section
%    \item Do I need $k|\gcd$?
%    \item For two layers lemma change the random variables from $h$ to $t$.
%\end{itemize}

\tableofcontents

\section{Introduction}\label{chapter3:section:introduction}

Over the past decade, neural networks have achieved remarkable empirical success in various machine learning applications including computer vision, natural language processing, speech recognition, and image/text generation to name a few. However, this success comes at the cost of exceedingly large and overparametrized architectures, often consisting of billions of parameters. Moreover, training these large networks has proven to be extremely expensive, requiring extensive amounts of resources both in terms of data and computing power. Naturally, the latter issues limit the deployability of large networks in embedded systems as one is typically constrained by operational costs of storage, training, and inference.

Neural network compression has therefore emerged as an attempt to reduce the size of networks without compromising their accuracy. Fundamentally, compression hinges on the idea that overparametrized networks contain \textit{redundancies} that can be exploited/eliminated to construct \textit{lighter} networks with similar prediction performance.  Many different compression techniques have been put forward including pruning \cite{lecun1989optimal}, quantization \cite{han2015deep}, neural architecture search \cite{chen2021adabert}, knowledge distillation \cite{hinton2014distilling}, weight clustering \cite{ullrich2017soft}, token pruning \cite{wei2023joint}, and  more. Out of the aforementioned techniques, pruning has received much attention as it has  proven to work well in practice and can be exploited by hardware more easily in the case of structured pruning \cite{he2023structured}. In fact, even simple pruning routines such as magnitude-based pruning and random pruning have had some success in the literature for moderate sparsities \cite{han2015deep,liuunreasonable, gadhikar2023random}. Nonetheless, theoretical justification for the success of these heuristic algorithms remains somewhat elusive. Compression comes in different settings that can be roughly categorized as follows:
\begin{enumerate}
    \item \textit{Compression at initialization}. The dense network is compressed before training in the hope of obtaining a sparse  network that can achieve a competitive loss once trained. Such sparse networks are  dubbed \textit{lottery tickets} \cite{frankle2018lottery, frankle2020pruning, kumar2024no}, as they use the very limited information available at initialization to achieve compression. 
    \item  \textit{Compression concurrently with training}. Typically, consists of a sequence of compression/training steps, whereby some of the weights are compressed then the network (or only a subnetwork) is trained further  \cite{lecun1989optimal, lym2019prunetrain, lasby2023dynamic}. The process is repeated until the desired compression rate and convergence  are achieved. This approach often leads to the best compression rates in practice at a computational training cost. 
    \item \textit{Compression post-training}. A dense network is first trained then compressed \cite{hassibi1992second, lazarevich2021post, kwon2022fast, zhang2024plug}, typically using simple heuristics such as magnitude-based pruning. 
\end{enumerate}

In this paper, we solely focus on the third setting for \textit{pruning} and \textit{quantization}. The practical motivations of  post-training compression are twofold: First, post-training compression eliminates the computational cost of further training. Second, many of the practically used networks are based on concatenating a dense pretrained model with a smaller MLP, before freezing the former and training the latter to learn specific data. Hence, replacing the dense pretrained models by compressed equivalents with minimal accuracy alteration would improve the efficiency of such pipelines in terms of storage, training cost and speed of inference.

Pruning is typically achieved by setting some of the network weights  to zero, and possibly adjusting  other weights. Moreover, we distinguish between unstructured and structured pruning. In the former, the pruned weights' positions need not follow any specific pattern within their layer. In contrast, structured pruning removes entire network structures, such as neurons for MLPs and filters/channels for CNNs. The weights set to zero are then ignored both at inference and training. On the other hand, quantization is achieved by replacing the weights with a discrete approximation from a predetermined finite set of  possible values. Namely, each weight $\b{W}_{ij}$ is replaced by a value $\hat{\b{W}}_{ij} \in \{v_1,\ldots,v_k\}$ where $v_1,\ldots,v_k$ are scalars, and typically $k \ll {\rm dim}(\b{W})$. Throughout this paper, we use \textit{compression} to  refer to both pruning and quantization interchangeably, and it will be clear from context whether we are only interested in pruning or quantization.

The compression procedure we consider borrows ideas from the \textit{Optimal Brain Damage} (OBD) algorithm \cite{lecun1989optimal}, and can be thought of as a randomized post-training iterative variant of OBD. Namely, our theoretical analysis hinges on a second order Taylor approximation of loss. The randomization we introduce leads to a random sparse network, as opposed to a deterministic sparse network. Crucially, we are able to derive bounds on the loss of the random sparse network on average, which in turn proves the existence of a nonrandom sparse network with the same guarantees.

The framework we use to analyze the performance of our algorithm is based on the Lindeberg interpolation method \cite{chatterjee2006generalization}. Informally, this technique allows us to bound the average variations in an MLP's loss upon perturbing one of its weights with a random variable. When the perturbation has zero mean, the first order term in the perturbed loss is zero, which leaves us with second order terms. Hence, we can use the second order terms as proxies for the \textit{importance} of weights. That is, we can select the weight whose perturbation leads to the \textit{least} loss variation, and switch it with its perturbed equivalent. The use of the interpolation technique provides two key benefits for our analysis: First, it allows us to iteratively control the discrepancy in loss following a single weight perturbation, as opposed to a fully perturbed weight layer. Second, this technique casts the problem of bounding the variations of the loss analytically as a simple calculus problem involving a one variable scalar function. Furthermore, we demonstrate the strength and versatility of this technique by extending our results to structured pruning of MLPs and CNNs with little effort.

While our results are theoretical in nature, we conduct several numerical simulations to validate the  content of our findings empirically. Namely, we train MLPs and CNNs on regression and classification tasks, then use our pruning approach to obtain compressed networks post-training. We vary the width of the initial networks and show that the pruning error decays with width for a fixed compression rate, which is consistent with our theoretical  results

%\section{Related Works}\label{chapter3:section:related}

We next summarize relevant results from the theoretical literature on compression bounds and compare them to our results. Early research on pruning dates back to the seminal papers of LeCun \cite{lecun1989optimal} and Hassibi \cite{hassibi1992second} which introduced second order pruning routines dubbed \textit{Optimal Brain Damage} (OBD) and \textit{Optimal Brain Surgeon} (OBS) respectively. These pruning algorithms are based on iteratively zeroing the weight that leads to the smallest increase in the loss,  whereby the loss is approximated by its second order Taylor series. Many papers have iterated over the main ideas behind OBD/OBS \cite{frantar2022optimal, yu2022combinatorial, kurtic2022optimal, benbaki2023fast}. Some simple error bounds on OBS have been given in \cite{dong2017learning}, though the bounds scale linearly with the depth, and involve products of operator norms of the pruned weight matrices, which could scale exponentially with dimension (see discussion following Assumption \ref{chapter3:assumption:1}). Our pruning approach uses the second order approximation of loss as well, but in a randomized fashion: the selected weight is replaced by a discrete random variable. The introduction of randomization is crucial for our analysis, as it cancels out the first-order term in the Taylor expansion of the loss variation. Whereas, the first order term in OBS is neglected as the network's loss is assumed to be minimal in its current state, which is achieved by adding training steps after each  compression step. 

Although the greedy aspect of our algorithm is not novel, and random pruning methods such as Bernoulli masks have been proposed before, our specific combination of greedy randomization with a novel importance score is new. Crucially, our theoretical analysis introduces the use of  the Lindeberg interpolation to eliminate first-order contributions entirely. To our knowledge, this is the first application of the Lindeberg interpolation technique in the pruning or quantization literature. In contrast, the standard analysis of OBD removes the first-order terms by assuming minimal loss (zero gradient), thus requiring retraining the model after each pruning step. In contrast, we do not train the model further, nor make any null gradient assumption. While the ideas behind our approach (randomization, interpolation) are elementary, we are, to the best of our knowledge, the first to demonstrate their successful application in pruning/quantization, and  consequently prove that pruning at linear sparsities is achievable without any data assumptions. 

A recent analysis for random and magnitude-based pruning is given in \cite{qian2021probabilistic}. The latter work is perhaps the most relevant for this paper, as we borrow aspects of its setting, including some  assumptions on the weights scale. However, the authors are only able to derive guarantees for sublinear pruning rates,  and thus do not explain the success of pruning at linear sparsities. Moreover, the error metric they use is a worst-case approximation error over the entire $\ell_2$ unit ball, i.e., they require the pruned network to uniformly approximate the original network well on the $\ell_2$  unit ball. The latter metric is rather restrictive, as the performance of neural networks  is typically evaluated by taking the average of a given loss over some distribution (e.g., training/validation data).  Furthermore, they require the weights to be independent random variables, while we do not, as it need not hold in practice. In fact, some compression algorithms exploit weight correlations \cite{kuznedelev2023cap}.

An algorithm based on quantizing neurons  deterministically in pre-trained networks with provable guarantees is given in \cite{lybrand2021greedy}, but the analysis is limited to one-layer networks with Gaussian input data and quantization alphabet $\{-1,0,1\}$. The latter results for one-layer networks were extended to mixtures of Gaussian data and more general alphabets in \cite{zhang2023post}. Another follow-up work \cite{zhang2024unified} by the latter authors extends a \textit{stochastic path following quantization} (SPFQ) method to pruning, and derives theoretical error bounds for one-layer networks that scale logarithmically with the layer's dimension. On the other hand, our theoretical bounds apply to deep networks of any depth, and do not require any distributional assumption on the input data. Furthermore, our algorithm is not restricted to act independently on neurons, unlike the aforementioned works.

Another line of theoretical analysis derives generalization bounds for compressed network based on the PAC-Bayes framework \cite{zhou2018non}. However, the authors rely on the existence of good compression algorithms, and thus do not show the feasibility of compression. Similar bounds are given in \cite{arora2018stronger} for compression by truncation of singular values.

Importance sampling pruning algorithms are derived in \cite{baykal2018data, liebenwein2019provable} for MLPs and CNNs, where sampling distributions based on \textit{sensitivity} scores are constructed over the parameters of the networks in order to retain important weights and discard redundant ones. However, the bounds given in \cite{baykal2018data, liebenwein2019provable} do not imply feasibility of pruning at linear sparsities, nor at any other prescribed level. The upper bound on the subnetwork parameters given in Theorem 4 of \cite{baykal2018data} can scale (in the worst case) larger than the size of the initial dense network. Indeed, the quantities $S_i^\ell$ are lower bounded by 1, and the quantities $\hat{\Delta}_k$ can scale as $\mathcal{O}({\rm width})$.  Hence, the term $\hat{\Delta}^{\ell \to}$ can scale as $\mathcal{O}(({\rm width})^{{\rm depth}-\ell+1})$. Combining these leads to a very crude bound of $\mathcal{O}(({\rm depth})^3 \times ({\rm width})^{2{\rm depth}-1})$ parameters, much larger than the dense network’s $\mathcal{O}({\rm depth} \times ({\rm width})^2)$ number of parameters. The authors of \cite{baykal2018data} do not clarify the scaling of the key quantities $\hat{\Delta}^{\ell\to}, S_i^\ell$, nor do they derive comprehensive upper bounds on them based on network parameters (width/depth/weights norm, etc), or claim a specific achievable pruning regime (linear/sublinear). In general, \cite{baykal2018data} does not show the feasibility of pruning at any comparable regime.

While most pruning algorithms are backward in the sense that they start from a dense network and gradually remove weights, a forward approach with provable asymptotic guarantees has been considered in \cite{ye2020good}. However, the provided theoretical analysis is limited to networks with two layers and does not directly show the interplay between sparsity level and degree of overparametrization.

Recently, the Lottery Ticket Hypothesis (LTH) has been put forward by \cite{frankle2018lottery} conjecturing that a trained network contains a sparse subnetwork capable of matching the original network's accuracy when trained from scratch. A stronger conjecture was subsequently made in \cite{ramanujan2020s}, and was later proven in \cite{malach2020proving} for multilayer perceptrons with ReLU activations, and generalized in \cite{orseau2020logarithmic, pensia2020optimal}. However, finding these sparse subnetworks has proven to be difficult \cite{frankle2020pruning}. Recently, a theoretical explanation for this difficulty based on the \textit{Law of Robustness} was put forward in \cite{kumar2024no}. However, these works deal with compression at initialization, while our work focuses on post-training compression, a fundamentally different setting. Namely, SLTH does not account for the empirical success of compression after training. In contrast, we provide the first rigorous proof that post-training pruning is feasible at linear sparsity levels.

Furthermore, the theoretical literature on SLTH is limited, and the most general statements of SLTH remain conjectural. Namely, showing that one can find sparse subnetworks  of dense networks  at initialization with good approximation error is an open problem. Thus, the general statement of SLTH cannot be compared to our results. Works such as \cite{malach2020proving}  prove slightly different variants of STLH, and we expand on this next. Indeed, the notion of sparsity we adopt is derived from the dense starting network $\Phi$. In particular, we show that sparse subnetworks $\hat{\Phi}$ of $\Phi$ exist, representing the same initial network $\Phi$. On the other hand, SLTH results, as presented in e.g., \cite{malach2020proving}, show that given a target network $F$ and a sufficiently overparametrized random network $G$, there exists a subnetwork $\hat{G}$ of $G$ approximating $F$ without any further training. Crucially, this subnetwork $\hat{G}$ has a similar number (if not more) of parameters to $F$. Moreover, the network $G$ is a polynomial order of magnitude larger than $F$. This result does not imply that we can prune the target network 
 at linear sparsity or even sublinear sparsity, but instead that if we have access to a very dense and much larger random network $G$, we can find a subnetwork $\hat{G}$ similar to the target $F$, with similar number of parameters to $F$. That is, the sparsity does not relate to the target network. Whereas in our case, sparsity is defined explicitly with respect to the target network $F=\Phi$, and our matching notation for 
 $\hat{G}$ is $\hat{\Phi}$. Similar works to \cite{malach2020proving} include \cite{orseau2020logarithmic, pensia2020optimal} and show similar results. Additionally, theoretical results quantifying the achievable sparsity regime from SLTH are lacking. The most relevant work we could find is \cite{natale2024sparsity}, whereby the authors derive sparsity bounds between 
 $\hat{G}$ and $G$, but these results cannot be compared with ours, as we consider the sparsity between $F$ and $\hat{G}$.

We end this section with a brief overview on the organization of this paper. The next section introduces key notation used in the remainder of this work. Our theoretical results are grouped into three sections: unstructured compression of MLPs is presented in section~\ref{chapter3:section:unstructured.fcnn}, structured pruning of MLPs  in section~\ref{chapter3:section:structured.fcnn}, and structured pruning of CNNs in section~\ref{chapter3:section:structured.cnn}. Each of these sections provides a formal definition of the compression setting at hand and states our main compression bound. Our empirical results are presented in Section~\ref{chapter3:section:numerical.simulation}. All proofs are deferred to Sections~\ref{chapter3:appendix:unstructured.shallow} and~\ref{chapter3:appendix:structured.shallow}. In each of these proof sections, we first derive compression bounds for shallow MLPs (one- and two-layer perceptrons), and then extend the arguments to general MLPs. Finally, several necessary auxiliary lemmas are established in Section~\ref{chapter3:appendix:preliminary.results}.

\section{Notation}\label{chapter3:section:notations}

We introduce notation that will be used in the remainder of this paper. We use standard big-O notation $\mathcal{O}(.)$ to hide explicit constants. For $n\in \mathbb{Z}_{\geq 0}$, we use $[0, n]$ and  $[n]$ to denote the sets $\{0,\ldots, n\}$ and $\{1,\ldots, n\}$ respectively. We use bold letters to denote matrices, tensors and vectors.  Given two matrices $\b{A}, \b{B}$ of same dimensions, we let $\b{C} = \b{A} \odot \b{B}$ be the matrix with entries $\b{C}_{ij} = \b{A}_{ij}\b{B}_{ij}$. We denote by $\b{A}_{i,:}, \b{A}_{:,j}$ the $i$-th row and $j$-th column of $\b{A}$ respectively. We use $\|\cdot\|$ to denote the operator norm for matrices, $\|\cdot\|_q$ to denote standard $\ell_q$ norms for $q\in \mathbb{R}_{\geq 0}$, and $\|\cdot\|_{\infty}$ to denote the infinity norm. Moreover, for any function $\varphi : \mathbb{R}^n \to \mathbb{R}^m$ we let $\|\varphi\|_{\rm Lip} = \sup_{\b{x}\neq \b{y} \in \R^n} \|\varphi(\b{x}) - \varphi(\b{y})\|/\|\b{x} - \b{y}\|$. We introduce for $q\in \mathbb{R}_{\geq 0}$
\begin{align*}
    \mathbb{B}^k_{q}(\rho) \triangleq \l\{\b{x}\in \R^k \mid \|\b{x}\|_q \leq \rho\r\}.
\end{align*}
We use the  notation  $\b{Z}(t; i, j)$ where $\b{Z}$ is an  $n\times m$ matrix and $(i,j)\in [n]\times [m]$ to denote the matrix given by
$$
\b{Z}(t; i, j)_{\ell_1 \ell_2} =
\begin{cases}
\b{Z}_{\ell_1 \ell_2}, & (\ell_1,\ell_2) \neq (i,j),\\
t, & (\ell_1,\ell_2) = (i,j).
\end{cases}
$$
Given a constant $\kappa>0$ and a vector $\b{v}\in \R^d$, we denote by $[\b{v}]_{\kappa}$ the projection of $\b{v}$ into the Euclidean ball $\mathbb{B}^d_2(\kappa)$ of radius $\kappa$, i.e $[\b{v}]_{\kappa} = \arg\min_{\b{z}\in \mathbb{B}^d_2(\kappa)}\|\b{z}-\b{v}\|$. Given a function $f$ with input $\b{x}$ and output in $\R^n$, we let $f(\b{x}) = (f_1(\b{x}),\dots, f_{n}(\b{x}))$ be its coordinate functions. Given two reals $a,b \in \R$ we use the notation $a \vee b = \max(a, b)$ and $a \wedge b = \min(a, b)$. Given a random variable $X$ with distribution $\mathcal{D}$, we denote $\E_{X}$ the expectation over $X$. Similarly, given a collection $\mathcal{C}=\{X_1,\ldots,X_n\}$ of random variables, we denote $\E_{\mathcal{C}}$ the expectation over the joint distribution of $(X_1,\dots,X_n)$. We use $\b{Ber}^{n,m}_p$ to denote an $n$ by $m$ matrix with independent ${\rm Bernoulli}(p)$ entries. For a random vector $\b{x}$, we denote by ${\rm Var}(\b{x})$ the covariance matrix of $\b{x}$. Given a matrix $\b{A}\in \R^{n\times m}$ and $p\in (0,1)$, we say that $\b{A}$ is $p$-sparse if $\|\b{A}\|_0 \leq p nm$. Furthermore, we say that $\b{A}$ is $k$-discrete where $k\in \mathbb{Z}_{\geq 0}$ if the entries of $\b{A}$  take at most $k$ distinct values.  Given a function $f$ depending on $L$ total parameters, we say that $f$ is $p$-sparse ($k$-discrete) if at most $pL$ of those parameters are nonzero (all parameters take at most $k$ distinct values). 

Given a function $\varphi: \R^n \to \R^n$, we say that $\varphi$ is  entrywise  if it is applied entrywise and use the abused notation $\varphi(\b{z}) = (\varphi(\b{z}_1),\dots \varphi(\b{z}_n))$ for $\b{z}\in \R^n$. Furthermore, we denote by $\varphi^{(j)}, j\in \mathbb{Z}_{\geq 0}$ its derivatives. For $ M\in \mathbb{R}_{>0}, w \in [-M, M]$ and $k\in \mathbb{Z}_{\geq 1}$ we denote by $q(w; M, k)$ the random variable given by
$$
q(w;M,k) = 
\begin{cases} 
\text{sign}(w)\frac{\ell_w M}{k} & \text{with probability } 1 - \ell_w + \frac{k|w|}{M}, \\
\text{sign}(w)\frac{(\ell_w-1)M}{k} & \text{with probability } \ell_w - \frac{k|w|}{M},
\end{cases}
$$
where    $\ell_w = \min \l\{ \ell \mid \ell \in [k], |w| \leq \frac{\ell M}{k} \r\}$, and ${\rm sign}(w) = 1_{w\geq 0} - 1_{w<0}$. Finally, we introduce the following class of gate matrices.
\begin{definition}[Gate Matrix]
Let $n\in \mathbb{Z}_{>0}$ and $S\subset [n]$. Introduce $\b{P}_S = \sum_{j \in S} \b{e}_j \b{e}_j^\top$, where $(\b{e}_i)_{i\in[n]}$ are the canonical basis vectors of $\R^n$. For $z\in \mathbb{R}$, the gate matrix $\b{G}(z; S)$ is given by
\begin{align*}
    \b{G}(z; S) = \b{I}_n + (z-1) \b{P}_S.
\end{align*}
When $S={k}$, we also use the notation $\b{G}(z; k)$.
\end{definition}

\section{Unstructured Compression of Multilayer Perceptrons}\label{chapter3:section:unstructured.fcnn}

\subsection{Problem Formulation}\label{chapter3:section:problem.formulation.1}

We begin by formalizing the class of MLPs and the compression questions of interest. Let $m\in \mathbb{Z}_{\geq 1}$ and introduce the following function representing an $m$-layer MLP
\begin{align}
    \Phi: \quad \R^{n_{1}} \to \R^{n_{m+1}}, \quad \b{x} \mapsto \varphi_{m}\l(\b{W}_m \varphi_{m-1}\l(\dots \varphi_{1}\l(\b{W}_1 \b{x}\r)\r)\r),  \label{chapter3:fcnn}
\end{align}
where $\b{W_\ell} \in \R^{n_{\ell+1} \times n_{\ell}}$ are weight matrices, and $\varphi_\ell: \R^{n_{\ell+1}} \to \R^{n_{\ell+1}}$ are  activation functions. Moreover, for $\ell\in [m]$, we let $\Phi^\ell$ be the subnetwork of $\Phi$ at depth $\ell$, that is
\begin{align}
    \Phi^\ell: \quad \R^{n_{1}} \to \R^{n_{\ell+1}}, \quad \b{x} \mapsto \varphi_{\ell}\l(\b{W}_\ell \varphi_{\ell-1}\l(\dots \varphi_{1}\l(\b{W}_1 \b{x}\r)\r)\r).  \label{chapter3:fcnn.subnetwork}
\end{align}
Let $\mathcal{D}$ be a joint distribution of  data $(\b{x}, \b{y}) \in \mathbb{B}_{2}^{n_1}(1) \times \mathbb{R}^{n_{m+1}}$, and introduce
\begin{align}
    \mathcal{L}(\Phi; \mathcal{D}) \triangleq \E_{(\b{x}, \b{y}) \sim \mathcal{D}}\l[\l\| \Phi(\b{x}) - \b{y}\r\|^2\r].
\end{align}
Namely $\mathcal{L}(\Phi;\mathcal{D})$ is the $\ell_2$-squared expected loss with data distribution $\mathcal{D}$. A canonical example is given by taking $\mathcal{D}$ to be the uniform distribution over a finite set of training data, i.e., $\mathcal{D} = \frac{1}{N} \sum_{i\in [N]} \delta_{(\b{x}_i, y_i)}$. While we assume the covariates $\b{x}$ have norm at most $1$ for simplicity, our results readily extend to any distribution $\mathcal{D}$, as long as the covariates $\b{x}$ are bounded in  norm.  We are interested in the following two compression feasibility problems.
\begin{problem}\label{chapter3:problem:1}
    Given $p\in (0, 1)$, an error threshold $\varepsilon>0$, a data distribution $\mathcal{D}$ and a network $\Phi$, is there a $p$-sparse subnetwork $\hat{\Phi}$ of $\Phi$ such that
    \begin{align}
        \mathcal{L}(\hat{\Phi}; \mathcal{D}) \leq \mathcal{L}(\Phi; \mathcal{D}) + \varepsilon.
    \end{align}
\end{problem}
Namely, given a starting network $\Phi$ and a target sparsity level $p$, we want to prune the network $\Phi$ into a $p$-sparse network $\hat{\Phi}$ with minimal effect on the  $\ell_2$-squared loss  on the data distribution $\mathcal{D}$. 
\begin{problem}\label{chapter3:problem:2}
    Given a quantization level $k\in \mathbb{Z}_{\geq 1}$, an error threshold $\varepsilon>0$, a data distribution $\mathcal{D}$ and a network $\Phi$, is there a $k$-discrete subnetwork $\hat{\Phi}$ of $\Phi$ such that
    \begin{align}
        \mathcal{L}(\hat{\Phi}; \mathcal{D}) \leq \mathcal{L}(\Phi; \mathcal{D}) + \varepsilon.
    \end{align}
\end{problem}
Similarly,  given a starting network $\Phi$ and a target quantization level $k$, we want to compress the network $\Phi$ into a $k$-discrete network $\hat{\Phi}$ with minimal effect on the  $\ell_2$-squared loss  on the data distribution $\mathcal{D}$.

Both Problem~\ref{chapter3:problem:1} and Problem~\ref{chapter3:problem:2} have two aspects. On the one hand, there is a computational question: can one design efficient algorithms that construct such compressed subnetworks? On the other hand, there is a feasibility question: under which structural conditions on $\Phi$ is compression at a given rate (sparsity $p$ or discretization level $k$) information-theoretically possible? In this paper, we are interested in the latter. That is, we aim at deriving network properties under which compression is achievable at the desired rate (sparsity $p$, discretization level $k$). Nonetheless, the ideas behind our results can be  repurposed to derive explicit compression algorithms, or heuristics thereof.

To address these feasibility questions, we introduce a randomized greedy compression algorithm in the next section. Its role is auxiliary: by analyzing its performance, we obtain sufficient conditions under which sparse or quantized subnetworks with small excess loss exist.

\subsection{Algorithm Description}\label{chapter3:section:algorithm.description.1}
We now describe the compression algorithm we use to derive our bounds. Before detailing the layer-wise greedy procedure, we explain how it is applied to MLPs. Our compression strategy operates on some of the network layers, and is described by two  sets $\mathcal{W} , \mathcal{B} \subset [m]$ satisfying
\begin{align}
    \mathcal{W} \cap \mathcal{B} = \emptyset, \quad \text{and} \quad (\mathcal{W} \cup \mathcal{B}) \cap (\mathcal{B} + 1) = \emptyset. \label{chapter3:set}
\end{align}
The choice of $\mathcal{W}, \mathcal{B}$ is left unspecified to keep our presentation general. Let $c\triangleq \max_{\ell \in [m]} \|\b{W}_\ell \|$. We apply our compression procedure top-bottom as follows:
\begin{enumerate}
    \item If we are currently at layer $\ell \not \in \mathcal{W} \cup \mathcal{B}$, we leave the weight matrix $\b{W}_{\ell}$ as is, i.e., we set the $\ell$-th layer's weight matrix in $\hat{\Phi}$ to be $\hat{\b{W}}_{\ell} = \b{W}_{\ell}$.
    \item If we are currently at layer $\ell  \in \mathcal{W}$, we compress $\b{W}_{\ell}$ into a matrix $\hat{\b{W}}_{\ell}$, and set the $\ell$-th layer's weight matrix in $\hat \Phi$ to be $\hat{\b{W}}_{\ell}$. Furthermore, we add a projection operation to the output of layer $\ell$ into the $\ell_2$ ball $\mathbb{B}_2^{n_{\ell+1}}(\kappa)$,  with radius $\kappa =  c^\ell  \geq \prod_{1\leq i \leq \ell} \|\b{W}_{i}\|$.
    \item If we are currently at layer $\ell  \in \mathcal{B}$, we compress $\b{W}_{\ell}$ into a matrix $\hat{\b{W}}_{\ell}$, and set the $\ell$-th ($(\ell+1)$-th resp.) layer's weight matrix in $\hat{\Phi}$ to be $\hat{\b{W}}_{\ell}$ ($\b{W}_{\ell+1}$ resp.). In particular, we do not apply any compression to $\b{W}_{\ell+1}$. Furthermore, we add a projection operation to the output of layer $\ell+1$ into the $\ell_2$ ball $\mathbb{B}_2^{n_{\ell+2}}(\kappa)$  with radius $\kappa = c^{\ell+1} \geq \prod_{1\leq i \leq \ell+1} \|\b{W}_{i}\|$.
\end{enumerate}
We next provide a pseudocode for the compression algorithm in the unstructured setting for the convenience of the reader.
\begin{algorithm}
\caption{Unstructured Compression Algorithm}
\label{chapter3:alg:pruning}
\begin{algorithmic}[1]
\REQUIRE Data distribution $\mathcal{D}$, network $\Phi$, sets $\mathcal{W}, \mathcal{B}\subset[m]$ satisfying (\ref{chapter3:set}), pruning parameter $p\in(0,1)$ (or quantization parameter $k\in \mathbb{Z}_{\geq 1}$), fraction parameter $\alpha \in (0, 1)$.
\STATE{Let $\hat{\Phi}=\Phi$, and $c\triangleq \max_{\ell \in [m]} \|\b{W}_\ell\|$.}
\FOR{Layer index $\ell = 1$ to $m$}
    \IF{$\ell \not \in \mathcal{W} \cup \mathcal{B}$}
        \STATE{Set $\hat{\b{W}}_{\ell} = \b{W}_{\ell}$}
    \ELSE
        \IF{$\ell \in \mathcal{W}$}
            \FOR{Step $i = 1$ to $\lfloor \alpha n_{\ell} n_{\ell+1} \rfloor $}
                \FORALL{uncompressed weights $[\b{W}_{\ell}]_{ij}$}
                    \STATE{Evaluate the score $|\E_{t_{ij}}[\mathcal{L}(\hat{\Phi}^{\ell}_{t_{ij}};\mathcal{D})]-\mathcal{L}(\hat{\Phi}^{\ell};\mathcal{D})|$}, with $t_{ij}=[\b{W}_{\ell}]_{ij} \times {\rm Bernoulli}(p)/p$ for pruning or $t_{ij} = q([\b{W}_\ell]_{ij}; \|\b{W}_{\ell}\|_{\infty}, k)$ for quantization.
                \ENDFOR
                \STATE{Let $(i,j)$ be the index of the weight with least score.}
                \STATE{Set $[\hat{\b{W}}_{\ell}]_{ij} = t_{ij}$ and mark $[\b{W}_\ell]_{ij}$ as compressed.}
            \ENDFOR
            \STATE{Add a projection onto $\mathbb{B}^{n_{\ell+1}}_2 \l(c^\ell\r)$ on the output of the $\ell$-th layer in $\hat{\Phi}$.}
        \ELSE
            \FOR{Step $i = 1$ to $\lfloor \alpha n_{\ell} n_{\ell+1} \rfloor $}
                \FORALL{uncompressed weights $[\b{W}_{\ell}]_{ij}$}
                    \STATE{Evaluate the score $|\E_{t_{ij}}[\mathcal{L}(\hat{\Phi}^{\ell+1}_{t_{ij}};\mathcal{D})]-\mathcal{L}(\hat{\Phi}^{\ell+1};\mathcal{D})|$}, with $t_{ij}=[\b{W}_{\ell}]_{ij} \times {\rm Bernoulli}(p)/p$ for pruning or $t_{ij} = q([\b{W}_\ell]_{ij}; \|\b{W}_{\ell}\|_{\infty}, k)$ for quantization.
                \ENDFOR
                \STATE{Let $(i,j)$ be the index of the weight with least score.}
                \STATE{Set $[\hat{\b{W}}_{\ell}]_{ij} = t_{ij}$ and mark $[\b{W}_\ell]_{ij}$ as compressed.}
            \ENDFOR
            \STATE{Add a projection onto $\mathbb{B}^{n_{\ell+2}}_2 \l(c^{\ell+1}\r)$ on the output of the $(\ell+1)$-th layer in $\hat{\Phi}$.}
        \ENDIF
    \ENDIF
\ENDFOR
\RETURN $\hat{\Phi}.$
\end{algorithmic}
\end{algorithm}

The layer-wise compression routine we adopt is greedy: at each step, we prune or quantize the weight whose modification has the smallest effect on the loss. We next describe the procedure for a layer indexed by $\ell \in \mathcal{W}$. For every weight $[\b{W}_\ell]_{ij}$, we do the following: First, we freeze all weights except $[\b{W}_\ell]_{ij}$ in the current network $\Phi$, and consider the one-variable function $t \mapsto \mathcal{L}(\Phi^{\ell}_t; \mathcal{D})$, where $\Phi_t$ is obtained by setting $[\b{W}_\ell]_{ij} = t$. Second, we evaluate the \textit{score} of weight $[\b{W}_\ell]_{ij}$ given by  $ |\E_t[\mathcal{L}(\Phi^\ell_t; \mathcal{D})] - \mathcal{L}(\Phi^\ell; \mathcal{D})|$, where $t$ follows some discrete distribution to be chosen (e.g., Bernoulli). Finally, we rank all \textit{scores} and compress (i.e., set $[\hat{\b{W}}_\ell]_{ij}=t$) the weight $[\b{W}_\ell]_{ij}$ corresponding to the smallest \textit{score}. That is to say, we compress the weight to which the loss of the subnetwork $\Phi^\ell$ is the least sensitive on average. This operation is then repeated over the weights of the $\ell$-th layer, until a sufficiently large portion $\alpha \in (0, 1)$ of weights has been compressed. We note that when we evaluate the \textit{scores}, we use the most up-to-date network, that is the network obtained in the last compression step. The algorithm operates similarly for layers indexed by $\ell \in \mathcal{B}$, with the only difference being that the \textit{scores} are given  by $ |\E_t[\mathcal{L}(\Phi^{\ell+1}_t; \mathcal{D})] - \mathcal{L}(\Phi^{\ell+1}; \mathcal{D})|$ instead.

For pruning with sparsity level $p\in (0, 1)$, we  let $t = [\b{W}_{\ell}]_{ij} \times  {\rm Bernoulli}(p)/p$. 
For quantization with discretization level $2k$, we let $t = q([\b{W}_{\ell}]_{ij}; \|\b{W}_{\ell}\|_{\infty}, k)$. We now discuss the intuition behind our constructions of $t$. For both pruning and quantization $t$ satisfies $\E_t[t] = [\b{W}_{\ell}]_{ij}$, which cancels the first order mean variation of loss when $[\b{W}_{\ell}]_{ij}$ is replaced by $t$. Namely, the first order term in $t$ in $\mathcal{L}(\Phi^\ell_t; \mathcal{D}) - \mathcal{L}(\Phi^\ell; \mathcal{D})$ cancels once we take the expectation on $t$. In the case of pruning, the choice of $t$ ensures that the pruned weight $[\hat{\b{W}}_{\ell}]_{ij}=t$ is null with probability at least $1-p$, and thus the portion of pruned weights will be roughly $1-p$ (assuming $\alpha \approx 1$). In the case of quantization, we first construct a set $\Sm_\ell$ of at most $2k$ distinct values using $\|\b{W}_{\ell}\|_{\infty}$. Then, $t$ is designed to take the two values in $\Sm_\ell$ that are closest to $[\b{W}_{\ell}]_{ij}$. In essence, $t$ is a randomly rounded version of $[\b{W}_{\ell}]_{ij}$. We note that the set $\mathcal{S}_\ell$ may vary with $\ell$, but we can easily make a unified choice $ \Sm_\ell = \Sm$ across all layers by setting $t = q([\b{W}_{\ell}]_{ij}; \max_{r \in[m]}\|\b{W}_r\|_{\infty}, k)$ instead. We choose to state our results with $t = q([\b{W}_{\ell}]_{ij}; \|\b{W}_{\ell}\|_{\infty}, k)$ to showcase the difference in quantization tolerance by layer.

The addition of projections onto $\ell_2$ balls is necessary for our proofs. Indeed, projecting  ensures that the norms of the pruned network's outputs remain comparable to those of the original dense network. One could avoid projections and instead work with spectral norms of the pruned matrices. But the spectral norm of a pruned matrix can vary substantially, making such bounds crude. For instance, if $\b{W}_{\ell}$ is a random matrix with entries sampled independently from the uniform distribution over $[-c/\sqrt{n_{\ell} \vee n_{\ell+1}}, c/\sqrt{n_{\ell} \vee n_{\ell+1}}]$, then with high probability as $n_{\ell}\wedge n_{\ell+1} \to \infty$, there exists a mask matrix $\b{M}\in \{0,1\}^{n_{\ell+1}\times n_{\ell} }$ such that $\|\b{M}\|_0 \leq n_{\ell} n_{\ell+1}/2 + o(n_{\ell}n_{\ell+1})$ and $\|\b{W}_{\ell} \odot \b{M}\| = \Omega(c \sqrt{n_{\ell} \wedge n_{\ell+1}})$, whereas  originally $\|\b{W}_{\ell}\| \leq \Theta(c)$  with high probability. An example of such a mask matrix is given by $\b{M}_{ij} = 1_{[\b{W}_{\ell}]_{ij} > 0}$.  While it is unlikely that our algorithm follows similar adversarial patterns, we do not make any assumption on spectral norms of pruned matrices to keep our analysis general. Bounding spectral norms of pruned matrices from our algorithm is interesting but challenging since probabilistic tools from random matrix theory are limited when it comes to greedily randomized matrices, and thus requires extensive technical work. Spectral norm shift is an issue for compression algorithms, and heuristics have been derived to mitigate it (e.g. layer-wise penalization of weight matrices norms). The projection step shows that output magnitudes can be preserved without explicitly constraining the norms of layers, allowing for aggressive pruning.

%%%%%%%%%%%%%%%%%%%%%%%%%%%%%%%%%%%%%%%%%%%%%

%%%%%%%%%%%%%%%%%%%%%%%%%%%%%%%%%%%%%%%%%%%%%

\subsection{Main Results}\label{chapter3:section:main.results.1}

We present in this section our main theoretical guarantees for the compression scheme introduced in Section~\ref{chapter3:section:algorithm.description.1}. These results specify conditions on the architecture under which pruning and quantization at prescribed rates are feasible with controlled excess loss. 

Informally, we show that pruning (quantization) for layers indexed by $\mathcal{W}\cup\mathcal{B}$ can be achieved at any sparsity level $p$ (quantization level $2k$), provided that the latter layers satisfy some dimensionality properties. Namely, the layers indexed by $\mathcal{W}$ are  \textit{wide} layers, and those indexed by $\mathcal{B}$ are \textit{bottlenecks} of  \textit{tall} layers followed by \textit{wide} layers. Up to this point, we have made no assumptions on the data $\mathcal{D}$, and our results largely avoid making any non-trivial data assumption. Indeed,  we want to avoid data-based  properties that may facilitate compression in order to isolate network properties driving compression feasibility. While width is a key element in our results, our analysis provides indirect insights into weight interactions. Namely, the use of Hessian-based information incorporates correlations between weights. Similarly, our second order analysis of tall-then-wide layers relies implicitly on inter-layer weight interactions. Furthermore, our results for structured pruning in sections \ref{chapter3:section:structured.fcnn}, \ref{chapter3:section:structured.cnn} naturally exploit block-level weight interactions. Hence, several aspects of our analysis leverage implicitly weight interactions. We next list some key assumptions relevant for our results, and discuss their validity.

\begin{assumption}\label{chapter3:assumption:1}
\leavevmode\par\noindent
    \begin{itemize}
        \item \textnormal{(Activation Functions)}. For all $\ell\in [m]$,  the activation function $\varphi_\ell$  satisfies
            \begin{align}
                \varphi_\ell(\b{0}_{n_{\ell+1}}) &= \b{0}_{n_{\ell+1}} , \label{chapter3:assumption:1.1}\\
                \|\varphi_\ell\|_{\rm Lip} &\leq 1 .  \label{chapter3:assumption:1.2}
            \end{align}
            Moreover, if $\ell \in \mathcal{B}$, then $\varphi_\ell$ is a twice differentiable entrywise activation, and
            \begin{align}
                \sup_{x\in \R} |\varphi^{(2)}_{\ell} (x)| &\leq 1.  \label{chapter3:assumption:1.3}
            \end{align}
        \item \textnormal{(Weight Scaling)} There exists two constants $c_1, c_2\in \R_{>0}$ such that the following holds for all $\ell\in [m]$
            \begin{alignat}{2}
                &\|\b{W}_\ell\| &&\leq c_1 , \label{chapter3:assumption:1.4}\\
                &\|\b{W}_\ell\|_{\infty} &&\leq \frac{c_2}{\sqrt{ n_{\ell} \vee n_{\ell+1}}}  . \label{chapter3:assumption:1.5}
            \end{alignat}
    \item \textnormal{(Noisy Data)}
        There exists $\omega> 0$ such that the data distribution $\mathcal{D}$ satisfies
    \begin{align}
        \E_{(\b{x}, \b{y}) \sim \mathcal{D}} \l[ \|\b{y} - \E[\b{y} | \b{x}] \|^2 \r] \geq \omega^2. \label{chapter3:assumption:1.6}
    \end{align}
    \end{itemize}
\end{assumption}
We now discuss the validity of Assumption \ref{chapter3:assumption:1}. Items (\ref{chapter3:assumption:1.1}), (\ref{chapter3:assumption:1.2}) are satisfied by some usual activation functions such as ${\rm ReLU}$ and $\rm Tanh$. Moreover, (\ref{chapter3:assumption:1.1}) is not essential for our results to hold; we merely adopt it to simplify the presentation of our proofs. Assumption (\ref{chapter3:assumption:1.3}) is satisfied by various activations, including ${\rm Sigmoid}$ and ${\rm Tanh}$. Finally, our results hold verbatim if the constant $1$ in the right-hand side of (\ref{chapter3:assumption:1.2}) and (\ref{chapter3:assumption:1.3}) is replaced by any other explicit constant. We note that while ${\rm ReLU}$ does not satisfy (\ref{chapter3:assumption:1.3}) as it is not differentiable, one can  replace ${\rm ReLU}$ with a smooth approximation such as $\log(1+e^{\beta x})/\beta$ with $\beta>0$. 

The inequality (\ref{chapter3:assumption:1.4}) bounds the operator norm of  the MLP's layers, and is essential to control the network's sensitivity to perturbations. Indeed, spectral norms of weight matrices are linked to the MLP's robustness and generalization properties \cite{bartlett2017spectrally}. In that sense,  (\ref{chapter3:assumption:1.4}) ensures that the network (\ref{chapter3:fcnn}) is $c_1^m$-Lipschitz, which is a desirable property. Assumption (\ref{chapter3:assumption:1.5}) bounds the scale of  weights using the dimension of layers, and is based on the Glorot/Xavier initialization \cite{glorot2010understanding},  whereby the weights of each layer $\b{W}_\ell$ are initialized as independent random variables sampled from the uniform distribution over $[- \sqrt{6}/\sqrt{n_{\ell}+ n_{\ell+1}}, \sqrt{6}/\sqrt{n_{\ell} + n_{\ell+1}} ]$. We are aware that weight norms can grow considerably during training as evidenced in \cite{niehaus2024weight}. However, this growth is typically at the scale of a constant times the initial weights norm, which is well within (\ref{chapter3:assumption:1.5}) for a sufficiently large constant $c_2$. Furthermore, we note that if $\b{W}_{\ell}$ is a random matrix with independent mean zero entries, and second and fourth moments bounded  by $\mathcal{O}(1/ n_{\ell} \vee n_{\ell+1}) $ and $\mathcal{O}(1/(n_{\ell} \vee n_{\ell+1})^2)$ respectively, then (\ref{chapter3:assumption:1.4}) holds with high probability with $c_1=\mathcal{O}(c_2)$ by an application of Latala's inequality \cite{latala2005some}.  To further justify the adoption of (\ref{chapter3:assumption:1.4}) and (\ref{chapter3:assumption:1.5}), we compute the constants  $c_1, c_2$ for the MLP layers in TinyBERT (Transformer) and ResMLP (MLP) using Hugging Face weights. Namely, we compute $c_1 = \max_{\ell \in {\rm MLP}} \|\b{W}_\ell\|$ and $c_2 = \max_{\ell\in {\rm MLP}} \|\b{W}_\ell\|_{\infty} \sqrt{n_{\ell} \vee n_{\ell+1}}$. We obtain for TinyBERT: $c_1\approx 5.5$, $c_2\approx 28.4$, and for ResMLP: $c_1\approx 3.2$, $c_2\approx 20.5$. 

The bound (\ref{chapter3:assumption:1.6}) adds a noise condition on the data $\mathcal{D}$. Namely, no network can achieve a loss $\mathcal{L}(\Phi; \mathcal{D})$ below $\omega^2$. This assumption does not hold when the target $\b{y}$ is perfectly predictable from the covariates $\b{x}$. Since there are dense networks that can perfectly fit  data, it is harder to make performance comparisons with sparse networks that have  non-zero loss, which motivates the introduction of Assumption (\ref{chapter3:assumption:1.6}) to avoid such settings. This assumption is not necessary for our main results to hold, but leads to slightly more intuitive error bounds on the loss of compressed networks, as shown in Corollary \ref{chapter3:corollary:1}.

 We are now ready to state our main result for pruning. 
\begin{proposition}\label{chapter3:proposition:1}
    Suppose  (\ref{chapter3:assumption:1.1})-(\ref{chapter3:assumption:1.5}) hold in Assumption \ref{chapter3:assumption:1}. Let $\mathcal{R}$ be a distribution over $\mathbb{B}^{n_1}_{2}(1) $ and $\b{x} \sim \mathcal{R}$. Given $\xi\in(0,1)$ and $p\in (0,1)$ there exists positive constants  $\delta=\delta(\xi)$ and $n_0=n_0(p, \xi)$ such that if  $\min_{\ell \in \mathcal{W} \cup \mathcal{B}} (n_{\ell} \vee n_{\ell+1}) \geq n_0$ and
    \begin{alignat}{3}
        &\forall \ell\in \mathcal{W}, \quad && \frac{n_{\ell+1}}{n_{\ell }} &&\leq  \frac{p(1-\alpha)\delta}{ \alpha(1-p)} \frac{c_1^2 }{c_2^2}, \label{chapter3:proposition:1.w}\\
        &\forall \ell\in \mathcal{B}, \quad && \frac{1}{\sqrt{n_{\ell + 1}}} &&\leq \frac{p(1-\alpha)\delta}{\alpha(1-p)} \frac{1}{c_1^{\ell} c_2^2 \l(1 \vee c_1^{-4}\r) \l(1 \vee c_2^2 \vee \frac{c_2^2(1-p)}{p}\r)},  \label{chapter3:proposition:1.b.1}\\
        &\forall \ell\in \mathcal{B}, &&\frac{n_{\ell+2}}{n_{\ell+1}} &&\leq \frac{p(1-\alpha)\delta}{\alpha(1-p)} \frac{c_1^4 }{c_2^4} \label{chapter3:proposition:1.b.2},
    \end{alignat}
    with $\alpha\approx 0.99$, then there exists a  network $\hat{\Phi}$ given  by $\hat{\Phi}(\b{x}) = [\varphi_m(\hat{\b{W}}_m[\varphi_{m-1}(\dots \hat{\b{W}}_1\b{x})]_{\kappa_{m-1}})]_{\kappa_m} $ with $\kappa_\ell=c_1^\ell$ such that
    \begin{alignat}{2}
        \forall \ell \in &\mathcal{W}\cup \mathcal{B}, \quad \frac{\|\hat{\b{W}}_{\ell}\|_{0}}{n_{\ell} n_{\ell+1}} &&\leq 0.01 + 1.01p,  \label{chapter3:proposition:1.1}\\
        &\E_{\b{x}}\l[ \|\Phi(\b{x}) - \hat{\Phi}(\b{x})\|^{2} \r] &&\leq c_1^{2m} (1+\xi)^{m}\xi , \label{chapter3:proposition:1.2}
    \end{alignat}
    where $\hat{\b{W}}_{\ell}$ are the pruned weight matrices in the network $\hat{\Phi}$.  Furthermore, if $\mathcal{D}$ is a data distribution over $(\b{x}, \b{y}) \in \mathbb{B}^{n_1}_2(1) \times \R^{n_{m+1}}$, and $\varepsilon = (1+\xi)^m \xi$ then
    \begin{align}
        \mathcal{L}(\hat{\Phi}; \mathcal{D}) &\leq \mathcal{L}(\Phi; \mathcal{D}) + 2c_1^{m} \sqrt{\varepsilon \mathcal{L}(\Phi; \mathcal{D})} + c_1^{2m} \varepsilon.  \label{chapter3:proposition:1.3}
    \end{align}
\end{proposition}
We also have a similar result for quantization.
\begin{proposition}\label{chapter3:proposition:2}
    Suppose  (\ref{chapter3:assumption:1.1})-(\ref{chapter3:assumption:1.5}) hold in Assumption \ref{chapter3:assumption:1}. Let $\mathcal{R}$ be a distribution over $\mathbb{B}^{n_1}_{2}(1) $, $\b{x} \sim \mathcal{R}$ and  $\xi\in (0, 1)$. Given  $k\in \mathbb{Z}_{\geq 1}$, let 
    \begin{align}
        \mathcal{K} &= \l\{\pm \frac{ic_2}{k \sqrt{n_{\ell} \vee n_{\ell+1}}} \biggm| i\in [k] \r\}. \nonumber
    \end{align}
    Then, there exists positive constants $\delta=\delta(\xi)$ and $n_0=n_0(k, \xi)$ such that if $\min_{\ell \in \mathcal{W} \cup \mathcal{B}} n_{\ell} \vee n_{\ell+1}\geq n_0$ and
    \begin{alignat}{3}
        &\forall \ell\in \mathcal{W}, \quad && \frac{n_{\ell+1}}{n_{\ell }} &&\leq  \frac{(1-\alpha) k^2 \delta}{\alpha} \frac{c_1^2 }{c_2^2}, \label{chapter3:proposition:2.w}\\
        &\forall \ell\in \mathcal{B}, \quad && \frac{1}{\sqrt{n_{\ell + 1}}} &&\leq \frac{(1-\alpha)k^2\delta}{\alpha} \frac{1}{c_1^\ell c_2^2 \l(1\vee c_1^4 \r) \l(1 \vee c_2^2k^{-1}\r)},\label{chapter3:proposition:2.b.1}\\
        &\forall \ell \in \mathcal{B}, &&\frac{n_{\ell+2}}{n_{\ell+1}} &&\leq \frac{(1-\alpha) k^2 \delta}{\alpha} \frac{c_1^4 }{c_2^4} \label{chapter3:proposition:2.b.2},
    \end{alignat}
    with $\alpha\approx 0.99$, then there exists a network $\hat{\Phi}$ given  by $\hat{\Phi}(\b{x}) = [\varphi_m(\hat{\b{W}}_m[\varphi_{m-1}(\dots \hat{\b{W}}_1\b{x})]_{\kappa_{m-1}})]_{\kappa_m} $ with $\kappa_\ell=c_1^\ell$ such that
    \begin{alignat}{2}
        \forall \ell \in &\mathcal{W}\cup \mathcal{B}, \quad \frac{\l\|\b{1}_{\hat{\b{W}}_{\ell} \not \in \mathcal{K}}\r\|_{0}}{n_{\ell} n_{\ell+1}} &&\leq 0.01,  \label{chapter3:proposition:2.1}\\
        &\E_{\b{x}}\l[ \|\Phi(\b{x}) - \hat{\Phi}(\b{x})\|^{2} \r] &&\leq c_1^{2m} (1+\xi)^{m}\xi,  \label{chapter3:proposition:2.2}
    \end{alignat}
    where $\hat{\b{W}}_{\ell}$ are the quantized weight matrices in the network $\hat{\Phi}$.  Furthermore, if $\mathcal{D}$ is a data distribution over $(\b{x}, \b{y}) \in \mathbb{B}^{n_1}_2(1) \times \R^{n_{m+1}}$, and $\varepsilon = (1+\xi)^m \xi$ then
    \begin{align}
        \mathcal{L}(\hat{\Phi}; \mathcal{D}) &\leq \mathcal{L}(\Phi; \mathcal{D}) + 2c_1^{m} \sqrt{\varepsilon \mathcal{L}(\Phi; \mathcal{D})} + c_1^{2m} \varepsilon.  \label{chapter3:proposition:2.3}
    \end{align}
\end{proposition}

While the above propositions focus on deriving bounds on the $\ell_2$-squared loss $\mathcal{L}$, their results readily extend to many other losses. Indeed, suppose $d: \R^n \times \R^n \to \R_{\geq 0}$ is a metric and satisfies $\forall \b{a}, \b{b} \in \R^n, d(\b{a},\b{b}) \leq C \|\b{a}-\b{b}\|_2$ for some constant $C$. We then have using the triangle inequality that $\hat{\mathcal{L}}\triangleq \E[d(\hat{\Phi}(\b{x}), \b{y})] \leq \E[d(\Phi(\b{x}), \b{y})] + \E[d(\hat{\Phi}(\b{x}), \Phi(\b{x}))]\leq \mathcal{L} + C \E\|\hat{\Phi}(\b{x}) - \Phi(\b{x})\|$. Combining the latter with (\ref{chapter3:proposition:1.2}) and (\ref{chapter3:proposition:2.2}) would yield bounds on $\hat{\mathcal{L}}$. Examples of such losses are $\ell_p$ losses, Huber loss, and some Perceptual losses.

Assumptions (\ref{chapter3:proposition:1.w})-(\ref{chapter3:proposition:1.b.2}) for pruning and (\ref{chapter3:proposition:2.w})-(\ref{chapter3:proposition:2.b.2}) for quantization illustrate the driving mechanism of compression in our results. Namely, networks with architectures involving very \textit{wide}  layers or \textit{bottleneck} structures consisting in \textit{tall} layers followed by \textit{wide} layers need fewer parameters to achieve roughly the same data representation. The latter claim is quite intuitive: \textit{wide} layers embed a high-dimensional input into a low-dimensional space, thus retaining only some aspects of the input, which makes such layers  amenable to compression. On the other hand, a block consisting of a \textit{tall}-then-\textit{wide} layer initially expands the input into a high-dimensional space before projecting it into a low-dimensional space in order to learn relevant high dimensional features, which are harder to learn in the initial input space. However, this transformation does not alter the \textit{intrinsic} dimensionality of the input, which is intuitively at most the input's dimension. Thus, we should expect compression to work well on blocks of \textit{tall}-then-\textit{wide} layers. Informally, we require layers indexed by $\ell \in \mathcal{W}$  to satisfy $n_{\ell}\gg n_{\ell+1}$, and layers indexed by $\ell\in \mathcal{B}$ to satisfy $n_{\ell + 1}\gg n_{\ell+2}$. Moreover, we also require $n_{\ell+1}$ to grow exponentially with the depth $\ell$ for $\ell \in \mathcal{B}$.

We note that the dimension bounds corresponding to $\ell\in \mathcal{B}$ do not involve $n_{\ell}$, the input dimension of layer $\ell$. In particular, the compression results for $\ell\in \mathcal{B}$ extend to blocks of two wide layers $\b{W}_\ell, \b{W}_{\ell+1}$ as $n_\ell$ is not required to be small. However, one could simply prune  $\b{W}_\ell$ as a wide layer in the latter case by setting $\ell\in \mathcal{W}$. Therefore, we assume throughout this paper that $n_{\ell} \leq n_{\ell+1}$ for $\ell\in \mathcal{B}$ to avoid this redundancy.

\begin{remark}
    We note that while neither Proposition \ref{chapter3:proposition:1} nor \ref{chapter3:proposition:2} explicitly mention Algorithm \ref{chapter3:alg:pruning} in their statements, the crux of their proofs is based on a careful analysis of the performance of Algorithm \ref{chapter3:alg:pruning}. Specifically, we refer the reader to the proofs of Proposition \ref{chapter3:appendix:unstructured.single.proposition:3} and \ref{chapter3:appendix:unstructured.two.layers.proposition:3}. One can alternatively restate our results using the algorithm's compressed network with the caveat of obtaining random weight matrices $\hat{\b{W}}_{\ell}$.
\end{remark}
 Using Assumption (\ref{chapter3:assumption:1.6}), we can obtain the following  corollary from Propositions \ref{chapter3:proposition:1} and  \ref{chapter3:proposition:2}.
\begin{corollary}\label{chapter3:corollary:1}
In the settings of Proposition \ref{chapter3:proposition:1} and \ref{chapter3:proposition:2}, let $\varepsilon = (1+\xi)^m \xi$.  if $\mathcal{D}$ satisfies Assumption (\ref{chapter3:assumption:1.6}), then the following holds
\begin{align}
    \mathcal{L}(\hat{\Phi}; \mathcal{D}) &\leq \l(1 + \frac{2c_1^m \sqrt{\varepsilon}}{\omega} \r)\mathcal{L}(\Phi; \mathcal{D}) +  c_1^{2m} \varepsilon.
\end{align}
\end{corollary}

\section{Structured Pruning of Multilayer Perceptrons}\label{chapter3:section:structured.fcnn}
\subsection{Problem Formulation}\label{chapter3:section:problem.formulation.2}

We now turn from unstructured to structured pruning. In this setting, compression acts at the level of neurons instead of individual connection weights. Using the same notation introduced in Section \ref{chapter3:section:unstructured.fcnn}, let $m\in \mathbb{Z}_{\geq 1}$ and $\Phi$ be an $m$-layer MLP as defined in (\ref{chapter3:fcnn}).  Let $\mathcal{N}(\Phi)$ be the total number of neurons in $\Phi$. Namely, $\mathcal{N}(\Phi) = \sum_{\ell=2}^{m} n_\ell$. In this section, we are interested in the following  structured pruning feasibility problems.
\begin{problem}\label{chapter3:problem:3}
    Given $p\in (0, 1)$, an error threshold $\varepsilon>0$, a data distribution $\mathcal{D}$ and a network $\Phi$, is there a subnetwork $\hat{\Phi}$ of $\Phi$ such that
    \begin{align}
        \mathcal{N}(\hat{\Phi}) \leq  p \mathcal{N}(\Phi), \quad  \text{and} \quad\mathcal{L}(\hat{\Phi}; \mathcal{D}) \leq \mathcal{L}(\Phi; \mathcal{D}) + \varepsilon.
    \end{align}
\end{problem}
We refer to the above pruning problem as \textit{structured} since the objective is to remove entire neurons rather than individual weights. In other words, we seek sparse weight matrices $\hat{\b{W}}$ whose columns (corresponding to input neurons) or rows (corresponding to output neurons) are set to zero. Structured pruning is often preferred in practice, as unstructured sparsity is more difficult to leverage efficiently in hardware implementations.

\subsection{Algorithm Description}\label{chapter3:section:algorithm.description.2}

We now describe the variant of the compression scheme used for structured pruning. It is a direct adaptation of the algorithm described in Section~\ref{chapter3:section:algorithm.description.1}. Specifically, we retain the same notation for $\mathcal{W}$ and $\mathcal{B}$, and employ the same top–bottom compression procedure together with the projection operations applied after each pruning step.

The main modification concerns the definition of layer-wise scores, which we detail next. We first consider the case of layers indexed by $\ell \in \mathcal{W}$. In the latter, we focus on pruning the columns of $\b{W}_\ell$, which corresponds to pruning the input neurons and setting $\hat{\b{W}}_\ell = \b{W}_\ell \b{D}$ where $\b{D} = {\rm diag}(h_1,\dots,h_{n_{\ell}})$ is a sparse diagonal matrix in $\R^{n_{\ell}}$. For every column index $j \in [n_{\ell}]$, we do the following. First, we freeze all weights except $[\b{W}_\ell]_{:, j}$ in the current network $\Phi$, and consider the function $t \mapsto \mathcal{L}(\Phi^{\ell}_t; \mathcal{D})$, where $\Phi_t$ is obtained by replacing the $j$-th column in $\b{W}_\ell$ with $t\times [\b{W}_\ell]_{:, j}$. Second we evaluate the \textit{score} of column $[\b{W}_\ell]_{:, j}$ given by $|\E_t[\mathcal{L}(\Phi^\ell_t; \mathcal{D})] - \mathcal{L}(\Phi^{\ell}; \mathcal{D})|$, where $t$ has distribution ${\rm Bernoulli}(p)/p$, and $p$ is the target sparsity level. The scores of all columns are ranked, and we prune the column corresponding to the smallest score, i.e., we set $[\hat{\b{W}}_\ell]_{:, j} = t [\b{W}_\ell]_{:, j}$). This operation is repeated over the columns of the $\ell$-th layer, until a sufficiently large portion $\alpha \in (0, 1)$ of columns has been pruned. Similarly to the unstructured pruning algorithm, we use the most up-to-date network when evaluating the scores. In the case $\ell \in \mathcal{B}$, we focus on pruning the rows of $\b{W}_\ell$, which corresponds to pruning the output neurons and setting $\hat{\b{W}}_\ell = \b{D}\b{W}_\ell$ where $\b{D}$ is a sparse diagonal matrix in $\R^{n_{\ell+1}}$. The score for a given row in the $\ell$-th layer indexed by $i\in[n_{\ell+1}]$ is given by $|\E_t[\mathcal{L}(\Phi^{\ell+1}_t; \mathcal{D})] - \mathcal{L}(\Phi^{\ell+1}; \mathcal{D})|$ where $\Phi_t$ is obtained by setting $[\b{W}_\ell]_{i,:} = t [\b{W}_\ell]_{i,:}$, and $t$ has again distribution ${\rm Bernoulli}(p)/p$.

A key advantage of structured pruning for bottleneck layers $\b{W}_\ell, \b{W}_{\ell+1}$ with $\ell\in\mathcal{B}$ is the ability to also prune columns in the layer $\b{W}_{\ell+1}$. Indeed, Let $\mathcal{U}$ be the subset of $[n_{\ell+1}]$ corresponding to the indices of rows set to zero (Note that $\mathcal{U}$ is random), and let $\mathcal{V}=[n_{\ell+1}]\setminus \mathcal{U}$. Denote by $\b{z}\in \R^{n_{\ell}}$ the input of the $\ell$-th layer. If $\varphi_\ell$ is entrywise, then the output of the $(\ell+1)$-th layer is $\sum_{i\in \mathcal{V}} [\b{W}_{\ell+1}]_{:, i} \varphi_\ell([\b{W}_\ell]_{i,:} \b{z}) + \varphi(0)\sum_{i\in \mathcal{U}} [\b{W}_{\ell+1}]_{:, i}$. In particular, all the relevant information within the columns set $\{[\b{W}_{\ell+1}]_{:, i \mid i\in \mathcal{U}}\}$ is containing in their sum. Hence, it is sufficient to only keep one column equal to the sum, and remove the remaining $|\mathcal{U}|-1$ columns, which achieves a pruning rate $1 - (|\mathcal{U}|-1)/n_{\ell+1} \approx p$ for large $n_{\ell+1}$. Therefore, structured pruning allows us to simultaneously prune both $\b{W}_\ell$ and $\b{W}_{\ell+1}$ for $\ell \in \mathcal{B}$, whereas we were limited to pruning $\b{W}_\ell$ alone in the case of unstructured pruning.

\subsection{Main Results}\label{chapter3:section:main.results.2}
We now state the main theoretical guarantee corresponding to the structured pruning scheme described in the previous section.
\begin{proposition}\label{chapter3:proposition:3}
    Suppose  (\ref{chapter3:assumption:1.1})-(\ref{chapter3:assumption:1.5}) hold in Assumption \ref{chapter3:assumption:1}. Let $\mathcal{R}$ be a distribution over $\mathbb{B}^{n_1}_{2}(1) $ and $\b{x} \sim \mathcal{R}$. Given $\xi\in(0,1)$ and $p\in (0,1)$ there exists positive constants $\delta=\delta(\xi)$ and $n_0=n_0(p, \xi)$ such that if $\min_{\ell \in \mathcal{W} \cup \mathcal{B}} (n_{\ell} \vee n_{\ell+1}) \geq n_0$, and
    \begin{alignat}{3}
        &\forall \ell\in \mathcal{W}, \quad && \frac{n_{\ell+1}}{n_\ell} &&\leq  \frac{\alpha(1-p)\delta}{p(1-\alpha)} \frac{c_1^2}{c_2^2}, \label{chapter3:proposition:3.w}\\
        &\forall \ell\in \mathcal{B}, \quad &&\frac{n_{\ell+2}}{n_{\ell+1}} &&\leq \frac{p^2(1-p)^2\delta^2}{\alpha^2 (1-p)^2}\frac{(1 \wedge c_1^{8})}{c_1^{2(\ell+1)}c_2^2}  \wedge  \frac{ p(1-\alpha)\delta}{\alpha(1-p)} \frac{c_1^2}{c_2^2},\label{chapter3:proposition:3.b.1} \\
        &\forall \ell \in \mathcal{B}, \quad &&\frac{\sqrt{n_\ell n_{\ell+2}}}{n_{\ell+1}} &&\leq \frac{p(1-\alpha)\delta}{\alpha(1-p)} \frac{1}{c_1^{\ell-2}c_2^2\l(1\vee \frac{1-p}{p}\r)} , \label{chapter3:proposition:3.b.2}
    \end{alignat}
with $\alpha\approx 0.99$, then there exists a  network $\hat{\Phi}$ given  by $\hat{\Phi}(\b{x}) = [\varphi_m(\hat{\b{W}}_m[\varphi_{m-1}(\dots \hat{\b{W}}_1\b{x})]_{\kappa_{m-1}})]_{\kappa_m} $ with $\kappa_\ell=c_1^\ell$ such that
\begin{enumerate}
    \item\label{chapter3:proposition:3.1} $\E_{\b{x}}\l[ \|\Phi(\b{x}) - \hat{\Phi}(\b{x})\|^{2} \r] \leq c_1^{2m} (1+\xi)^{m}\xi$.
    \item \label{chapter3:proposition:3.2} For all $ \ell \in \mathcal{W}$, the matrix $\hat{\b{W}}_\ell$ has at most $1.01p$ fraction of its columns not set to zero.
    \item \label{chapter3:proposition:3.3} For all $ \ell \in \mathcal{B}$, the matrix $\hat{\b{W}}_\ell$ has at most $1.01p$ fraction of its rows not set to zero.
\end{enumerate}
 Furthermore, if $\mathcal{D}$ is a data distribution over $(\b{x}, \b{y}) \in \mathbb{B}^{n_1}_2(1) \times \R^{n_{m+1}}$, and $\varepsilon = (1+\xi)^m \xi$ then
    \begin{align}
        \mathcal{L}(\hat{\Phi}; \mathcal{D}) &\leq \mathcal{L}(\Phi; \mathcal{D}) + 2c_1^{m} \sqrt{\varepsilon \mathcal{L}(\Phi; \mathcal{D})} + c_1^{2m} \varepsilon.  \label{chapter3:proposition:3.4}
    \end{align}
\end{proposition}
The result of Corollary \ref{chapter3:corollary:1} also extends verbatim in the case of structured pruning if we add the assumption \ref{chapter3:assumption:1.6}.

\section{Structured Pruning of Convolutional Neural Networks}\label{chapter3:section:structured.cnn}

\subsection{Problem Formulation}\label{chapter3:section:problem.formulation.3}
We now turn to convolutional architectures. In this setting, the basic building blocks are convolutional layers rather than fully connected ones. Let $m\in \mathbb{Z}_{\geq 1}$, and introduce similarly to (\ref{chapter3:fcnn}) the following function representing an $m$-layer convolutional neural network (CNN)
\begin{align}
    \Phi: \R^{d_1 \times h_1 \times w_1} \to \R^{d_{m+1} \times h_{m+1} \times w_{m+1}}: \b{x} \mapsto \varphi_m(\b{K}_m * \varphi_{m-1}(\dots \varphi_1(\b{K}_1 * \b{x}))) \label{chapter3:cnn}
\end{align}
where $\b{K}_\ell \in \R^{d_{\ell+1} \times d_{\ell} \times q \times q}$ are four-dimensional tensors representing the convolutional filters in the network and $*$ is the convolution operator. In particular, for each pair $(o, i) \in [d_{\ell+1}]\times[d_{\ell}]$, the kernel $[\b{K}_\ell]_{o, i} \in \R^{q \times q}$ connects the $i$-th input feature map in the $\ell$-layer to the $o$-th output feature map in the same layer. Furthermore, we make the assumption that all kernels $[\b{K}_{\ell}]_{o, i}$ have dimension $q\times q$ to simplify the presentation of our guarantees, but our results extend more generally to networks with varying kernel dimensions. Moreover, we assume $q\leq \min_{\ell \in [m+1]} \max(h_\ell, w_{\ell})$. This assumption is not restrictive, as in practice $q$ is typically chosen from $\{3, 5, 7\}$.

We also use the notation $\Phi^\ell$ for $\ell\in[m]$ to denote the subnetwork of $\Phi$ at depth $\ell$. Let $\mathcal{F}(\Phi)$ be the total number of filters in $\Phi$. Namely, $\mathcal{F}(\Phi) = \sum_{\ell=2}^{m} d_\ell$. In this section, we are interested in the following  structured pruning feasibility problems.
\begin{problem}\label{chapter3:problem:4}
    Given $p\in (0, 1)$, an error threshold $\varepsilon>0$, a data distribution $\mathcal{D}$ and a network $\Phi$, is there a subnetwork $\hat{\Phi}$ of $\Phi$ such that
    \begin{align}
        \mathcal{F}(\hat{\Phi}) \leq  p \mathcal{F}(\Phi), \quad  \text{and} \quad\mathcal{L}(\hat{\Phi}; \mathcal{D}) \leq \mathcal{L}(\Phi; \mathcal{D}) + \varepsilon. 
    \end{align}
\end{problem}
The analysis of CNN architectures is more challenging than that of their MLP counterparts, primarily due to the complex nature of convolution operations. To address this technical difficulty, we first show that convolutional layers can be equivalently represented as fully connected layers with sparse, doubly circular weight matrices, which is a well known in the literature. In the following section, we further demonstrate that CNNs can be recast as MLPs with structured weight matrices and flattened input data. These equivalent representations enable us to extend the pruning analysis from Section~\ref{chapter3:section:structured.fcnn} to CNNs.

\subsection{Convolutional Networks Representation as Multilayer Perceptrons}\label{chapter3:section.cnn.networks.representation.fcnn}
In order to apply our MLP-based compression analysis to CNNs, we first express convolutional layers as linear maps acting on suitably vectorized inputs. We briefly recall the standard CNN representation. A convolutional layer maps an input tensor $\b{X} \in \R^{d_{\rm in} \times h_{\rm in} \times w_{\rm in}}$ to an output tensor $\b{Z} \in \R^{d_{\rm out} \times h_{\rm out} \times w_{\rm out}}$ using a kernel tensor $\b{K} \in \R^{d_{\rm out} \times d_{\rm in} \times q \times q}$. The height and width of the input and output feature maps are denoted by $h_{\rm in}, w_{\rm in}$ and $h_{\rm out}, w_{\rm out}$, where $d_{\rm in}$ and $d_{\rm out}$  are the number of input and output channels, and $q \times q$ is the spatial dimension of each convolutional kernel.

The output feature map $\b{Z}$ is obtained by applying $d_{\rm out}$ three-dimensional filters $\b{K}_o \in \R^{d_{\rm in} \times q \times q}$ with $o\in [d_{\rm out}]$. To simplify the exposition of our results, we assume that all the convolutional layers use unit stride and  circular padding. Moreover, we assume that the height and width of each feature map are identical. Under these assumptions, we have $h_{\rm in}=h_{\rm out}= w_{\rm in} = w_{\rm out}  \triangleq r$, and the convolution operation can be written as
\begin{align*}
    \b{Z}_{o, u, v} &= \sum_{i=1}^{d_{\rm in}} \sum_{a=1}^{r} \sum_{b=1}^{r} \b{K}_{o, i, a, b} \b{X}_{i, \langle u+a \rangle_r, \langle v+b \rangle_r}, \quad \forall (o, u, v) \in [d_{\rm out}] \times [r] \times [r]
\end{align*}
where $\langle t \rangle_r = ((t-1) \bmod r) + 1$ denotes circular indexing modulo $r$. In particular, $\langle t \rangle_r \in [r]$. For notational simplicity, we adopt the (slightly abused) shorthand $\b{X}_{i, u+a, v+b} = \b{X}_{i, \langle u+a \rangle_r, \langle v+b \rangle_r}$. Furthermore, we extend each kernel tensor $\b{K}_o$ to $\R^{d_{\rm in} \times r \times r}$ by setting $\b{K}_{o, i, a, b} = 0$ for all $i\in [d_{\rm in}]$ and $(a, b) \in [r]\times [r]$ satisfying $\max(a, b)> q$. Finally, let $\phi(u, v) = (u-1)r + v$. For $i\in [d_{\rm in}]$ define the column vector $\b{x}_i \in \R^{r^2}$ by
\begin{align*}
    [\b{x}_i]_{\phi(u, v)} = \b{X}_{i, u, v}, \quad \forall (u, v) \in [r] \times [r]
\end{align*}
and let $\b{x} = (\b{x}_1^\top, \dots, \b{x}_{d_{\rm in}}^\top)^\top \in \R^{d_{\rm in}r^2}$. Similarly, for each $o\in [d_{\rm out}]$ define the column vector $\b{z}_o \in \R^{r^2}$ by
\begin{align*}
    [\b{z}_o]_{ \phi(u, v)} = \b{Z}_{o, u, v}, \quad \forall (u, v) \in [r] \times [r]
\end{align*}
and let $\b{z} = (\b{z}_1^\top, \dots, \b{z}_{d_{\rm out}}^\top)^\top \in \R^{d_{\rm out} r^2}$.  For a given matrix $\b{U} \in \R^{r\times r}$, introduce the doubly block circulant matrix $\mathcal{C}(\b{U}) \in \R^{r^2 \times r^2}$ given by
\begin{align*}
    \l[\mathcal{C}(\b{U}) \b{y}\r]_{\phi(u, v)} &= \sum_{a=1}^{r}\sum_{b=1}^{r} \b{U}_{a, b} \b{y}_{\phi(\langle u+a \rangle_r, \langle v+b\rangle_r)}, \quad \forall \b{y} \in \R^{r^2},  \forall (u, v) \in [r] \times [r]
\end{align*}
Note in particular that $\mathcal{C}(\b{U})_{\phi(u, v), \phi(u', v')} = \b{U}_{\langle u'-u\rangle_r, \langle v'-v\rangle_r}$. Finally, let
\begin{align*}
    \mathbf{W}(\b{K}) =
    \begin{bmatrix}
    \mathcal{C}(\b{K}_{1,1}) & \cdots & \mathcal{C}(\b{K}_{1,d_{\mathrm{in}}}) \\
    \vdots  & \ddots & \vdots \\
    \mathcal{C}(\b{K}_{d_{\mathrm{out}},1}) & \cdots & \mathcal{C}(\b{K}_{d_{\mathrm{out}},d_{\mathrm{in}}})
    \end{bmatrix}
    \;\in\; \mathbb{R}^{d_{\mathrm{out}}r^{2}\times d_{\mathrm{in}}r^{2}}.
\end{align*}
Using Lemma \ref{chapter3:lemma:5}, it follows that $\b{z} = \b{W}(\b{K}) \b{x}$. Therefore, we can equivalently represent the CNN layer $\b{K}$ using the MLP layer $\b{W}(\b{K})$. Let the convolutional layers be given by $\b{K}_{\ell} \in \R^{d_{\ell+1} \times d_{\ell} \times r \times r}$ for $\ell \in [m]$, and let $\b{x} \in \R^{d_1 r^2}$ be the flattened input feature map. Then the CNN (\ref{chapter3:cnn}) can be represented equivalently by an MLP $\Phi$ given by
\begin{align*}
    \Phi(\b{x}) = \varphi_m(\b{W}_m \varphi_{m-1}(\dots \varphi_1(\b{W}_1 \b{x}))),
\end{align*}
where
\begin{align*}
    \mathbf{W}_\ell = \b{W}(\b{K}_\ell) =
    \begin{bmatrix}
    \mathcal{C}([\b{K}_\ell]_{1,1}) & \cdots & \mathcal{C}([\b{K}_\ell]_{1,d_{\ell}}) \\
    \vdots  & \ddots & \vdots \\
    \mathcal{C}([\b{K}_\ell]_{d_{\ell+1},1}) & \cdots & \mathcal{C}([\b{K}_\ell]_{d_{\ell+1},d_{\ell}})
    \end{bmatrix}
    \;\in\; \mathbb{R}^{d_{\ell+1}r^{2} \times d_{\ell}r^{2}}.
\end{align*}
In particular, pruning the $o$-th block row of size $r^2$ given by the matrix $[\mathcal{C}([\b{K}_\ell]_{o,1})  \cdots  \mathcal{C}([\b{K}_\ell]_{o,d_{\ell}})]$ corresponds to pruning the filter $\b{K}_{o, :, :, :}$. Similarly, pruning the $i$-th block column of size $r^2$ given by $[\mathcal{C}([\b{K}_\ell]_{1,i})  \cdots  \mathcal{C}([\b{K}_\ell]_{d_{\ell+1},i})]^\top$ corresponds to pruning the filter $[\b{K}_\ell]_{:,i,:,:}$. 

\subsection{Algorithm Description}\label{chapter3:section:algorithm.description.3}
Building on the MLP representation of CNN layers derived in the previous section, we now describe the corresponding structured pruning procedure. The algorithm we use is an adapted variant of the procedure described in Section \ref{chapter3:section:structured.fcnn}. For $\ell \in \mathcal{W}$, we prune column-blocks of $\b{W}_\ell$, which corresponds to removing  input filters, and set $\hat{\b{W}}_\ell = \b{W}_\ell \b{D}$, where $\b{D} = \b{G}(h; S)$ is a sparse diagonal matrix, and $S$ denotes the set of columns in $\b{W}_\ell$ corresponding to the pruned input filters. Namely, for every set of columns $S$ representing input filters, we freeze all other weights except $[\b{W}_{\ell}]_{:, S}$ in the current network $\Phi$ and consider the function $t\mapsto \mathcal{L}(\Phi_t^\ell; \mathcal{D})$, where $\Phi_t$ is obtained by replacing the column block $[\b{W}_\ell]_{:, S}$ with $t\times [\b{W}_\ell]_{:, S}$. The scores are then defined by $|\E\l[\mathcal{L}(\Phi_t^\ell; \mathcal{D})\r]-\mathcal{L}(\Phi^\ell; \mathcal{D})|$.
Similarly, for $\ell \in \mathcal{B}$, we prune row-blocks of $\b{W}_\ell$, which corresponds to removing output filters, and set $\hat{\b{W}}_\ell =  \b{D} \b{W}_\ell $, where $\b{D} = \b{G}(h; S)$ is again a sparse diagonal matrix, and $S$ denotes the set of rows in $\b{W}_\ell$ corresponding to the pruned output filters. As in the MLP case, the scores are given by $|\E\l[\mathcal{L}(\Phi_t^{\ell+1}; \mathcal{D})\r]-\mathcal{L}(\Phi^{\ell+1}; \mathcal{D})|$ where $\Phi_t$ is obtained by changing the block $[\b{W}_\ell]_{S, :}$ into $t\times [\b{W}_\ell]_{S, :}$, with $t={\rm Bernoulli}(p)/p$.

\subsection{Main Results}\label{chapter3:section:main.results.3}
We now combine the CNN–MLP representation and the structured pruning algorithm described above to obtain our main pruning guarantee for convolutional networks.
\begin{proposition}\label{chapter3:proposition:4}
    Suppose (\ref{chapter3:assumption:1.1})- (\ref{chapter3:assumption:1.4}) hold in Assumption \ref{chapter3:assumption:1}. Moreover, suppose that there exists a constant $c_2$ such that
    \begin{align}
        \forall \ell\in [m], \quad \|\b{W}_\ell\|_{\infty} \leq \frac{c_2}{q\sqrt{d_{\ell} \vee d_{\ell+1}}}. \label{chapter3:proposition:4.0}
    \end{align}
    Let $\mathcal{R}$ be a distribution over $\mathbb{B}^{n_1}_{2}(1) $ and $\b{x} \sim \mathcal{R}$. Given $\xi\in(0,1)$ and $p\in (0,1)$ there exists positive constants $\delta=\delta(\xi)$ and $n_0=n_0(p, \xi)$ such that if $\min_{\ell \in \mathcal{W} \cup \mathcal{B}} (d_{\ell} \vee d_{\ell+1}) \geq n_0$ and
    \begin{alignat}{3}
        &\forall \ell\in \mathcal{W}, \quad && \frac{d_{\ell+1}}{d_\ell} &&\leq \frac{\alpha(1-p)\delta}{p(1-\alpha)} \frac{c_1^2}{c_2^2q^2}, \label{chapter3:proposition:4.w}\\
        &\forall \ell\in \mathcal{B}, \quad &&\frac{d_{\ell+2}}{d_{\ell+1}} &&\leq \frac{p^2(1-p)^2\delta^2}{\alpha^2 (1-p)^2}\frac{(1 \wedge c_1^{8})}{c_1^{2(\ell+1)}c_2^2 q^2}  \wedge  \frac{ p(1-\alpha)\delta}{\alpha(1-p)} \frac{c_1^2}{c_2^2q^2}, \label{chapter3:proposition:4.b.1}\\
        &\forall \ell\in \mathcal{B}, \quad &&\frac{\sqrt{d_\ell d_{\ell+2}}}{d_{\ell+1}} &&\leq \frac{p(1-\alpha)\delta}{\alpha(1-p)} \frac{1}{c_1^{\ell-2}c_2^2q^2\l(1\vee \frac{1-p}{p}\r)}, \label{chapter3:proposition:4.b.2}
    \end{alignat}
     where $\alpha\approx 0.99$, then there exists a  network $\hat{\Phi}$ given  by $\hat{\Phi}(\b{x}) = [\varphi_m(\hat{\b{W}}_m[\varphi_{m-1}(\dots \hat{\b{W}}_1\b{x})]_{\kappa_{m-1}})]_{\kappa_m} $ with $\kappa_\ell=c_1^\ell$ such that
    \begin{enumerate}
        \item\label{chapter3:proposition:4.1} $\E_{\b{x}}\l[ \|\Phi(\b{x}) - \hat{\Phi}(\b{x})\|^{2} \r] \leq c_1^{2m} (1+\xi)^{m}\xi$.
        \item \label{chapter3:proposition:4.2} For all $ \ell \in \mathcal{W}$, the matrix $\hat{\b{W}}_\ell$ has at most $1.01p$ fraction of its input filters not set to zero.
        \item \label{chapter3:proposition:4.3} For all $ \ell \in \mathcal{B}$, the matrix $\hat{\b{W}}_\ell$ has at most $1.01p$ fraction of its output filters not set to zero.
    \end{enumerate}
    \begin{alignat}{2}
        &\E_{\b{x}}\l[ \|\Phi(\b{x}) - \hat{\Phi}(\b{x})\|^{2} \r] &&\leq c_1^{2m} (1+\xi)^{m}\xi. \label{chapter3:proposition:4.4}
    \end{alignat}
    Furthermore, if $\mathcal{D}$ is a data distribution over $(\b{x}, \b{y}) \in \mathbb{B}^{n_1}_2(1) \times \R^{n_{m+1}}$, and $\varepsilon = (1+\xi)^m \xi$ then
    \begin{align}
        \mathcal{L}(\hat{\Phi}; \mathcal{D}) &\leq \mathcal{L}(\Phi; \mathcal{D}) + 2c_1^{m} \sqrt{\varepsilon \mathcal{L}(\Phi; \mathcal{D})} + c_1^{2m} \varepsilon.  \label{chapter3:proposition:4.5}
    \end{align}
\end{proposition}
The result of Corollary \ref{chapter3:corollary:1} also extends verbatim in the case of structured pruning if we add the assumption \ref{chapter3:assumption:1.6}. We note that Assumption (\ref{chapter3:proposition:4.0}) bounds the magnitude of weights in CNNs using the kernel dimensions $q, d_{\rm in}$ and $d_{\rm out}$, and is the derived from the Glorot/Xavier and Kaiming initialization schemes in the case of CNNs \cite{he2015delving, glorot2010understanding}.

\section{Numerical Simulations}\label{chapter3:section:numerical.simulation}

To validate the theoretical insights from Propositions~\ref{chapter3:proposition:1},~\ref{chapter3:proposition:2},~\ref{chapter3:proposition:3}, and~\ref{chapter3:proposition:4}, we conduct a series of numerical simulations. The goal of these experiments is to illustrate the tradeoff between width and compressibility, and to demonstrate the effectiveness of the randomized greedy pruning algorithm analyzed in this work. Our setup for experiments is general and designed to be consistent across different pruning settings and learning tasks. 

\textbf{Data}. We use two benchmarks: the California Housing (regression), and the Digits  (classification) datasets. For each, the data $(\b{x}_i, \b{y}_i)_{i\in [N]}$ is split into an $80\%$ training set and a $20\%$ test set, with corresponding empirical distributions denoted by $\mathcal{D}^{\rm train}$ and $\mathcal{D}^{\rm test}$.

\textbf{Model Training}. We consider a series of networks $\Phi_w$ with varying width $w\in \mathbb{Z}_{\geq 0}$. For each width value $w$ within a predefined range $\{w_1,\ldots, w_K\}$, a dense model $\Phi_w$ is trained on $\mathcal{D}^{\rm train}$ using the Adam optimizer. We use the Mean Squared Error loss for regression, and the Cross-Entropy loss for classification tasks.

\textbf{Pruning Application}. For all our simulations, we set the pruning fraction to $\alpha=0.9$, and pruning  probability to $p=0.3$. This leads to an expected sparsity level of $1-\alpha + \alpha p = 37\%$ (i.e. roughly $37\%$ of weights remain in pruned layers). We apply our pruning algorithm on each trained network $\Phi_w$. Due to the randomized nature of this algorithm, we train several (randomly initialized) networks per width, and repeat the pruning  process 50 times for each trained network. This allows us to capture the expected value and confidence interval of our evaluation metrics across the randomness from pruning and training. We consider two pruning setups: $(\mathcal{W}, \mathcal{B})=(\{1\}, \emptyset)$ in blue and $(\mathcal{W}, \mathcal{B})=(\emptyset, \{2\}) $ in orange. While the pruning procedure described in the preceding sections includes projection operations, we omit them in our numerical experiments, as they are primarily required for theoretical analysis.  

\textbf{Evaluation}. Let $\hat{\Phi}_w$ be the pruned network obtained from $\Phi_w$. We measure the impact of pruning using two metrics
\begin{itemize}
    \item The adjusted $\ell_2$-squared distance between models' outputs given by
    \begin{align*}
        \Delta(\Phi, \hat{\Phi}) = c^{-2m}\E_{\mathcal{D}^{\rm test}}\l[\|\Phi(\b{x}) - \hat{\Phi}(\b{x})\|^2\r],
    \end{align*}
    where $c\triangleq\max_{\ell \in [m]} \|\b{W}_\ell\|$ is the maximum layer-wise operator norm and $m$ is the network depth. The multiplicative constant $c^{-2m}$ matches the scaling factor for the bounds presented in Proposition~\ref{chapter3:proposition:1}-\ref{chapter3:proposition:4}, and adjusts for depth and error variations caused by the final learned weights, which differ between each stochastic training run.
    \item A task-specific performance metric: classification accuracy or  regression R-squared, evaluated on both the dense and pruned models $\Phi$, $\hat{\Phi}$. 
\end{itemize}
As predicted by propositions~\ref{chapter3:proposition:1}-\ref{chapter3:proposition:4}, we observe a steady decrease of the error $\Delta(\Phi, \hat{\Phi})$ in all pruning settings as $w$ increases. Furthermore, the performance of pruned networks improves with width as well.

\renewcommand{\topfraction}{0.95}
\renewcommand{\bottomfraction}{0.95}
\renewcommand{\textfraction}{0.05}
\renewcommand{\floatpagefraction}{0.8}
\setcounter{topnumber}{3}
\setcounter{bottomnumber}{2}
\setcounter{totalnumber}{4}

\begin{figure}[H]
    \centering
    \includegraphics[width=\textwidth,trim=0 0 0 50,clip]{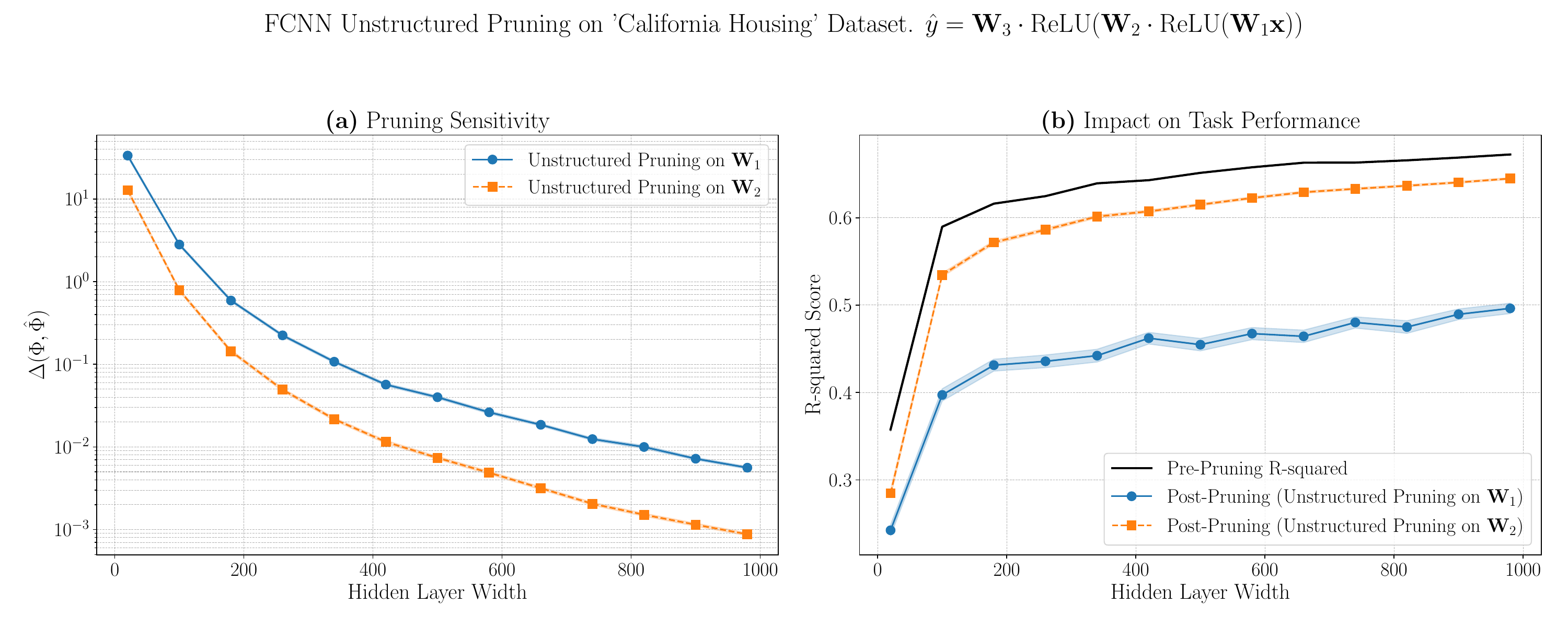}
    \caption{Evaluation plots for unstructured pruning  on the California Housing dataset using an MLP model $\Phi_w(\b{x}) = \b{W}_3{\rm ReLU}(\b{W}_2 {\rm ReLU}(\b{W}_1 \b{x}))$ with $\b{W}_1\in \R^{w\times 8}, \b{W}_2\in \R^{40\times w}, \b{W}_3\in \R^{1\times 40}$. }
    \label{chapter3:fig:fcnn_california_results_unstructured}
\end{figure}

\begin{figure}[H]
    \centering
    \includegraphics[width=\textwidth,trim=0 0 0 50,clip]{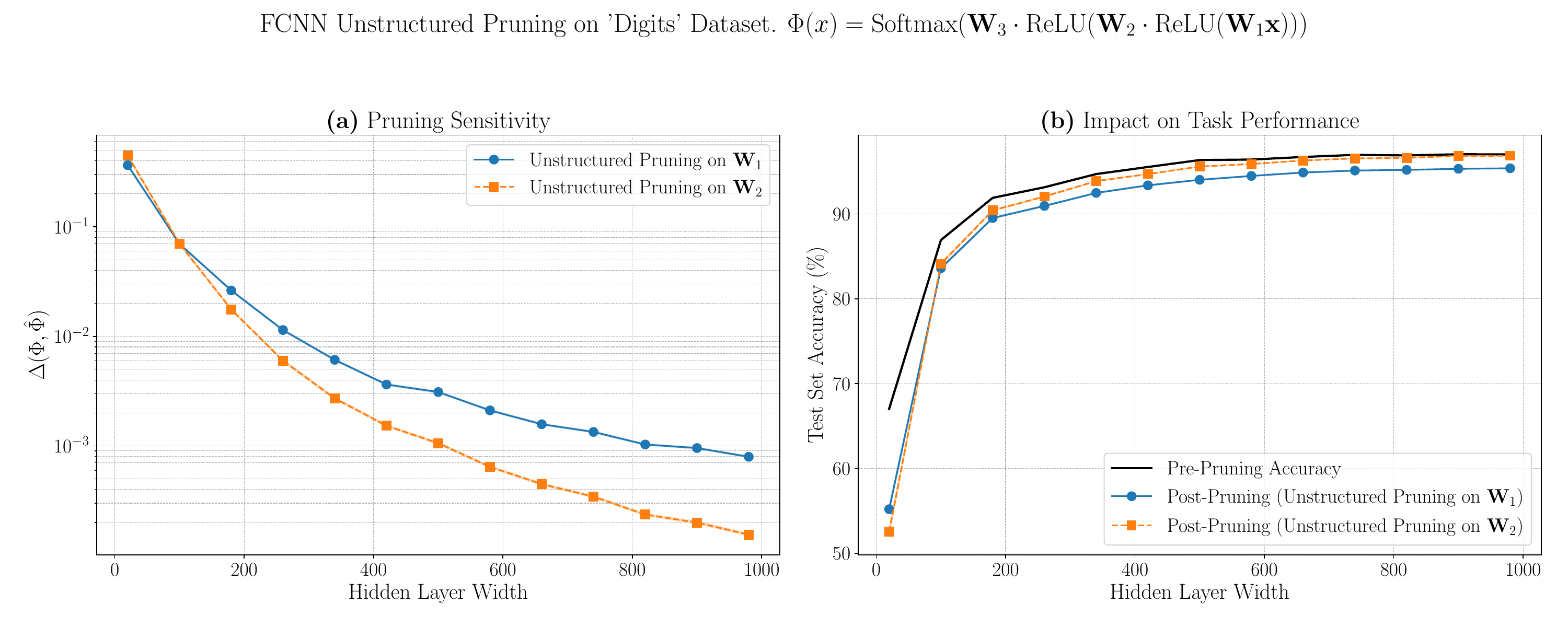}
    \caption{Evaluation plots for unstructured pruning on the Digits dataset using an MLP model $\Phi_w(\b{x}) = {\rm Softmax}(\b{W}_3{\rm ReLU}(\b{W}_2 {\rm ReLU}(\b{W}_1 \b{x})))$ with $\b{W}_1\in \R^{w\times 64}, \b{W}_2\in \R^{40\times w}, \b{W}_3\in \R^{10\times 40}$.}
    \label{chapter3:fig:fcnn_digits_results_unstructured}
\end{figure}

\begin{figure}[H]
    \centering
    \includegraphics[width=\textwidth,trim=0 0 0 50,clip]{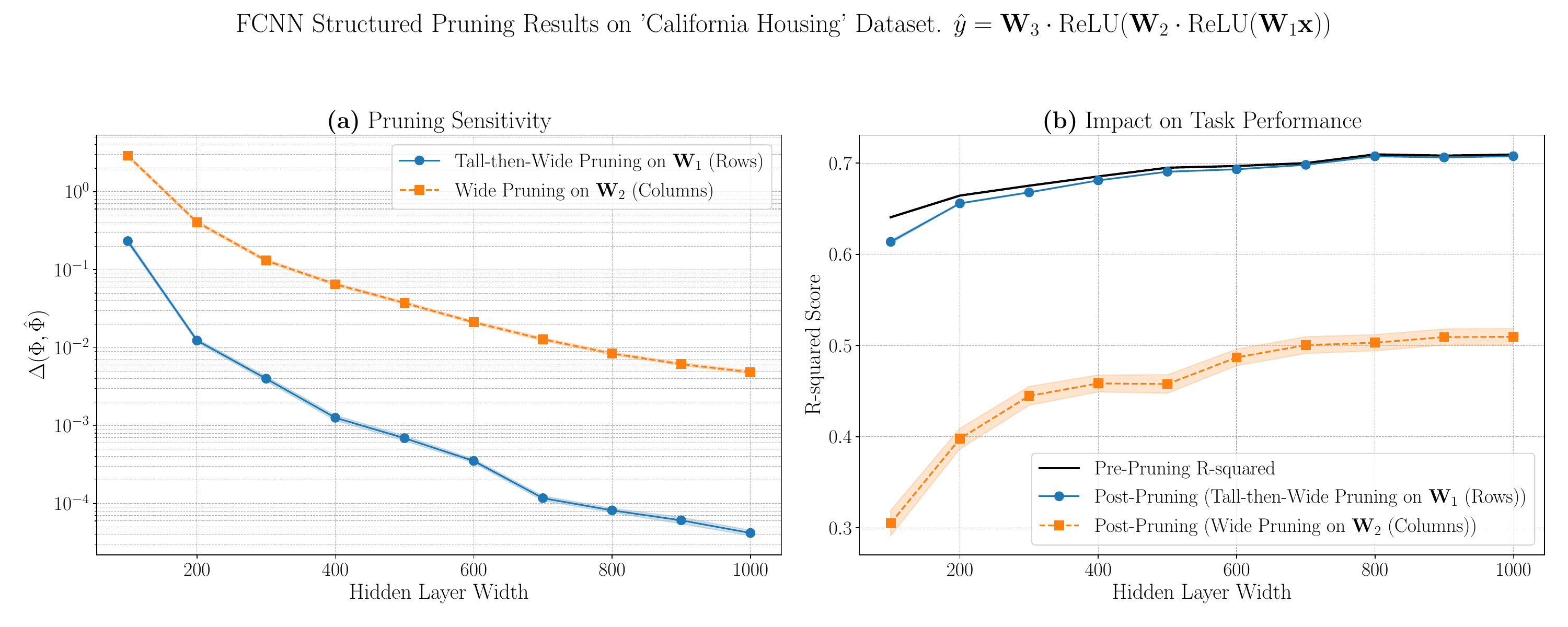}
    \caption{Evaluation plots for structured pruning on the California Housing dataset using an MLP model $\Phi_w(\b{x}) = \b{W}_3{\rm ReLU}(\b{W}_2 {\rm ReLU}(\b{W}_1 \b{x}))$ with $\b{W}_1\in \R^{w\times 8}, \b{W}_2\in \R^{20\times w}, \b{W}_3\in \R^{1\times 20}$.}
    \label{chapter3:fig:structured_fcnn_california_r2}
\end{figure}

\begin{figure}[H]
    \centering
    \includegraphics[width=\textwidth,trim=0 0 0 50,clip]{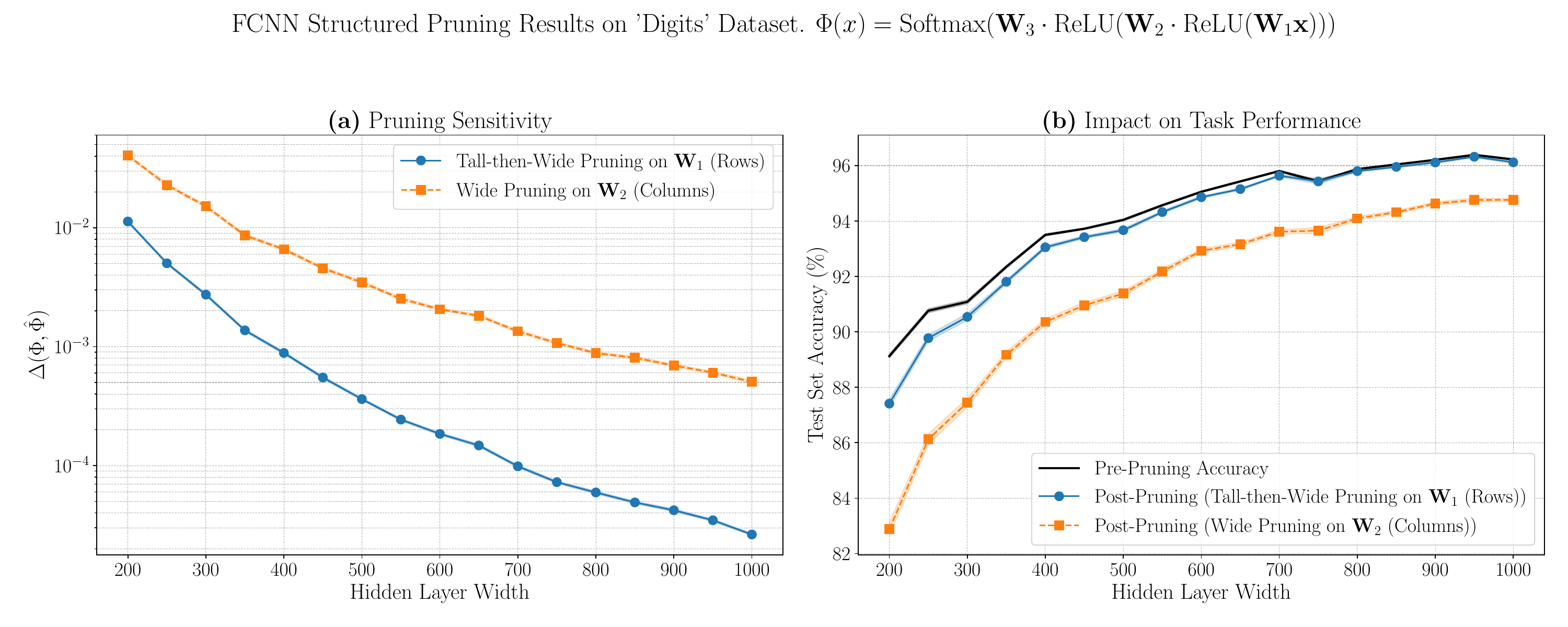}
    \caption{Evaluation plots for structured pruning on the Digits dataset using an MLP model $\Phi_w(\b{x}) = {\rm Softmax}(\b{W}_3{\rm ReLU}(\b{W}_2 {\rm ReLU}(\b{W}_1 \b{x})))$ with $\b{W}_1\in \R^{w\times 64}, \b{W}_2\in \R^{20\times w}, \b{W}_3\in \R^{10\times 20}$.}
    \label{chapter3:fig:fcnn_bottleneck_results}
\end{figure}

\begin{figure}[H]
    \centering
    \includegraphics[width=\textwidth,trim=0 0 0 50,clip]{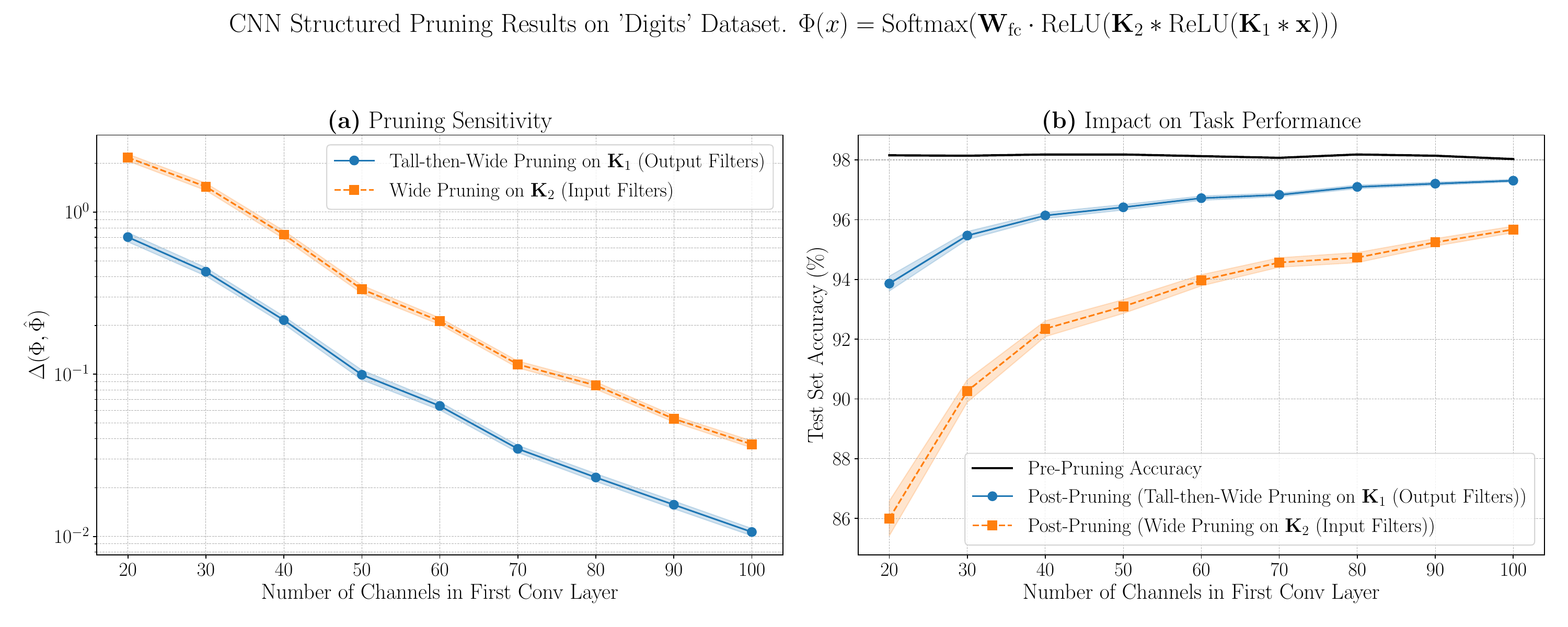}
    \caption{Evaluation plots for structured pruning on the Digits dataset using a CNN model $\Phi_w(\b{x}) = {\rm Softmax}(\b{W}_{\rm fc}{\rm ReLU}(\b{K}_2 {\rm ReLU}(\b{K}_1 \b{x})))$ with $\b{K}_1\in \R^{w\times 1\times 3\times 3}, \b{K}_2\in \R^{16 \times w \times 3 \times 3}, \b{W}_{\rm fc}\in \R^{10\times 1024}$.}
    \label{chapter3:fig:cnn_pruning_results}
\end{figure}

%\newpage
%\appendix

%\clearpage
%\FloatBarrier
%\newpage
\section{Preliminary Results}\label{chapter3:appendix:preliminary.results}

%%%%%%%%%%%%%%%%%%%%%%%%%%%%%%%%%%%%%%%%%%%%%%%%%%%%%%%%%%%%%%%%%%%%%%%%%%%%
%%%%%%%%%%%%%%%%%%%%%%%%%%%%%%%%%%%%%%%%%%%%%%%%%%%%%%%%%%%%%%%%%%%%%%%%%%%%
%%%%%%%%%%%%%%%%%%%%%%%%%%%%%%%%%%%%%%%%%%%%%%%%%%%%%%%%%%%%%%%%%%%%%%%%%%%%
\begin{lemma}\label{chapter3:lemma:0}
    Let $k\in \mathbb{Z}_{\geq 1}$,  $M\in \R_{>0}$. For $w\in [-M, M]$, it holds
    \begin{enumerate}
        \item $\E[q(w; M, k)] = w$.
        \item  $|q(w; M, k) - w| \leq  \frac{M}{k}$, almost surely.
    \end{enumerate}
\end{lemma}
\begin{proof}
    As $q(-w;M, k)=-q(w;M,k)$, we might suppose without loss of generality that $w\geq 0$. Let $\ell_w = \min \{\ell \in [k] \mid |w|\leq \ell M /k\}$. By construction, we have
    \begin{align}
        \E[q(w; M, k)] &= \frac{M}{k}\l( \ell_w \l(1-\ell_w + \frac{kw}{M}\r) + (\ell_w - 1) \l(\ell_w - \frac{kw}{M}\r) \r) \nonumber\\
        &= \frac{ M}{k} \l(\ell_w - \ell_w + \frac{kw}{M} \r) \nonumber \\
        &= w \nonumber,
    \end{align}
    which ends the proof of the first item in the lemma. To show the second item, note that
    \begin{align}
            \frac{(\ell_w-1) M}{k} \leq w \leq \frac{\ell_w M}{k} \nonumber.
    \end{align}
    Since $\frac{\ell_w M}{k} - \frac{(\ell_w-1) M}{k}  = \frac{M}{k}$, it follows that $|q(w; M, k)-w|\leq \frac{M}{k}$ almost surely. This ends the proof.
\end{proof}
%%%%%%%%%%%%%%%%%%%%%%%%%%%%%%%%%%%%%%%%%%%%%%%%%%%%%%%%%%%%%%%%%%%%%%%%%%%%

%%%%%%%%%%%%%%%%%%%%%%%%%%%%%%%%%%%%%%%%%%%%%%%%%%%%%%%%%%%%%%%%%%%%%%%%%%%%
%%%%%%%%%%%%%%%%%%%%%%%%%%%%%%%%%%%%%%%%%%%%%%%%%%%%%%%%%%%%%%%%%%%%%%%%%%%%
%%%%%%%%%%%%%%%%%%%%%%%%%%%%%%%%%%%%%%%%%%%%%%%%%%%%%%%%%%%%%%%%%%%%%%%%%%%%

\begin{lemma}\label{chapter3:lemma:1}
    Let $\b{u}, \b{v} \in \R^n$, and $\kappa \in [\|\b{u}\|, \infty]$. Then $\|\b{u} - [\b{v}]_{\kappa}\| \leq \|\b{u} - \b{v}\|$.
\end{lemma}
\begin{proof}
    Note that $[\b{u}]_{\kappa} = \b{u}$. Since the projection operator $\b{v}\mapsto [\b{v}]_{\kappa}$ is $1$-Lipschitz, it follows that
    \begin{align*}
        \|\b{u} - [\b{v}]_{\kappa}\| &= \|[\b{u}]_{\kappa} - [\b{v}]_{\kappa}\| \leq \|\b{u} - \b{v}\|,
    \end{align*}
    which readily yields the result of the lemma.
\end{proof}

%\begin{lemma}\label{chapter3:lemma:1}
    %Let $\b{u}, \b{v} \in \R^n$, and $\kappa \in [2\|u\|, \infty]$.Then $\|\b{u} - [\b{v}]_{\kappa}\| \leq \|\b{u} - \b{v}\|$.
%\end{lemma}
%%%%%%%%%%%%%%%%%%%%%%%%%%%%%%%%%%%%%%%%%%%%%%%%%%%%%%%%%%%%%%%%%%%%%%%%%%%%
%\begin{proof}
    %If  $\|\b{v}\|\leq \kappa$ then the result of the lemma is immediate. %Suppose now that $\|\b{v}\|>\kappa$.  Note that $[\b{v}]_{\kappa} = %\frac{\kappa}{\|\b{v}\|} \b{v}$. Therefore
    %\begin{align}
        %\|\b{u} - [\b{v}]_{\kappa}\| &= \l\|\b{u} - \frac{\kappa}%{\|\b{v}\|}\b{v}\r\| \nonumber\\
        %&= \l\| \frac{\kappa}{\|\b{v}\|} (\b{u} - \b{v})  + \l(1- %\frac{\kappa}{\|\b{v}\|} \r) \b{u} \r\| \nonumber\\
        %&\leq \frac{\kappa}{\|\b{v}\|} \|\b{u} - \b{v}\| + \l(1- %\frac{\kappa}{\|\b{v}\|}\r)\|\b{u}\|.\nonumber
    %\end{align}
    %As $\|\b{v}\|> \kappa \geq 2\|\b{u}\|$, it follows that $\|\b{u} - %\b{v}\| \geq \|\b{v}\|-\|\b{u}\| \geq  \|\b{u}\|$. Using the latter %in the above yields
    %\begin{align}
        %\|\b{u} - [\b{v}]_{\kappa}\| &\leq \frac{\kappa}{\|\b{v}\|} %\|\b{u} - \b{v}\| + \l(1- \frac{\kappa}{\|\b{v}\|}\r)\|\b{u} - %\b{v}\| = \|\b{u} - \b{v}\|,\nonumber
%    \end{align}
%    which ends the proof of the lemma.
%\end{proof}
%%%%%%%%%%%%%%%%%%%%%%%%%%%%%%%%%%%%%%%%%%%%%%%%%%%%%%%%%%%%%%%%%%%%%%%%%%%%

%%%%%%%%%%%%%%%%%%%%%%%%%%%%%%%%%%%%%%%%%%%%%%%%%%%%%%%%%%%%%%%%%%%%%%%%%%%%
%%%%%%%%%%%%%%%%%%%%%%%%%%%%%%%%%%%%%%%%%%%%%%%%%%%%%%%%%%%%%%%%%%%%%%%%%%%%
%%%%%%%%%%%%%%%%%%%%%%%%%%%%%%%%%%%%%%%%%%%%%%%%%%%%%%%%%%%%%%%%%%%%%%%%%%%%
\begin{lemma}\label{chapter3:lemma:2}
    Suppose $(a_k)_{k \in \Sm}$ is a sequence of nonnegative reals indexed by a finite set $\Sm$, and let $\mu\in \R$ satisfy $\mu \geq \sum_{i\in \Sm} a_i / |\Sm|$. For $\eta \in \R_{> 0}$, let $\Sm(\eta) = \{i\in \Sm \mid a_i \leq \eta \mu\}$. Then 
    \begin{align}
        \forall \eta\in \R_{\geq 1}, \quad  |\Sm(\eta)|\geq \l(1-\frac{1}{\eta}\r)|\Sm|.\nonumber
    \end{align}
\end{lemma}
%%%%%%%%%%%%%%%%%%%%%%%%%%%%%%%%%%%%%%%%%%%%%%%%%%%%%%%%%%%%%%%%%%%%%%%%%%%%
\begin{proof}
    By way of contradiction, suppose that $|\Sm(\eta)|< (1-1/\eta)|\Sm|$ for some $\eta\in \R_{\geq 1}$. We have
    \begin{align}
        \mu &\geq \frac{1}{|\Sm|} \sum_{i\in \Sm} a_i\nonumber\\
        &\geq \frac{1}{|\Sm|} \sum_{i\in \Sm\setminus \Sm(\eta)} a_i\nonumber\\
        &> \l(1-\frac{|\Sm(\eta)|}{|\Sm|}\r) \eta\mu\nonumber\\
        &> \mu,\nonumber
    \end{align}
    which yields the result of the lemma.
\end{proof}
%%%%%%%%%%%%%%%%%%%%%%%%%%%%%%%%%%%%%%%%%%%%%%%%%%%%%%%%%%%%%%%%%%%%%%%%%%%%

%%%%%%%%%%%%%%%%%%%%%%%%%%%%%%%%%%%%%%%%%%%%%%%%%%%%%%%%%%%%%%%%%%%%%%%%%%%%
%%%%%%%%%%%%%%%%%%%%%%%%%%%%%%%%%%%%%%%%%%%%%%%%%%%%%%%%%%%%%%%%%%%%%%%%%%%%
%%%%%%%%%%%%%%%%%%%%%%%%%%%%%%%%%%%%%%%%%%%%%%%%%%%%%%%%%%%%%%%%%%%%%%%%%%%%
\begin{lemma}\label{chapter3:lemma:3}
    Let $n\in \mathbb{Z}_{\geq 1}$, and $(a_k)_{k\geq 0}$ be a sequence of nonnegative reals satisfying $|a_{k} - a_{k+1}| \leq u a_k + v$ for all $k\in [0, n-1]$, where $u ,v\in  \R_{>0}$. Then
    \begin{align}
        \forall k\in [0,n-1], \quad |a_k - a_{k+1}| &\leq e^{nu}\l(u a_0 + v\r).\nonumber
    \end{align}
\end{lemma}
%%%%%%%%%%%%%%%%%%%%%%%%%%%%%%%%%%%%%%%%%%%%%%%%%%%%%%%%%%%%%%%%%%%%%%%%%%%%
\begin{proof}
    Let $k\in [0, n-1]$. It follows that $a_{k+1} \leq (u+1)a_k + v $. Iterating the latter inequality we obtain for all $k\in [n]$
    \begin{align}
        a_{k} &\leq (u+1)^k a_0 + v\sum_{\ell=0}^{k-1} (u+1)^\ell\nonumber\\
        &= (u+1)^k a_0 + \frac{v}{u}((u+1)^k - 1)\nonumber\\
        &\leq (u+1)^n a_0 + \frac{v}{u}((u+1)^n - 1)\nonumber\\
        &\leq e^{nu} a_0 + \frac{v}{u}(e^{nu}-1),\nonumber
    \end{align}
    where we used the inequality  $1+x \leq e^x$ in the last line. Plugging the above in $|a_k - a_{k+1}|\leq u a_k + v$ yields for $k\in [0, n-1]$
\begin{align}
        |a_{k}-a_{k+1}| &\leq u\l(e^{nu} a_0 + \frac{v}{u}(e^{nu}-1)\r) + v\nonumber\\
        &=e^{nu} \l( ua_0 + v \r),\nonumber
\end{align}
 which concludes the proof of the lemma.
\end{proof}
%%%%%%%%%%%%%%%%%%%%%%%%%%%%%%%%%%%%%%%%%%%%%%%%%%%%%%%%%%%%%%%%%%%%%%%%%%%%

%%%%%%%%%%%%%%%%%%%%%%%%%%%%%%%%%%%%%%%%%%%%%%%%%%%%%%%%%%%%%%%%%%%%%%%%%%%%
%%%%%%%%%%%%%%%%%%%%%%%%%%%%%%%%%%%%%%%%%%%%%%%%%%%%%%%%%%%%%%%%%%%%%%%%%%%%
%%%%%%%%%%%%%%%%%%%%%%%%%%%%%%%%%%%%%%%%%%%%%%%%%%%%%%%%%%%%%%%%%%%%%%%%%%%%
\begin{lemma}\label{chapter3:lemma:4}
    Let $(f_n)_{n\geq 1}$ be a sequence of functions with  $f_n: \R^n \to \R_{\geq 0}$. Let $ p\in (0, 1)$ and $\b{b}_n = \b{ Ber}^n_p$. Moreover, suppose  $\E_{\b{b}_n}[f_n(\b{b}_n)] \leq 1$ for all $n\in \mathbb{Z}_{\geq 1}$.  Given $\gamma, \varepsilon \in (0, 1)$, there exists a constant $n_0 = n_0(\gamma, p, \varepsilon)$ such that for all $n\geq n_0$, there exists a nonrandom realization $\hat{\b{b}}_n$ satisfying $f_n(\hat{\b{b}}_n) \leq 1+\varepsilon$ and $\|\hat{\b{b}}_n \|_0 \leq (1+\gamma)np$.
\end{lemma}
%%%%%%%%%%%%%%%%%%%%%%%%%%%%%%%%%%%%%%%%%%%%%%%%%%%%%%%%%%%%%%%%%%%%%%%%%%%%
\begin{proof}
    Since $f_n\geq 0$, we have
    \begin{align*}
        1\geq \E[f_n(\b{b}_n)] \geq \E\l[f_n(\b{b}_n) \biggm| \|\b{b}_n \|_0 \leq (1+\gamma)np \r] \p\l(\|\b{b}_n \|_0 \leq (1+\gamma)np\r),
    \end{align*}
    by standard concentration inequalities, we have $\p\l(\|\b{b}_n \|_0 > (1+\gamma)np\r) \leq e^{-c\gamma^2np}$, where $c$ is an explicit constant. Let $n_0=n_0(\gamma, p, \varepsilon)$ be the positive constant satisfying $e^{-c\gamma^2 n_0 p} = \varepsilon/(1+\varepsilon)$. It follows that for $n\geq n_0$ we have $\p\l(\|\b{b}_n \|_0 \leq (1+\gamma)np\r)\geq 1/(1+\varepsilon)$. Therefore,
    \begin{align*}
        1 \geq \frac{\E\l[f_n(\b{b}_n) \biggm| \|\b{b}_n \|_0 \leq (1+\gamma)np \r] }{1+\varepsilon}.
    \end{align*}
    Hence, there exists a realization $\hat{\b{b}}_n$ such that $\|\hat{\b{b}}_n \|_0 \leq (1+\gamma)np$ and $1 \geq \frac{f_n(\hat{\b{b}}_n)}{1+\varepsilon}$, which ends the proof.
\end{proof}

%%%%%%%%%%%%%%%%%%%%%%%%%%%%%%%%%%%%%%%%%%%%%%%%%%%%%%%%%%%%%%%%%%%%%%%%%%%%
%%%%%%%%%%%%%%%%%%%%%%%%%%%%%%%%%%%%%%%%%%%%%%%%%%%%%%%%%%%%%%%%%%%%%%%%%%%%
%%%%%%%%%%%%%%%%%%%%%%%%%%%%%%%%%%%%%%%%%%%%%%%%%%%%%%%%%%%%%%%%%%%%%%%%%%%%
\begin{lemma}\label{chapter3:lemma:5}
    Let $\b{K}\in \R^{d_{\rm out} \times d_{\rm in} \times r \times r}$, $\b{x}\in \R^{d_{\rm in}r^2}$, and $\b{z}\in \R^{d_{\rm out}r^2}$ as in Section \ref{chapter3:section.cnn.networks.representation.fcnn}. Then $\b{z} = \b{W}(\b{K}) \b{x}$.
\end{lemma}
%%%%%%%%%%%%%%%%%%%%%%%%%%%%%%%%%%%%%%%%%%%%%%%%%%%%%%%%%%%%%%%%%%%%%%%%%%%%
\begin{proof}
    Using the same notation of Section \ref{chapter3:section.cnn.networks.representation.fcnn} for $\phi$ and $\langle .\rangle_r$, we have for $o\in[d_{\rm out}]$
    \begin{align*}
        (\b{z}_o)_{\phi(u, v)} &= \b{Z}_{o, u, v}\\
        &= \sum_{i=1}^{d_{\rm in}} \sum_{a=1}^{r} \sum_{b=1}^{r} \b{K}_{o, i, a, b} \b{X}_{i, \langle u+a \rangle_r, \langle v+b \rangle_r}\\
        &= \sum_{i=1}^{d_{\rm in}} \sum_{a=1}^{r} \sum_{b=1}^{r} \b{K}_{o, i, a, b} (\b{x}_i)_{\phi(\langle u+a \rangle_r, \langle v+b \rangle_r)}\\
        &= \sum_{i=1}^{d_{\rm in}} (\mathcal{C}(\b{K}_{o, i}) \b{x}_i)_{\phi(u, v)},
    \end{align*}
    thus $\b{z}_o = \sum_{i=1}^{d_{\rm in}} \mathcal{C}(\b{K}_{o, i}) \b{x}_i$. Hence

\[
\b{z}
=
\begin{bmatrix}
\sum_{i=1}^{d_{\rm in}}\mathcal{C}(\b{K}_{1,i})\,\b{x}_i\\[2pt]
\vdots\\[2pt]
\sum_{i=1}^{d_{\rm in}}\mathcal{C}(\b{K}_{d_{\rm out},i})\,\b{x}_i
\end{bmatrix}
=
\begin{bmatrix}
\mathcal{C}(\b{K}_{1,1}) & \cdots & \mathcal{C}(\b{K}_{1,d_{\rm in}})\\
\vdots & \ddots & \vdots\\
\mathcal{C}(\b{K}_{d_{\rm out},1}) & \cdots & \mathcal{C}(\b{K}_{d_{\rm out},d_{\rm in}})
\end{bmatrix}
\begin{bmatrix}
\b{x}_1\\ \vdots\\ \b{x}_{d_{\rm in}}
\end{bmatrix}
\;=\;
\b{W}(\b{K})\b{x},
\]
which ends the proof.
\end{proof}
%%%%%%%%%%%%%%%%%%%%%%%%%%%%%%%%%%%%%%%%%%%%%%%%%%%%%%%%%%%%%%%%%%%%%%%%%%%%

%%%%%%%%%%%%%%%%%%%%%%%%%%%%%%%%%%%%%%%%%%%%%%%%%%%%%%%%%%%%%%%%%%%%%%%%%%%%
%%%%%%%%%%%%%%%%%%%%%%%%%%%%%%%%%%%%%%%%%%%%%%%%%%%%%%%%%%%%%%%%%%%%%%%%%%%%
%%%%%%%%%%%%%%%%%%%%%%%%%%%%%%%%%%%%%%%%%%%%%%%%%%%%%%%%%%%%%%%%%%%%%%%%%%%%
\begin{lemma}\label{chapter3:lemma:6}
    Let $\b{U}\in \R^{r \times r}$ as in Section \ref{chapter3:section.cnn.networks.representation.fcnn}. Then,
    \begin{align*}
        \| \mathcal{C}(\b{U}) \| \leq \|\b{U} \|_{\infty} \| \b{U}\|_{0}.
    \end{align*}
\end{lemma}
%%%%%%%%%%%%%%%%%%%%%%%%%%%%%%%%%%%%%%%%%%%%%%%%%%%%%%%%%%%%%%%%%%%%%%%%%%%%
\begin{proof}
    We use the same notation of Section \ref{chapter3:section.cnn.networks.representation.fcnn} for $\phi$ and $\langle.\rangle_r$.  Let $\omega_r=\exp(-2\pi i/r)$. We will show that the eigenvalues of $\mathcal{C}(\b{U})$ are given by
    \begin{align*}
        \hat{\b{U}}_{k, \ell} = \sum_{a=1}^{r} \sum_{b=1}^{r} \b{U}_{a, b} \omega_r^{ka +\ell b}, \quad (k, \ell) \in [r] \times [r].
    \end{align*}
    Fix $(k, \ell) \in [r]\times [r]$, and let $\b{v}^{k, \ell} \in \mathbb{C}^{r^2}$ be a vector given by $\b{v}^{k, \ell}_{\phi(u, v)} = \omega_r^{ku + \ell v}$. We have for $(u, v) \in [r]\times [r]$,
    \begin{align*}
        (\mathcal{C}(\b{U}) \b{v}^{k,\ell})_{\phi(u, v)} &= \sum_{a=1}^{r}\sum_{b=1}^{r} \b{U}_{a, b} \b{v}^{k, \ell}_{\phi(\langle  u+a\rangle_r, \langle v+b  \rangle_r)}\\
        &=\sum_{a=1}^{r}\sum_{b=1}^{r} \b{U}_{a, b} \omega_r^{k(u+a) + \ell (v+b)}\\
        &= \omega_r^{ku + \ell v} \sum_{a=1}^{r}\sum_{b=1}^{r} \b{U}_{a, b} \omega_r^{ka + \ell b}\\
        &= \b{v}^{k, \ell}_{\phi(u, v)} \hat{\b{U}}_{k, \ell}.
    \end{align*}
    Therefore, $\b{v}^{k, \ell}$ is an eigenvector of $\mathcal{C}(\b{U})$ with eigenvalue $\hat{\b{U}}_{k, \ell}$. Moreover, the eigenvectors $\b{v}^{k, \ell}_{\phi(u, v)}$ form an orthogonal basis. Indeed, let $(k, \ell), (k', \ell')\in [r]\times [r]$. We have
    \begin{align*}
        \sum_{u=1}^{r}\sum_{v=1}^{r} \b{v}^{k, \ell}_{\phi(u, v)} \bar{\b{v}}^{k', \ell'}_{\phi(u, v)}        &= \sum_{u=1}^{r}\sum_{v=1}^{r} \omega_r^{(k-k')u + (\ell - \ell')v}\\
        &= \l(\sum_{u=1}^{r} \omega_r^{(k-k')u} \r) \l(\sum_{v=1}^{r} \omega_r^{(\ell-\ell')v}\r)\\
        &= r^2 1_{k=k'}1_{\ell = \ell'}.
    \end{align*}
    Hence, the eigenvalues of $\mathcal{C}(\b{U})$ are $\{\hat{\b{U}}_{k, \ell}, (k, \ell) \in [r]\times [r]\}$. Note then that
    \begin{align*}
        |\hat{\b{U}}_{k, \ell}| &\leq \sum_{a=1}^{r} \sum_{b=1}^{r} \l|\b{U}_{a, b}\r|\leq \|\b{U}\|_{\infty} \|\b{U}\|_0.
    \end{align*}
    Therefore
    \begin{align*}
        \|\mathcal{C}(\b{U})\| &= \max_{(k, \ell) \in [r] \times [r]} |\hat{\b{U}}_{k, \ell}|\leq \|\b{U}\|_{\infty} \|\b{U}\|_0,
    \end{align*}
    which concludes the proof.
\end{proof}
%%%%%%%%%%%%%%%%%%%%%%%%%%%%%%%%%%%%%%%%%%%%%%%%%%%%%%%%%%%%%%%%%%%%%%%%%%%%

%%%%%%%%%%%%%%%%%%%%%%%%%%%%%%%%%%%%%%%%%%%%%%%%%%%%%%%%%%%%%%%%%%%%%%%%%%%%
%%%%%%%%%%%%%%%%%%%%%%%%%%%%%%%%%%%%%%%%%%%%%%%%%%%%%%%%%%%%%%%%%%%%%%%%%%%%
%%%%%%%%%%%%%%%%%%%%%%%%%%%%%%%%%%%%%%%%%%%%%%%%%%%%%%%%%%%%%%%%%%%%%%%%%%%%
\begin{lemma}\label{chapter3:lemma:7}
Let $\rho, \theta \in \R_{>0}$,   $\b{W}_1\in \R^{n_2\times n_1}$, $\b{W}_2 \in \R^{n_3 \times n_2}$ be random matrices, and $\mathcal{Q}_{\rho, \theta}$ be a joint distribution over $\mathbb{B}_{2}^{n_1}(\rho)\times \mathbb{B}^{n_1}_{2}(\theta)$, and $(\b{x}, \b{x}')\sim \mathcal{Q}_{\rho, \theta}$. Suppose $f: \R^{n_1} \to \R^{n}$ for some $n\in \Z_{\geq 1}$ is a function with parameters $\b{A}\in \R^{n_2\times n_1}, \b{B}\in\R^{n_3\times n_2}$. That is, $f(\b{x}) = f(\b{x}; \b{A}, \b{B})$. Introduce
\begin{align}
    \Psi_f: \quad \R^{n_2\times n_1} \times \R^{n_3 \times n_2} \to \R, \quad (\b{A}, \b{B}) \mapsto \E\l[\|f(\b{x}; \b{W}_1, \b{W}_2) - f(\b{x}'; \b{A}, \b{B})\|^2\r],\nonumber
\end{align}
and the expectation is taken over $(\b{x}, \b{x}')\sim \mathcal{Q}_{\rho, \theta}$ and $\b{W}_1, \b{W}_2$. Suppose $\mathcal{S}_1 \cup \mathcal{S}_2 = [n_2]\times [n_1]$ is a partition of $[n_2]\times [n_1]$ with $\Sm_1 \neq \emptyset$, and $(h_{ij})_{(i, j) \in [n_2] \times [n_1]}$ is a collection of random variables.  Let $\b{Z}\in \R^{n_2 \times n_1}$ be given by
    \begin{alignat*}{2}
        \b{Z}_{ij} &= [\b{W}_1]_{ij}, \quad &&\text{if } (i,j) \in \mathcal{S}_1,\nonumber\\
        \b{Z}_{ij} &= h_{ij}, \quad &&\text{if } (i,j) \in \mathcal{S}_2.\nonumber
    \end{alignat*}
    Let $\mathcal{T} \triangleq \{h_{ij} \mid (i, j)\in \Sm_2 \} $, and $\sigma_i = \| \b{W}_i\|,  \nu_i = \|\b{W}_i\|_{\infty}$ for $i=1,2$. Suppose  $(t_{ij})_{(i,j)\in \Sm_1} $ are random variables. 
    \begin{enumerate}
        \item \textbf{Single Layer}. Suppose $f(\b{x}; \b{A}, \b{B}) = \b{A}\b{x}$, and write $f(\b{x}; \b{A}), \Psi_f(\b{A})$ for simplicity as $\b{B}$ has no effect on $f$ or $\Psi_f$. Then, for all $\eta \geq 1$ there exists at least $\l(1- \frac{1}{\eta}\r)|\Sm_1|$ many pairs $(i, j)$ in $\Sm_1$ such that the following holds
    \begin{align}
         \E\l[ \l|\frac{\partial^2 \Psi_f(\b{Z}(t_{ij}; i, j))}{\partial t_{ij}^2} \r| \r] &\leq \frac{2\eta \theta^2 n_2}{|\Sm_1|}. \label{chapter3:lemma:7.1} 
    \end{align}
         where the expectation is taken over the randomness in $h_{ij}, t_{ij}, \b{W}_1$ and $\b{W}_2$.
        \item \textbf{Block of Two Layers}. Suppose $f(\b{x}; \b{A}, \b{B}) = \b{B}\varphi(\b{A}\b{x})$, where $\varphi$ is a function satisfying (\ref{chapter3:assumption:1.1}), (\ref{chapter3:assumption:1.2}), and (\ref{chapter3:assumption:1.3}) in Assumption \ref{chapter3:assumption:1}. Furthermore, suppose that $|t_{ij} -[\b{W}_1]_{ij}|\leq \tau \nu_1$  where $\tau \in \R_{\geq 0}$.  Then, for all $\eta \geq 1$ there exists at least $\l(1- \frac{1}{\eta}\r)|\Sm_1|$ many pairs $(i, j)$ in $\Sm_1$ such that the following holds
            \begin{align}
         \E&\l[ \l|\frac{\partial^2 \Psi_f(\b{Z}(t_{ij}; i, j), \b{W}_2)}{\partial t_{ij}^2} \r| \r] \leq \frac{2\eta \theta^2 \sigma_2 \sqrt{n_2}}{|\Sm_1|} \nonumber\\
         &\times \l( \sqrt{\E_{\T} \l[ \Psi_f(\b{Z}, \b{W}_2) \r]} + \frac{\nu_2^2 n_3\sqrt{n_2}}{\sigma_2}  + \theta \tau  \nu_1 \nu_2 \sqrt{n_2 n_3}\r). \label{chapter3:lemma:7.2}
    \end{align}
        where the expectation is taken over the randomness in $h_{ij}, t_{ij}, \b{W}_1$ and $ \b{W}_2$.
    \end{enumerate}
\end{lemma}
%%%%%%%%%%%%%%%%%%%%%%%%%%%%%%%%%%%%%%%%%%%%%%%%%%%%%%%%%%%%%%%%%%%%%%%%%%%%
\begin{proof}
We first show (\ref{chapter3:lemma:7.1}). Note that for $\ell \in [n_2]$
\begin{align}
     \frac{\partial f_{\ell}(\b{x}', \b{Z}(t;i,j))}{\partial t} &= \b{x}'_j 1_{\ell=i}, \quad \text{and} \quad \frac{\partial^2 f_{\ell}(\b{x}'; \b{Z}(t;i,j))}{\partial t^2} = 0.\nonumber
\end{align}
Using the standard Leibniz integral rule twice and the above, it follows that 
\begin{align}
    \frac{\partial^2 \Psi_f(\b{Z}(t;i,j))}{\partial t^2} &= 2 \E\l[  \l\| \frac{\partial f(\b{x}'; \b{Z}(t;i,j))}{\partial t}\r\|^2 \r]=2\E\l[   \l(\b{x}'_j\r)^2 \r].\nonumber
\end{align}
Setting $t=t_{ij}$, taking the expectation over $\{t_{ij}\}\cup \T$ and summing over all $t_{ij} \in \Sm_1$ yields
\begin{align}
    \sum_{(i, j) \in \Sm_1} \E_{\{t_{ij}\} \cup \T}\l[\l| \frac{\partial^2 \Psi_f(\b{Z}(t_{ij};i,j))}{\partial t_{ij}^2} \r|\r] & \leq 2\sum_{(i, j) \in \Sm_1} \E\l[(\b{x}'_j)^2\r]\nonumber\\
    &\leq 2\sum_{(i, j) \in [n_2] \times [n_1]} \E\l[(\b{x}'_j)^2\r]\nonumber\\
    &= 2n_2\E\l[   \|\b{x}'\|^2 \r] \nonumber\\
    &\leq 2n_2 \theta^2.\nonumber
\end{align}
The result of (\ref{chapter3:lemma:7.1}) follows  readily by applying Lemma \ref{chapter3:lemma:2} to the sequence $a_{ij}$ defined as $a_{ij} \triangleq \E_{\{t_{ij}\} \cup \T}\l[\l|\frac{\partial^2 \Psi_f(\b{Z}({t}_{ij};i,j))}{\partial t_{ij}^2 }\r|\r]$. We now show (\ref{chapter3:lemma:7.2}). Similarly to the previous case, we start by deriving the  expressions of the derivatives of $t\mapsto f(\b{x}', \b{Z}(t;i, j), \b{W}_2)$. We have for $\ell \in [n_3]$
\begin{align}
     \frac{\partial f_{\ell}(\b{x}'; \b{Z}(t;i,j), \b{W}_2)}{\partial t} &= [\b{W}_2]_{\ell i} \varphi^{(1)}(\b{Z}(t;i,j)\b{x}')_i \b{x}'_j,\nonumber\\
     \frac{\partial^2 f_{\ell}(\b{x}'; \b{Z}(t;i,j), \b{W}_2)}{\partial t^2} &= [\b{W}_2]_{\ell i} \varphi^{(2)}(\b{Z}(t;i,j)\b{x}')_i \l(\b{x}'_j\r)^2.\nonumber
\end{align}
Using the standard Leibniz integral rule twice and the above, it follows that 
\begin{align}
    \frac{\partial^2 \Psi_f(\b{Z}(t;i,j), \b{W}_2)}{\partial t^2} & = 2\E\l[\l\|\frac{\partial f(\b{x}'; \b{Z}(t;i,j), \b{W}_2)}{\partial t}\r\|^2\r] \nonumber \\
    &+ 2\E\l[ \l(f(\b{x}'; \b{Z}(t; i, j), \b{W}_2) - f(\b{x}; \b{W}_1, \b{W}_2)\r)^\top\frac{\partial^2 f(\b{x}'; \b{Z}(t;i,j), \b{W}_2)}{\partial t^2}\r] \nonumber \\
    &= 2\sum_{\ell\in[n_3]}\E\l[ \l( [\b{W}_2]_{\ell i}\varphi^{(1)}(\b{Z}(t; i, j) \b{x}')_i \b{x}'_j  \r)^2\r]\nonumber\\
     &+2\sum_{\ell\in[n_3]}\E\l[\bigg([\b{W}_2]_{\ell,:} \varphi(\b{Z}(t;i,j) \b{x}') - [\b{W}_2]_{\ell, :}\varphi(\b{W}_1\b{x})\bigg)  \r. \nonumber\\
     &\l. \times[\b{W}_2]_{\ell i}\varphi^{(2)}(\b{Z}(t; i, j) \b{x}')_i \l( \b{x}'_j \r)^2  \r]\nonumber\\
     &= H_1(t) + H_2(t).  \label{chapter3:lemma:7.3}
\end{align}
Setting $t=t_{ij}$ and using (\ref{chapter3:assumption:1.2}), we obtain
\begin{align}
    \l|H_1(t_{ij})\r| &\leq 2\nu_2^2 n_3  \E\l[\l(\b{x}'_j\r)^2\r].\nonumber
\end{align}
Taking the expectation of the above over $\{t_{ij}\}\cup \T$ and $\b{Z}$, and summing over all $(i,j) \in \Sm_1$ yields
\begin{align}
    \sum_{(i,j) \in \Sm_1} \E\l[|H_1(t_{ij})|\r]  &\leq 2 \nu_2^2 n_2 n_3 \E[\|\b{x}'\|^2]\nonumber\\
    &\leq 2 \theta^2 \nu_2^2 n_2 n_3. \label{chapter3:lemma:7.H1}
\end{align}
We next bound $|H_2(t_{ij})|$. We have by the triangle inequality 
\begin{align}
    |H_2(t_{ij})| &\leq 2\E\l[\underbrace{\bigg| \sum_{\ell\in[n_3]} \bigg( [\b{W}_2]_{\ell,:} \varphi(\b{Z}(t_{ij};i,j) \b{x}') - [\b{W}_2]_{\ell, :}\varphi(\b{Z}\b{x}') \bigg)  [\b{W}_2]_{\ell i} \bigg|}_{H_3(t_{ij})}  \l( \b{x}'_j \r)^2  \r]\nonumber\\
    &+2\E\l[\underbrace{\bigg| \sum_{\ell\in[n_3]} \bigg( [\b{W}_2]_{\ell,:} \varphi(\b{Z} \b{x}') - [\b{W}_2]_{\ell, :}\varphi(\b{W}_1\b{x}) \bigg) [\b{W}_2]_{\ell i} \bigg| \l( \b{x}'_j \r)^2  }_{H_4(i, j)}  \r].\nonumber
\end{align}
Using the Cauchy-Schwarz inequality, we have 
\begin{align}
    H_3(t_{ij}) &\leq \|\b{W}_2 \varphi(\b{Z}(t_{ij}; i, j)\b{x}') - \b{W}_2 \varphi(\b{Z}\b{x}')\| \l(\sum_{\ell \in [n_3]} ([\b{W}_2]_{\ell i})^2 \r)^{\frac{1}{2}}\nonumber\\
    &\leq \sigma_2  |t_{ij} - [\b{W}_1]_{ij}| |\b{x}_j'| \nu_2\sqrt{n_3} \nonumber\\
    &\leq \theta \tau \sigma_2   \nu_1 \nu_2\sqrt{n_3}. \nonumber
\end{align}
Therefore
\begin{align}
    \sum_{(i, j) \in \Sm_1} \E\l[H_3(t_{ij}) \l( \b{x}'_j \r)^2  \r] &\leq  \theta \tau \sigma_2   \nu_1 \nu_2 n_2 \sqrt{n_3}\E[\|\b{x}'\|^2]\nonumber\\
    &\leq  \theta^3 \tau \sigma_2   \nu_1 \nu_2 n_2 \sqrt{n_3}.  \label{chapter3:lemma:7.H3}
\end{align}
We next bound $\sum_{(i,j) \in \Sm_1} \E[H_4(i, j)]$. We have 
\begin{align}
     &\sum_{(i,j) \in \Sm_1} \E\l[H_4(i, j)\r]  \nonumber \\
     &= \sum_{(i,j)\in \Sm_1}\E\l[\bigg| \sum_{\ell\in[n_3]} \bigg( [\b{W}_2]_{\ell,:} \varphi(\b{Z} \b{x}') - [\b{W}_2]_{\ell, :}\varphi(\b{W}_1\b{x})  \bigg) [\b{W}_2]_{\ell i} \bigg| \l( \b{x}'_j \r)^2  \r]\nonumber\\
     &\leq \sum_{i\in[n_2]} \E\l[  \bigg|  \sum_{\ell\in[n_3]}  \bigg([\b{W}_2]_{\ell,:} \varphi(\b{Z}\b{x}') - [\b{W}_2]_{\ell, :}\varphi(\b{W}_1\b{x})  \bigg)  [\b{W}_2]_{\ell i}\bigg| \|\b{x}'\|^2 \r]\nonumber\\
     &\leq  \theta^2\sqrt{n_2} \E\l[ \|\b{W}_2^{\top} \l(\b{W}_2 \varphi(\b{Z}\b{x}') - \b{W}_2 \varphi(\b{W}_1 \b{x}) \r)\| \r]  \label{chapter3:lemma:7.5}\\
     &\leq \theta^2 \sigma_2 \sqrt{n_2} \E\l[ \|\b{W}_2 \varphi(\b{Z}\b{x}') - \b{W}_2 \varphi(\b{W}_1 \b{x}) \| \r] \nonumber\\
     &\leq  \theta^2 \sigma_2 \sqrt{n_2}   \sqrt{\E_{\T} \l[ \Psi_f(\b{Z}, \b{W}_2) \r]},  \label{chapter3:lemma:7.H4}
\end{align}
where we used the Cauchy-Schwarz inequality in (\ref{chapter3:lemma:7.5}), and Jensen's inequality in (\ref{chapter3:lemma:7.H4}). Combining (\ref{chapter3:lemma:7.H4}) with (\ref{chapter3:lemma:7.H3}) and (\ref{chapter3:lemma:7.H1}) we obtain
\begin{align}
    &\sum_{(i, j) \in \Sm_1} \E\l[\l|  \frac{\partial^2 \Psi_f(\b{Z}(t_{ij};i,j), \b{W}_2)}{\partial t_{ij}^2} \r|\r] \nonumber\\
    &\leq 2 \theta^2 \nu_2^2 n_2 n_3  +  2\theta^3 \tau \sigma_2   \nu_1 \nu_2 n_2 \sqrt{n_3} + 2\theta^2 \sigma_2 \sqrt{n_2}   \sqrt{\E_{\T} \l[ \Psi_f(\b{Z}, \b{W}_2) \r]}\nonumber\\
    &\leq 2 \theta^2 \sigma_2 \sqrt{n_2}  \l( \sqrt{\E_{\T} \l[ \Psi_f(\b{Z}, \b{W}_2) \r]} + \frac{\nu_2^2 n_3\sqrt{n_2}}{\sigma_2}  + \theta \tau  \nu_1 \nu_2 \sqrt{n_2 n_3}\r). \nonumber
\end{align}
The result of (\ref{chapter3:lemma:7.2}) follows readily by applying Lemma \ref{chapter3:lemma:2} to the sequence $a_{ij}$ given by $a_{ij}\triangleq \E_{\{t_{ij}\}\cup \T}\l[\l|\frac{\partial^2 \Psi_f(\b{Z}({t}_{ij};i,j), \b{W}_2)}{\partial t_{ij}^2} \r|\r]$. This concludes the proof.
\end{proof}
%%%%%%%%%%%%%%%%%%%%%%%%%%%%%%%%%%%%%%%%%%%%%%%%%%%%%%%%%%%%%%%%%%%%%%%%%%%%

%%%%%%%%%%%%%%%%%%%%%%%%%%%%%%%%%%%%%%%%%%%%%%%%%%%%%%%%%%%%%%%%%%%%%%%%%%%%
%%%%%%%%%%%%%%%%%%%%%%%%%%%%%%%%%%%%%%%%%%%%%%%%%%%%%%%%%%%%%%%%%%%%%%%%%%%%
%%%%%%%%%%%%%%%%%%%%%%%%%%%%%%%%%%%%%%%%%%%%%%%%%%%%%%%%%%%%%%%%%%%%%%%%%%%%
\begin{lemma}\label{chapter3:lemma:8}
Let $\rho, \theta \in \R_{>0}$, $\b{W}_1\in \R^{n_2\times n_1}$, $\b{W}_2 \in \R^{n_3 \times n_2}$ be random matrices, and $\mathcal{Q}_{\rho, \theta}$ be a joint distribution over $\mathbb{B}_{2}^{n_1}(\rho)\times \mathbb{B}^{n_1}_{2}(\theta)$, and $(\b{x}, \b{x}')\sim \mathcal{Q}_{\rho, \theta}$. Suppose $f: \R^{n_1} \to \R^n$ for some $n\in \Z_{\geq 1}$ is a function with parameters $\b{A}\in \R^{n_2\times n_1}, \b{B}\in \R^{n_3\times n_2}$. That is, $f(\b{x})=f(\b{x}; \b{A}, \b{B})$. Introduce
\begin{align}
    \Psi_f: \quad \R^{n_2\times n_1} \times \R^{n_3 \times n_2} \to \R, \quad (\b{A}, \b{B}) \mapsto \E\l[\|f(\b{x}; \b{W}_1, \b{W}_2) - f(\b{x}'; \b{A}, \b{B})\|^2\r],\nonumber
\end{align}
where the expectation is taken over $(\b{x}, \b{x}')\sim \mathcal{Q}_{\rho, \theta}$ and $\b{W}_1, \b{W}_2$. Let $k | \gcd(n_1, n_2, n_3)$ and  $\mathcal{E}_{\ell}=\{k(\ell-1) + q \mid q\in [k]\}$ for $\ell \in \mathbb{Z}_{\geq 1}$. Let $\b{B}^{1}_{i,j} \in \R^{k\times k}$ for $(i,j) \in \l[n_2/k]\times[n_1/k\r]$ be the block matrix of $\b{W}_1$ given by $[\b{B}^{1}_{i,j}]_{u, v} = [\b{W}_1]_{k(i-1) + u, k(j-1) + v}$. Define $\b{B}^{2}_{i,j}$ similarly for $(i, j) \in [n_3/k] \times [n_2/k]$ with $\b{W}_2$. Finally, let $\sigma_i = \| \b{W}_i\|$ for $i=1,2$, and
\begin{align*}
    \nu_1 = \max_{(i,j) \in [n_2/k]\times [n_1/k]} \|\b{B}^1_{i,j}\|, \quad \nu_2 = \max_{(i,j) \in [n_3/k]\times [n_2/k]} \|\b{B}^2_{i,j}\|.
\end{align*}
\begin{enumerate}
    \item \textbf{Single Layer}. Suppose $f(\b{x}; \b{A}, \b{B}) = \b{A}\b{x}$. Write $f(\b{x}; \b{A})$ and $ \Psi_f(\b{A})$ for simplicity as $\b{B}$ has no effect on $f$ or $\Psi_f$. Let $\Sm_1 \cup \Sm_2 = [n_1/k]$ be a partition of $[n_1/k]$. Suppose $\b{U}_1\in\R^{n_2 \times n_1}$ satisfies
    \begin{align*}
        [\b{U}_1]_{:, \mathcal{E}_\ell} = [\b{W}_1]_{:, \mathcal{E}_\ell}, \quad \forall \ell \in \Sm_1.
    \end{align*}

    Then, for $\eta>1$ and a  $t\in \R$, there exists at least $(1-1/\eta)|\Sm_1|$ indices $\ell \in \Sm_1$ such that
    \begin{align}
        \E \l[ \l| \frac{\partial^2 \Psi_f (\b{U}_1 \b{G}(t; \mathcal{E}_\ell)) }{\partial t^2} \r|_{t=1} \r] &\leq \frac{2\eta \theta^2}{|\Sm_1|} \frac{\nu_1^2 n_2}{k}, \label{chapter3:lemma:8.1}
    \end{align}
    where the expectation is taken over the randomness in $\b{U}_1$, $\b{W}_1$ and $\b{W}_2$.
    \item \textbf{Block of Two Layers}. Suppose $f(\b{x}; \b{A}, \b{B}) = \b{B}\varphi(\b{A}\b{x})$, where $\varphi$ is a function satisfying (\ref{chapter3:assumption:1.1}), (\ref{chapter3:assumption:1.2}), and (\ref{chapter3:assumption:1.3}) in Assumption \ref{chapter3:assumption:1}. Let $\Sm_1 \cup \Sm_2 = [n_2/k]$ be a partition of $[n_2/k]$ and  $(t_\ell)_{\ell\in [n_2/k]}$ be a collection of random variables. Let $\tau\in \R_{>0}$ be an upper bound on the terms $|1-t_\ell|$. Suppose $\b{U}_1\in\R^{n_2 \times n_1}$ satisfies
    \begin{align*}
        [\b{U}_1]_{\mathcal{E}_\ell, :} &= [\b{W}_1]_{ \mathcal{E}_\ell, :}, \quad \forall \ell \in \Sm_1.
    \end{align*}
    Then, for $\eta>1$, there exists at least $(1-1/\eta)|\Sm_1|$ indices  $\ell \in \Sm_1$ such that
    \begin{align}
        \E \l[ \l| \frac{\partial^2 \Psi_f (\b{G}(t_\ell; \mathcal{E}_\ell)\b{U}_1, \b{W}_2) }{\partial t_\ell^2} \r| \r] &\leq \frac{2\eta \theta^2 }{|\Sm_1|}\frac{\sigma_1^2 \nu_2 \sqrt{n_3} }{\sqrt{k}} \nonumber\\
        &\times \l( \frac{\nu_2 \sqrt{n_3} + \tau \theta  \sigma_2  \nu_1 \sqrt{n_1}}{\sqrt{k}} + \sqrt{\E\l[\Psi_f(\b{U}_1, \b{W}_2)\r]} \r), \label{chapter3:lemma:8.2}
    \end{align}
    where the expectation is taken over the randomness in $t_{\ell}, \b{U}_1, \b{W}_1$, and $\b{W}_2$.
\end{enumerate}
\end{lemma}
%%%%%%%%%%%%%%%%%%%%%%%%%%%%%%%%%%%%%%%%%%%%%%%%%%%%%%%%%%%%%%%%%%%%%%%%%%%%
\begin{proof}
Using the standard Leibniz integral rule twice, we have
\begin{align}
    \frac{\partial^2 \Psi_f(\b{A}(t), \b{B}(t))}{\partial t^2} &= 2\E\l[ \l\|\frac{\partial f(\b{x}'; \b{A}(t), \b{B}(t))}{\partial t} \r\|^2 \r]\nonumber\\
    &+ 2\E\l[  \l( f(\b{x}'; \b{A}(t), \b{B}(t)) - f(\b{x}, \b{W}_1, \b{W}_2) \r)^\top \frac{\partial^2 f(\b{x}'; \b{A}(t), \b{B}(t))}{\partial t^2} \r]. \label{chapter3:lemma:8.3}
\end{align}
    \begin{enumerate}
        \item \textbf{Single Layer}. We have for $\ell\in [n_1/k]$
        \begin{align}
            \frac{\partial f(\b{x}'; \b{U}_1\b{G}(t; \mathcal{E}_\ell))}{\partial t} &= \b{U}_1 \frac{\partial \b{G}(t; \mathcal{E}_\ell)}{\partial t} \b{x}' \nonumber \\
            &= [\b{U}_1]_{:, \mathcal{E}_\ell} \b{x}_{\mathcal{E}_\ell}' \nonumber \\
            &=[\b{W}_1]_{:, \mathcal{E}_\ell} \b{x}_{\mathcal{E}_\ell}', \label{chapter3:lemma:8.4}
        \end{align}
        and
        \begin{align}
            \frac{\partial^2 f(\b{x}'; \b{U}_1\b{G}(t; \mathcal{E}_\ell))}{\partial t^2} &= \b{0}_{\R^{n_2}}. \label{chapter3:lemma:8.5}
        \end{align}
        Combining (\ref{chapter3:lemma:8.5}) and (\ref{chapter3:lemma:8.4}) in (\ref{chapter3:lemma:8.3}), we obtain
        \begin{align*}
             \l\| \frac{\partial^2 \Psi_f (\b{U}_1 \b{G}(t; \mathcal{E}_\ell)) }{\partial t^2} \r\|  &= 2\E\l[ \l\| \frac{\partial f(\b{x}'; \b{U}_1\b{G}(t; \mathcal{E}_\ell))}{\partial t} \r\|^2\r]\\
             &= 2\E \l[ \|[\b{W}_1]_{:, \mathcal{E}_\ell} \b{x}_{\mathcal{E}_\ell}' \|^2 \r]\\
             &= 2\E\l[ \sum_{i=1}^{n_2/k} \| \b{B}^{1}_{i, \ell} \b{x}_{\mathcal{E}_\ell}'\|^2 \r]\\
             &\leq \frac{2\nu_1^2 n_2}{k} \E\l[\|\b{x}_{\mathcal{E}_\ell}'\|^2\r].
        \end{align*}
        Taking the expectation of the above over $\b{U}_1, \b{W}_1$ and summing over $\ell\in \Sm_1$, we obtain
        \begin{align*}
            \sum_{\ell \in \Sm_1} \E \l[\l| \frac{\partial^2 \Psi_f (\b{U}_1 \b{G}(t; \mathcal{E}_\ell)) }{\partial t^2} \r|_{t=1} \r] &\leq \frac{2\nu_1^2 n_2}{k} \E\l[\|\b{x}'\|^2\r]\\
            &\leq \frac{ 2\theta^2 \nu_1^2 n_2}{k}.
        \end{align*}
        Using the above inequality, the result of (\ref{chapter3:lemma:8.1}) follows by applying Lemma \ref{chapter3:lemma:2} to the sequence $a_{\ell} \triangleq\E \l[\l| \frac{\partial^2 \Psi_f (\b{U}_1 \b{G}(t; \mathcal{E}_\ell)) }{\partial t^2} \r|_{t=1}\r]$. 
        \item \textbf{Block of Two Layers}. We have for $\ell\in [n_2/k]$
        \begin{align}
            \frac{\partial f(\b{x}'; \b{G}(t_\ell; \mathcal{E}_\ell)\b{U}_1, \b{W}_2 )}{\partial t_\ell} &= \b{W}_2  \b{J}_{\varphi}(\b{G}(t_\ell;\mathcal{E}_\ell)\b{U}_1\b{x'}) \frac{\partial \b{G}(t_\ell; \mathcal{E}_\ell)}{\partial t_\ell} \b{U}_1 \b{x}' \nonumber \\
            &= \b{W}_2 \l(\varphi^{(1)}(t_\ell \b{U}_1 \b{x}') \odot  \b{U}_1 \b{x}' \odot \b{e}_{\mathcal{E}_\ell} \r), \label{chapter3:lemma:8.6}
        \end{align}
        where $\b{e}_{\mathcal{E}_\ell}\in \R^{n_2}$ is given by $(\b{e}_{\mathcal{E}_\ell})_j = 1_{j \in \mathcal{E}_\ell}$. Similarly, we have
        \begin{align}
            \frac{\partial^2 f(\b{x}'; \b{G}(t_\ell; \mathcal{E}_\ell)\b{U}_1, \b{W}_2)}{\partial t_\ell^2} &= \b{W}_2 \l( \varphi^{(2)}(t_\ell \b{U}_1 \b{x}') \odot  (\b{U}_1 \b{x}')^2 \odot \b{e}_{\mathcal{E}_\ell} \r), \label{chapter3:lemma:8.7}
        \end{align}
        where we used the notation $(\b{U}_1 \b{x}')^2 = (\b{U}_1 \b{x}') \odot (\b{U}_1 \b{x}')$.  We then have
        \begin{align*}
            \l\|\frac{\partial f(\b{x}'; \b{G}(t_\ell; \mathcal{E}_\ell)\b{U}_1, \b{W}_2)}{\partial t_\ell} \r\|^2 &\leq \l\|\b{W}_2 \l( \varphi^{(1)}(t_\ell \b{U}_1 \b{x}') \odot  \b{U}_1 \b{x}' \odot \b{e}_{\mathcal{E}_\ell} \r) \r\|^2 \\
            &= \sum_{i \in [n_3/k]} \l\| [\b{W}_2]_{\mathcal{E}_i, \mathcal{E}_\ell} \l[\varphi^{(1)}(t_\ell \b{U}_1 \b{x}') \odot  \b{U}_1 \b{x}'\r]_{\mathcal{E}_\ell} \r\|^2\\
            &= \sum_{i \in [n_3/k]} \l\| \b{B}^2_{i, \ell} \l[\varphi^{(1)}(t_\ell \b{U}_1 \b{x}') \odot  \b{U}_1 \b{x}'\r]_{\mathcal{E}_\ell} \r\|^2\\
            &\leq  \frac{\nu_2^2 n_3}{k} \l\| [\b{U}_1 \b{x}' ]_{\mathcal{E}_\ell} \r\|^2\\
            &=  \frac{\nu_2^2 n_3}{k} \l\| [\b{W}_1 \b{x}' ]_{\mathcal{E}_\ell} \r\|^2,
        \end{align*}
        where we used $[\b{U}_1 \b{x}' ]_{\mathcal{E}_\ell} = \sum_{j} [\b{U}_1]_{\mathcal{E}_\ell, \mathcal{E}_j } \b{x}_{\mathcal{E}_j}' = \sum_{j} [\b{W}_1]_{\mathcal{E}_\ell, \mathcal{E}_j } \b{x}_{\mathcal{E}_j}'  = [\b{W}_1 \b{x}']_{\mathcal{E}_\ell}$ in the last line.  Taking the expectation over $t_\ell, \b{U}_1, \b{W}_1, \b{W}_2$ and summing over $\ell\in \Sm_1$, we obtain
        \begin{align}
            \sum_{\ell \in \Sm_1} \E \l[ \l\|\frac{\partial f(\b{x}'; \b{G}(t_\ell; \mathcal{E}_\ell)\b{U}_1, \b{W}_2 )}{\partial t_\ell} \r\|^2\r] &\leq  \frac{\nu_2^2 n_3}{k} \E\l[ \|\b{W}_1 \b{x}'\|^2 \r] \nonumber \\
            &\leq \frac{(\sigma_1 \nu_2 )^2 n_3}{k} \E\l[ \|\b{x}'\|^2 \r] \nonumber \\
            &\leq \frac{\theta^2 (\sigma_1 \nu_2 )^2 n_3}{k}. \label{chapter3:lemma:8.8}
        \end{align}
        On the other hand, we have
        \begin{align*}
            &\E\l[\l| \l( f(\b{x}; \b{W}_1, \b{W}_2) - f(\b{x}'; \b{G}(t_\ell; \mathcal{E}_\ell)\b{U}_1, \b{W}_2) \r)^\top \frac{\partial^2 f(\b{x}'; \b{G}(t_\ell; \mathcal{E}_\ell)\b{U}_1, \b{W}_2 )}{\partial t_\ell^2}\r| \r] \\
            &\leq \E\l[ \l| \l( f(\b{x}; \b{W}_1, \b{W}_2) - f(\b{x}'; \b{U}_1, \b{W}_2) \r)^\top \frac{\partial^2 f(\b{x}'; \b{G}(t_\ell; \mathcal{E}_\ell)\b{U}_1, \b{W}_2 )}{\partial t_\ell^2}\r| \r]\\
            &+ \E\l[\l| \l( f(\b{x}'; \b{U}_1, \b{W}_2) - f(\b{x}'; \b{G}(t_\ell; \mathcal{E}_\ell)\b{U}_1, \b{W}_2) \r)^\top \frac{\partial^2 f(\b{x}'; \b{G}(t_\ell; \mathcal{E}_\ell)\b{U}_1, \b{W}_2 )}{\partial t_\ell^2}\r|\r]\\
            &\leq \E\l[ \l\|  f(\b{x}; \b{W}_1, \b{W}_2) - f(\b{x}'; \b{U}_1, \b{W}_2) \r\| \l\| \frac{\partial^2 f(\b{x}'; \b{G}(t_\ell; \mathcal{E}_\ell)\b{U}_1, \b{W}_2 )}{\partial t_\ell^2}  \r\| \r]\\
            &+ \E\l[\l\| f(\b{x}'; \b{U}_1, \b{W}_2) - f(\b{x}'; \b{G}(t_\ell; \mathcal{E}_\ell)\b{U}_1, \b{W}_2) \r\| \l\| \frac{\partial^2 f(\b{x}'; \b{G}(t_\ell; \mathcal{E}_\ell)\b{U}_1, \b{W}_2 )}{\partial t_\ell^2}\r \|\r]\\
            &= H^1_\ell + H^2_\ell.
        \end{align*}
        We next bound $\sum_{\ell \in \Sm_1} H_\ell^1$. We have
        \begin{align*}
            \l\|\frac{\partial^2 f(\b{x}'; \b{G}(t_\ell; \mathcal{E}_\ell)\b{U}_1, \b{W}_2 )}{\partial t_\ell^2} \r\| &= \l\| \b{W}_2 \l(\varphi^{(2)}(t_\ell \b{U}_1 \b{x}') \odot  (\b{U}_1 \b{x}')^2 \odot \b{e}_{\mathcal{E}_\ell} \r) \r\| \\
            &= \sqrt{\sum_{i \in [n_3/k]} \l\| \b{B}^2_{i, \ell} \l[\varphi^{(2)}(t_\ell\b{U}_1 \b{x}') \odot  (\b{U}_1 \b{x}')^2 \r]_{\mathcal{E}_\ell} \r\|^2 }\\
            &\leq  \sqrt{\frac{\nu_2^2 n_3}{k} \l\| \l([\b{U}_1 \b{x}' ]_{\mathcal{E}_\ell}\r)^2 \r\|^2}\\
            &= \frac{\nu_2 \sqrt{n_3}}{\sqrt{k}} \l\| \l([\b{W}_1 \b{x}' ]_{\mathcal{E}_\ell}\r)^2 \r\|\\
            &\leq \frac{\nu_2 \sqrt{n_3}}{\sqrt{k}} \l\|[\b{W}_1 \b{x}' ]_{\mathcal{E}_\ell} \r\|^2,
        \end{align*}
        where we used $\|\b{v} \odot \b{v}\| \leq \|\b{v}\|^2$ for all vectors $\b{v}$. Using the above, it follows that
        \begin{align}
            \sum_{\ell \in \Sm_1} H_\ell^1 &\leq \frac{\nu_2 \sqrt{n_3}}{\sqrt{k}} \sum_{\ell \in \Sm_1} \E\l[ \l\|  f(\b{x}; \b{W}_1, \b{W}_2) - f(\b{x}'; \b{U}_1, \b{W}_2) \r\|  \l\|[\b{W}_1 \b{x}' ]_{\mathcal{E}_\ell} \r\|^2\r] \nonumber \\
            &\leq \frac{\nu_2 \sqrt{n_3}}{\sqrt{k}} \E\l[ \l\|  f(\b{x}; \b{W}_1, \b{W}_2) - f(\b{x}'; \b{U}_1, \b{W}_2) \r\|  \l\|\b{W}_1 \b{x}'  \r\|^2\r] \nonumber \\
            &\leq \frac{\theta^2 \sigma_1^2 \nu_2 \sqrt{n_3}}{\sqrt{k}} \E\l[ \l\|  f(\b{x}; \b{W}_1, \b{W}_2) - f(\b{x}'; \b{U}_1, \b{W}_2) \r\| \r] \nonumber \\
            &\leq \frac{\theta^2 \sigma_1^2 \nu_2 \sqrt{n_3}}{\sqrt{k}} \sqrt{\E\l[\Psi_f(\b{U}_1, \b{W}_2)\r]}, \label{chapter3:lemma:8.9}
        \end{align}
        where we used the Cauchy-Schwarz inequality in the last line. We next bound $\sum_{\ell \in \Sm_1} H_\ell^2$. We have
        \begin{align*}
             \l\| f(\b{x}'; \b{U}_1, \b{W}_2) - f(\b{x}'; \b{G}(t_\ell; \mathcal{E}_{\ell})\b{U}_1, \b{W}_2 )\r\| &= \|\b{W}_2 \varphi(\b{U}_1 \b{x}') - \b{W}_2 \varphi(\b{G}(t_\ell; \mathcal{E}_\ell)\b{U}_1 \b{x}') \|\\
             &\leq \sigma_2 \|\varphi(\b{U}_1 \b{x}') -  \varphi(\b{G}(t_\ell; \mathcal{E}_\ell)\b{U}_1 \b{x}') \|\\
             &\leq \sigma_2 \|\b{U}_1 \b{x}' -  \b{G}(t_\ell; \mathcal{E}_\ell)\b{U}_1 \b{x}' \|\\
             &= \sigma_2 |1-t_\ell| \l\| \sum_{j \in [n_1/k]} \b{B}^1_{\ell,j} \b{x}'_{\mathcal{E}_j}\r\| \\
             &\leq \tau \sigma_2  \nu_1 \sum_{j \in [n_1/k]} \| \b{x}'_{\mathcal{E}_j}\| \\
             &\leq \tau \sigma_2  \nu_1   \sqrt{\frac{n_1}{k}} \| \b{x}'\|,
        \end{align*}
        where we used the Cauchy-Schwarz inequality in the last line. Using the above, it follows
        \begin{align}
            &\sum_{\ell \in \Sm_1} H^2_\ell \nonumber \\
            &\leq \sum_{\ell \in \Sm_1} \E\l[\l\| f(\b{x}'; \b{U}_1, \b{W}_2) - f(\b{x}'; \b{G}(t_\ell; \mathcal{E}_\ell)\b{U}_1, \b{W}_2) \r\| \l\| \frac{\partial^2 f(\b{x}'; \b{G}(t_\ell; \mathcal{E}_\ell)\b{U}_1, \b{W}_2 )}{\partial t_\ell^2}\r \|\r] \nonumber \\
            &\leq \tau \sigma_2 \nu_1   \sqrt{\frac{n_1}{k}} \frac{\nu_2 \sqrt{n_3}}{\sqrt{k}} \E\l[ \| \b{x}'\|  \sum_{\ell \in \Sm_1} \l\| [\b{W}_1 \b{x}' ]_{\mathcal{E}_\ell} \r\|^2 \r] \nonumber \\
            &\leq \tau \sigma_2 \nu_1 \nu_2   \frac{\sqrt{n_1 n_3}}{k}\E\l[ \|\b{x}'\| \| \b{W}_1 \b{x}'\|^2\r] \nonumber \\
            &\leq \tau \sigma_1^2 \sigma_2 \nu_1 \nu_2   \frac{\sqrt{n_1 n_3}}{k} \E\l[ \|\b{x}'\|^3\r] \nonumber \\
            &\leq \tau \theta^3 \sigma_1^2 \sigma_2  \nu_1 \nu_2  \frac{\sqrt{n_1 n_3}}{k}. \label{chapter3:lemma:8.10}
        \end{align}
        Combining (\ref{chapter3:lemma:8.8}), (\ref{chapter3:lemma:8.9}), and (\ref{chapter3:lemma:8.10}), we obtain 
        \begin{align*}
            &\sum_{\ell \in \Sm_1} \E \l[\l| \frac{\partial^2 \Psi_f ( \b{G}(t_\ell; \mathcal{E}_\ell)\b{U}_1, \b{W}_2) }{\partial t_\ell^2} \r| \r] \\
            &\leq \frac{2\theta^2 (\sigma_1 \nu_2 )^2 n_3}{k} +  \frac{2\theta^2 \sigma_1^2 \nu_2  \sqrt{n_3}}{\sqrt{k}}  \sqrt{\E\l[\Psi_f(\b{U}_1, \b{W}_2)\r]} \\
            &+ 2\tau \theta^3  \sigma_1^2\sigma_2 \nu_1 \nu_2   \frac{\sqrt{n_1 n_3}}{k} \\
            &= \frac{2\theta^2 \sigma_1^2 \nu_2 \sqrt{n_3} }{\sqrt{k}} \l( \frac{\nu_2 \sqrt{n_3} + \tau \theta \sigma_2 \nu_1  \sqrt{n_1}}{\sqrt{k}} + \sqrt{\E\l[\Psi_f(\b{U}_1, \b{W}_2) \r]} \r).
        \end{align*}
        Using the above bound, the result of (\ref{chapter3:lemma:8.2}) follows then by applying Lemma \ref{chapter3:lemma:2} to the sequence $a_{\ell} \triangleq \E \l[\l| \frac{\partial^2 \Psi_f ( \b{G}(t_\ell; \mathcal{E}_\ell)\b{U}_1, \b{W}_2) }{\partial t_\ell^2} \r| \r]$. 
    \end{enumerate}
\end{proof}
%%%%%%%%%%%%%%%%%%%%%%%%%%%%%%%%%%%%%%%%%%%%%%%%%%%%%%%%%%%%%%%%%%%%%%%%%%%%

\begin{lemma}\label{chapter3:lemma:9}
    Let $\mathcal{R}$ be a distribution over $\mathbb{B}^{n_1}_{2}(1) $ and $\b{x} \sim \mathcal{R}$. Given a network $\Phi$ with layers $\b{W}_\ell \in \R^{n_{\ell+1} \times n_\ell }, \ell \in [m]$ and activations $\varphi_\ell, \ell \in [m]$, let $\mathcal{W}, \mathcal{B} $ be disjoint subsets of $[m]$ satisfying $(\mathcal{B}+1) \cap (\mathcal{W} \cup \mathcal{B})= \emptyset$, and $\hat{\b{W}}_\ell \in \R^{n_{\ell+1} \times n_\ell}, \ell \in [m]$ be a collection of matrices. Let $\hat{\Phi}$ be the network  given  by $\hat{\Phi}(\b{x}) = [\varphi_m(\hat{\b{W}}_m[\varphi_{m-1}(\dots \hat{\b{W}}_1\b{x})]_{\kappa_{m-1}})]_{\kappa_m} $ where $\kappa_\ell \triangleq \sigma^\ell$ for $\ell\geq 0$,  and $\sigma \geq \max_{\ell \in [m]} \|\b{W}_\ell \|$. Moreover, let $\b{z}^\ell, \hat{\b{z}}^\ell$ be given recursively by $\b{z}^0 = \hat{\b{z}}^0 = \b{x}$, and for $\ell \in [m]$
    \begin{alignat}{2}
        &\b{z}^{\ell} &&= \varphi_{\ell}(\b{W}_\ell \b{z}^{\ell-1}),\nonumber\\
        &\hat{\b{z}}^{\ell} &&= \varphi_{\ell}(\hat{\b{W}}_\ell [\hat{\b{z}}^{\ell-1}]_{\kappa_{\ell-1}}).\nonumber
    \end{alignat}
   Suppose that the activations $\varphi_\ell$ satisfy (\ref{chapter3:assumption:1.1}), (\ref{chapter3:assumption:1.2}), and (\ref{chapter3:assumption:1.3}). If the following holds
    \begin{alignat}{2}
        &\forall \ell \not \in \mathcal{W}\cup \mathcal{B}, \quad \hat{\b{W}}_\ell = \b{W}_\ell, && \label{chapter3:lemma:9.1} \\
        &\forall \ell \in \mathcal{W}, \quad \E\l[  \|\b{z}^\ell - [\hat{\b{z}}^\ell]_{\kappa_\ell}\|^{2} \r] &&\leq (1+\varepsilon_1)\sigma^2 \E\l[\|\b{z}^{\ell-1} - [\hat{\b{z}}^{\ell-1}]_{\kappa_{\ell-1}}\|^2\r] +  \varepsilon_2 \sigma^{2\ell}, \label{chapter3:lemma:9.2}\\
        &\forall \ell \in \mathcal{B}, \quad \E\l[  \|\b{z}^{\ell+1} - [\hat{\b{z}}^{\ell+1}]_{\kappa_{\ell+1}}\|^{2} \r] &&\leq (1+\varepsilon_3)\sigma^{4} \E\l[\|\b{z}^{\ell-1} - [\hat{\b{z}}^{\ell-1}]_{\kappa_{\ell-1}}\|^2\r] +  \varepsilon_4 \sigma^{2(\ell+1)}, \label{chapter3:lemma:9.3}
    \end{alignat}
    where $\varepsilon_j, j\in [4]$ are positive constants. Then, for all $\ell \in [0, m] \setminus \mathcal{B}$, we have
    \begin{align}
       \mathcal{A}_\ell: \quad \E\l[ \| \b{z}^\ell - [\hat{\b{z}}^\ell]_{\kappa_\ell} \|^2 \r] &\leq \sigma^{2\ell} (1+\xi)^{\ell} \xi, \label{chapter3:lemma:9.4}
    \end{align}
    for all $\xi \geq  2\varepsilon_1 \vee 2\varepsilon_3 \vee \sqrt{2\varepsilon_2} \vee \sqrt{2 \varepsilon_4}$.
\end{lemma}
\begin{proof}
    We prove $\mathcal{A}_\ell$ recursively over $\ell \in [0, m]$. Suppose $\ell=0$. Since $\kappa_0 \geq \|\b{x}\|$, it follows that $[\hat{\b{z}}^0]_{\kappa_0} = \b{x} = \b{z}^0$. Thus $\E\l[\| \b{z}^\ell - [\hat{\b{z}}^\ell]_{\kappa_\ell}\|^2\r]=0 \leq \xi$ and $\mathcal{A}_0$ holds trivially. Suppose we have shown $\mathcal{A}_i, \forall i<\ell$ for some $\ell\in [m]$, we next show $\A_\ell$. Noting that
    \begin{align*}
        [m]\setminus \mathcal{B} = ([m]\setminus (\mathcal{W} \cup \mathcal{B} \cup (\mathcal{B}+1))) \cup \mathcal{W}  \cup  (\mathcal{B} + 1),
    \end{align*}
    we consider $3$ cases.
    \begin{enumerate}
        \item Case 1. $\ell \not \in \mathcal{W} \cup \mathcal{B} \cup (\mathcal{B}+1)$.
       In this case, we have $\hat{\b{W}}_{\ell}=\b{W}_{\ell}$. Since $\kappa_\ell = \sigma^\ell \geq \|\b{z}^\ell\|$, it follows from  Lemma \ref{chapter3:lemma:1} that
        \begin{align}
            \E\l[\|\b{z}^{\ell}-[\hat{\b{z}}^{\ell}]_{\kappa_\ell}\|^2\r] &\leq \E\l[\|\b{z}^\ell - \hat{\b{z}}^\ell \|^2\r]\nonumber\\
            &= \E\l[\| \varphi_{\ell}(\b{W}_{\ell}\b{z}^{\ell-1} ) - \varphi_{\ell}(\b{W}_{\ell } [\hat{\b{z}}^{\ell-1}]_{\kappa_{\ell-1}}) \|^2\r]\nonumber\\
            &\leq \sigma^2 \E\l[\|\b{z}^{\ell-1} - [\hat{\b{z}}^{\ell-1}]_{\kappa_{\ell-1}} \|^2\r] \label{chapter3:lemma:9.5}\\
            &\leq \sigma^{2\ell}(1+\xi)^{\ell-1} \xi \label{chapter3:lemma:9.6}\\
            &\leq \sigma^{2\ell}(1+\xi)^{\ell}\xi, \nonumber
        \end{align}
        where we used  $\|\varphi_{\ell}\|_{\rm Lip}\leq 1$ and $\|\b{W}_\ell\|\leq \sigma$ in (\ref{chapter3:lemma:9.5}), and the inductive hypothesis $\mathcal{A}_{\ell-1}$ in (\ref{chapter3:lemma:9.6}) (note that we have $\ell-1 \not \in \mathcal{B}$). This yields $\mathcal{A}_\ell$.

        \item Case 2. $\ell\in \mathcal{W}$. Note that $(\b{z}^{\ell-1}, [\hat{\b{z}}^{\ell-1}]_{\kappa_{\ell-1}}) \in \mathbb{B}^{n_{\ell}}_2(\kappa_{\ell-1}) \times \mathbb{B}^{n_{\ell}}_{2}(\kappa_{\ell-1})$. We have using (\ref{chapter3:lemma:9.2}) 
        \begin{align}
            \E\l[\| \b{z}^\ell  - [\hat{\b{z}}^{\ell}]_{\kappa_\ell} \|^2\r] &\leq (1+\varepsilon_1) \sigma^2 \E\l[\| \b{z}^{\ell-1} - [\hat{\b{z}}^{\ell-1}]_{\kappa_{\ell-1}} \|^2\r]  + \sigma^{2\ell} \varepsilon_2 \nonumber\\
            &\leq \sigma^{2\ell}(1+\varepsilon_1)(1+\xi)^{\ell-1} \xi  + \sigma^{2\ell} \varepsilon_2 \label{chapter3:lemma:9.7}\\
            &\leq \sigma^{2\ell} (1+\xi)^\ell \xi \l(  \frac{1+\varepsilon_1}{1+\xi} +  \frac{\varepsilon_2}{(1+\xi)^\ell \xi}\r)\nonumber\\
            &\leq \sigma^{2\ell} (1+\xi)^\ell \xi \l(\frac{1 + \frac{\xi}{2}}{1+\xi} + \frac{\frac{\xi}{2}}{1+\xi}\r) \label{chapter3:lemma:9.8}\\
            &= \sigma^{2\ell}(1+\xi)^\ell \xi, \nonumber
        \end{align}
        where we used the inductive hypothesis $\mathcal{A}_{\ell-1}$ in (\ref{chapter3:lemma:9.7}), and $\xi \geq \sqrt{2\varepsilon_2} \vee 2\varepsilon_1$ together with $(1+\xi)^{\ell}\geq 1+\xi$ in line (\ref{chapter3:lemma:9.8}). This yields $\mathcal{A}_\ell$.

        \item $\ell\in \mathcal{B}+1$. Note that $(\b{z}^{\ell-1}, [\hat{\b{z}}^{\ell-1}]_{\kappa_{\ell-1}}) \in \mathbb{B}^{n_{\ell}}_2(\kappa_{\ell-1}) \times \mathbb{B}^{n_{\ell}}_{2}(\kappa_{\ell-1})$. We have using (\ref{chapter3:lemma:9.3})
        \begin{align}
            \E\l[\| \b{z}^\ell  - [\hat{\b{z}}^{\ell}]_{\kappa_\ell} \|^2\r] &\leq (1+\varepsilon_3) \sigma^4 \E\l[\| \b{z}^{\ell-2} - [\hat{\b{z}}^{\ell-2}]_{\kappa_{\ell-2}} \|^2\r]  + \sigma^{2\ell} \varepsilon_4 \nonumber \\
            &\leq \sigma^{2\ell}(1+\varepsilon_3)(1+\xi)^{\ell-2} \xi  + \sigma^{2\ell} \varepsilon_4 \label{chapter3:lemma:9.9}\\
            &\leq \sigma^{2\ell} (1+\xi)^\ell \xi \l(  \frac{1+\varepsilon_3}{(1+\xi)^2} +  \frac{\varepsilon_4}{(1+\xi)^\ell \xi}\r)\nonumber \\
            &\leq \sigma^{2\ell} (1+\xi)^\ell \xi \l(\frac{1 + \frac{\xi}{2}}{1+\xi} + \frac{\frac{\xi}{2}}{1+\xi}\r) \label{chapter3:lemma:9.10}\\
            &= \sigma^{2\ell}(1+\xi)^\ell \xi ,\nonumber
        \end{align}
        where we used the inductive hypothesis $\mathcal{A}_{\ell-2}$ in (\ref{chapter3:lemma:9.9}), and  $\xi \geq \sqrt{2\varepsilon_4} \vee 2\varepsilon_3$ together with $(1+\xi)^\ell \geq 1+\xi$ in (\ref{chapter3:lemma:9.10}). The latter yields $\mathcal{A}_\ell$.
    \end{enumerate}
    This concludes the proof.
\end{proof}

%\newpage
\section{Proofs for Unstructured Compression of Multilayer Perceptrons}\label{chapter3:appendix:unstructured.shallow}
%%%%%%%%%%%%%%%%%%%%%%%%%%%%%%%%%%%%%%%%%%%%%%%%%%%%%%%%%%%%%%%%%%%%%%%%%%%%
%%%%%%%%%%%%%%%%%%%%%%%%%%%%%%%%%%%%%%%%%%%%%%%%%%%%%%%%%%%%%%%%%%%%%%%%%%%%
%%%%%%%%%%%%%%%%%%%%%%%%%%%%%%%%%%%%%%%%%%%%%%%%%%%%%%%%%%%%%%%%%%%%%%%%%%%%
\subsection{Single-Layer Perceptron}\label{chapter3:appendix:unstructured.single.layer}

In this section we consider a one-layer perceptron. Namely, we let $m=1$ in (\ref{chapter3:fcnn}). Let $\rho, \theta \in \R_{>0}$ and $\mathcal{Q}_{\rho, \theta}$ be a joint distribution over $\mathbb{B}_{2}^{n_1}(\rho)\times \mathbb{B}^{n_1}_{2}(\theta)$.  Introduce
\begin{align}
    \Psi: \quad \R^{n_2\times n_1}\to \R, \quad \b{A} \mapsto \E\l[\|\b{W}_1 \b{x} - \b{A} \b{x}'\|^2\r], \nonumber
\end{align}
where the expectation is taken over $(\b{x}, \b{x}') \sim \mathcal{Q}_{\rho, \theta}$. Given $\kappa \in (0, \infty)$, we also introduce
\begin{align}
    \Pi^{\kappa}: \quad \R^{n_2\times n_1}\to \R, \quad \b{A} \mapsto \E\l[\|\varphi_1(\b{W}_1 \b{x}) - [\varphi_1(\b{A} \b{x}')]_\kappa\|^2\r].\nonumber
\end{align}
We consider two compression techniques: quantization and pruning. In the latter case, we aim at zeroing as many weights in $\b{W}_1$ while preserving the loss $\Pi^{\kappa}$. In the former case, we replace the weights in $\b{W}_1$ by a discrete approximation from a fixed set of possible weight values.  We first state this section's result for quantization. 
%%%%%%%%%%%%%%%%%%%%%%%%%%%%%%%%%%%%%%%%%%%%%%%%%%%%%%%%%%%%%%%%%%%%%%%%%%%%
%%%%%%%%%%%%%%%%%%%%%%%%%%%%%%%%%%%%%%%%%%%%%%%%%%%%%%%%%%%%%%%%%%%%%%%%%%%%
%%%%%%%%%%%%%%%%%%%%%%%%%%%%%%%%%%%%%%%%%%%%%%%%%%%%%%%%%%%%%%%%%%%%%%%%%%%%
\begin{proposition}\label{chapter3:appendix:unstructured.single.proposition:1}
    Suppose the activation $\varphi_1$ satisfies (\ref{chapter3:assumption:1.1}) and (\ref{chapter3:assumption:1.2}) in Assumption \ref{chapter3:assumption:1}.
    Let $\alpha \in \l(\frac{1}{n_1 n_2}, 1-\frac{1}{n_1 n_2}\r]$, $\sigma_1 = \|\b{W}_1\|, \nu_1 = \|\b{W}_1\|_{\infty}$, and $\kappa\in [ \sigma_1 \rho, \infty)$. Given a quantization parameter $k \in \mathbb{Z}_{\geq 1}$, there exists a mask matrix $\b{M} \in \l\{0, 1\r\}^{n_2\times n_1}$ satisfying
    \begin{align}
        \frac{\|\b{M}\|_0}{n_1 n_2} = 1-\frac{\lfloor \alpha n_1 n_2 \rfloor}{n_1 n_2},  \label{chapter3:appendix:unstructured.single.proposition:1.1}
    \end{align}
     and a matrix $\b{Q} \in \l\{ \pm \frac{\ell \nu_1 }{k}  \mid \ell \in [k]\r\}^{n_2 \times n_1}$, such that 
    \begin{align}
        \Pi^{\kappa}\l( (\b{1} \b{1}^\top - \b{M})\odot \b{Q} + \b{M} \odot \b{W}_1\r) &\leq  \sigma_1^2\E\l[\|\b{x} - \b{x}'\|^2\r]+ \frac{2 \alpha \theta^2 \nu_1^2  n_2}{k^2(1-\alpha)}.  \label{chapter3:appendix:unstructured.single.proposition:1.2}
    \end{align}
\end{proposition}
Next, we state this section's result for pruning.
\begin{proposition}\label{chapter3:appendix:unstructured.single.proposition:2}
    Suppose the activation $\varphi_1$ satisfies (\ref{chapter3:assumption:1.1}) and (\ref{chapter3:assumption:1.2}) in Assumption \ref{chapter3:assumption:1}. Let $p,\gamma, \varepsilon \in (0, 1)$, $\alpha \in \l(\frac{1}{n_1 n_2}, 1-\frac{1}{n_1 n_2}\r], \sigma_1=\|\b{W}_1\|, \nu_1 = \|\b{W}_1\|_{\infty}$ and $\kappa\in [\sigma_1 \rho, \infty)$. There exists a constant $n_0 = n_0(\gamma, p, \alpha, \varepsilon)$ such that if $n_1 \vee n_2 \geq n_0$, then  there exists a mask matrix $\b{M} \in \l\{0, 1, \frac{1}{p}\r\}^{n_2\times n_1}$ satisfying
    \begin{align}
        \frac{\|\b{M}\|_0}{n_1 n_2} \leq 1-\frac{\lfloor \alpha n_1 n_2 \rfloor}{n_1 n_2}+ \frac{\lfloor \alpha n_1 n_2 \rfloor}{n_1 n_2} (1+\gamma)p,  \label{chapter3:appendix:unstructured.single.proposition:2.1}
    \end{align}
     such that 
    \begin{align}
        \Pi^{\kappa}(\b{M} \odot \b{W}_1) &\leq (1+\varepsilon)\sigma_1^2\E\l[\|\b{x} - \b{x}'\|^2\r] +  (1+\varepsilon)\frac{2\alpha \theta^2 (1-p)\nu_1^2  n_2}{p(1-\alpha)}. \label{chapter3:appendix:unstructured.single.proposition:2.2}
    \end{align}
\end{proposition}

%%%%%%%%%%%%%%%%%%%%%%%%%%%%%%%%%%%%%%%%%%%%%%%%%%%%%%%%%%%%%%%%%%%%%%%%%%%%
%%%%%%%%%%%%%%%%%%%%%%%%%%%%%%%%%%%%%%%%%%%%%%%%%%%%%%%%%%%%%%%%%%%%%%%%%%%%
%%%%%%%%%%%%%%%%%%%%%%%%%%%%%%%%%%%%%%%%%%%%%%%%%%%%%%%%%%%%%%%%%%%%%%%%%%%%

\subsubsection{Proof of Propositions \ref{chapter3:appendix:unstructured.single.proposition:1} and \ref{chapter3:appendix:unstructured.single.proposition:2}}
We first show the following Proposition.
%%%%%%%%%%%%%%%%%%%%%%%%%%%%%%%%%%%%%%%%%%%%%%%%%%%%%%%%%%%%%%%%%%%%%%%%%%%%
%%%%%%%%%%%%%%%%%%%%%%%%%%%%%%%%%%%%%%%%%%%%%%%%%%%%%%%%%%%%%%%%%%%%%%%%%%%%
%%%%%%%%%%%%%%%%%%%%%%%%%%%%%%%%%%%%%%%%%%%%%%%%%%%%%%%%%%%%%%%%%%%%%%%%%%%%
\begin{proposition}\label{chapter3:appendix:unstructured.single.proposition:3}
    Suppose the activation $\varphi_1$ satisfies (\ref{chapter3:assumption:1.1}) and (\ref{chapter3:assumption:1.2}) in Assumption \ref{chapter3:assumption:1}. Let $\alpha \in \l(\frac{1}{n_1 n_2}, 1-\frac{1}{n_1 n_2}\r]$, $\sigma_1=\|\b{W}_1\|, \nu_1 = \|\b{W}_1\|_{\infty}$, and $\kappa\in [\sigma_1 \rho, \infty)$. Suppose $\b{h}\in \R^{n_2 \times n_1}$ is a random matrix with independent entries $\b{h}_{ij}$ satisfying
    \begin{align}
        \forall (i, j) \in [n_2] \times [n_1], \quad \E[\b{h}_{ij}] &= [\b{W}_1]_{ij},  \quad \text{and } \quad {\rm Var}(\b{h}_{ij}) \leq \phi^2 \nu_1^2 , \nonumber
    \end{align}
    where $\phi\in \mathbb{R}_{>0}$. Then, there exists a (nonrandom) mask matrix $\b{M} \in \l\{0, 1\r\}^{n_2\times n_1}$ satisfying
    \begin{align}
        \frac{\|\b{M}\|_0}{n_1 n_2} = 1- \frac{\lfloor \alpha n_1 n_2 \rfloor}{n_1 n_2},\label{chapter3:appendix:unstructured.single.proposition:3.1}
    \end{align}
     such that 
    \begin{align}
        \E_{\b{h}}\l[\Pi^{\kappa}\l( (\b{1} \b{1}^\top - \b{M})\odot \b{h} + \b{M} \odot \b{W}_1\r) \r] &\leq  \sigma_1^2\E\l[\|\b{x} - \b{x}'\|^2\r] + \frac{2 \alpha\theta^2 \phi^2 \nu_1^2  n_2}{1-\alpha}. \label{chapter3:appendix:unstructured.single.proposition:3.2}
    \end{align}
\end{proposition}
%%%%%%%%%%%%%%%%%%%%%%%%%%%%%%%%%%%%%%%%%%%%%%%%%%%%%%%%%%%%%%%%%%%%%%%%%%%%
%%%%%%%%%%%%%%%%%%%%%%%%%%%%%%%%%%%%%%%%%%%%%%%%%%%%%%%%%%%%%%%%%%%%%%%%%%%%
%%%%%%%%%%%%%%%%%%%%%%%%%%%%%%%%%%%%%%%%%%%%%%%%%%%%%%%%%%%%%%%%%%%%%%%%%%%%
\begin{proof}
We use an interpolation argument whereby we iteratively pick \textit{suitable} indices $(i, j) \in [n_2]\times [n_1]$ and switch the $(i,j)$-th weight  from $[\b{W}_1]_{ij}$ to $\b{h}_{ij}$. We control the resulting error from each switch, and  make $\lfloor \alpha n_1 n_2 \rfloor$ switches. Let $\b{U}^d$ be the weight matrix at the end of the $d$-th  interpolation step where $d \in [0, \lfloor \alpha n_1 n_2 \rfloor - 1]$, and let $\mathcal{S}^d_1 \cup \mathcal{S}^d_2 = [n_2]\times [n_1]$ track the interpolation process where $\mathcal{S}^d_1$ contains indices of  unswitched weights, and $\mathcal{S}^d_2$ contains indices of switched weights. Namely,
\begin{alignat}{2}
    \b{U}^d_{ij} &= [\b{W}_1]_{ij}, \quad &&\text{if } (i,j) \in \mathcal{S}^d_1,\nonumber\\
    \b{U}^d_{ij} &= \b{h}_{ij} \quad &&\text{if } (i,j)  \in \mathcal{S}^d_2.\nonumber
\end{alignat}
In particular $\b{U}^0 = \b{W}_1$ (i.e. $\mathcal{S}^0_1=[n_2]\times [n_1], \mathcal{S}^0_2=\emptyset$). Finally, we let $\T^d = \{\b{h}_{ij}, (i,j) \in \Sm_2^d\}$, with $\T^0 = \emptyset$. Suppose we are at step $d< \lfloor \alpha n_1 n_2 \rfloor$ so that  $|\mathcal{S}^d_2| < \lfloor \alpha n_1 n_2 \rfloor$ (otherwise the interpolation is over). Let $(i, j) \in \mathcal{S}^d_1$ (note that $\mathcal{S}^d_1 \neq \emptyset$ as $|\mathcal{S}^d_1|\geq  n_1 n_2 - \lfloor \alpha  n_1 n_2 \rfloor + 1 \geq 1 $). Since $t\mapsto \Psi(\b{U}^d(t; i,j))$ is a quadratic function,  we have by Taylor's
\begin{align}
    \Psi(\b{U}^{d}(\b{h}_{ij};i,j)) - \Psi(\b{U}^{d}) & =  \Psi(\b{U}^{d}(\b{h}_{ij};i,j)) - \Psi(\b{U}^{d}([\b{W}_1]_{ij};i, j)) \nonumber \\
    &= (\b{h}_{ij}-[\b{W}_1]_{ij})\frac{\partial \Psi(\b{U}^{d}([\b{W}_1]_{ij};i,j))}{\partial t}  \nonumber\\
    &+ \frac{(\b{h}_{ij}  - [\b{W}_1]_{ij})^2}{2}\frac{\partial^2 \Psi(\b{U}^{d}([\b{W}_1]_{ij};i,j))}{\partial t^2}. \nonumber
\end{align}
Taking the expectation of the above over  $\{\b{h}_{ij}\}\cup \T^d$  we obtain
\begin{align}
    \l|\E_{\{\b{h}_{ij}\}\cup \T^d}\l[ \Psi(\b{U}^{d}(\b{h}_{ij};i,j)) \r] - \E_{\T^d}\l[\Psi(\b{U}^{d}) \r]  \r|     &\leq \frac{\phi^2 \nu_1^2}{2} \E_{\{\b{h}_{ij}\} \cup \T^d}\l[ \l| \frac{\partial^2 \Psi(\b{U}^{d}([\b{W}_1]_{ij};i,j))}{\partial t^2} \r| \r].   \label{chapter3:appendix:unstructured.single.proposition:3.3}
\end{align}
Let $\eta \in \R_{\geq 1}$ and introduce
\begin{alignat}{2}
    \mathcal{C} &\triangleq \Bigg\{ (i, j) \in \mathcal{S}^d_1 \Biggm| \quad \E_{ \{\b{h}_{ij}\} \cup\T^d}&&\l[ \l|\frac{\partial^2 \Psi(\b{U}^{d}([\b{W}_1]_{ij};i,j))}{\partial t^2} \r| \r] \leq \frac{2\eta\theta^2 n_2}{|\mathcal{S}^d_1|}  \Bigg\}.\nonumber
\end{alignat}
Using Lemma \ref{chapter3:lemma:7}, it follows that $|\mathcal{C}| \geq \l(1- \frac{1}{\eta}\r)|\mathcal{S}^d_1|$.  Since $\alpha\leq 1-\frac{1}{n_1 n_2}$, we have $|\Sm^d_1|\geq n_1 n_2 - \lfloor \alpha n_1 n_2\rfloor + 1 \geq (1-\alpha)n_1 n_2 + 1 \geq 2 $. Setting $\eta=2$, it follows that $|\mathcal{C}| \geq (1-1/\eta)2=1$, and thus $\mathcal{C}  \neq \emptyset$. Henceforth, if we pick the pair $(i,j)\in \Sm_1^d$ with the smallest \textit{score} given by $\l|\E_{\{\b{h}_{ij}\}\cup \T^d}\l[ \Psi(\b{U}^{d}(\b{h}_{ij};i,j)) \r] - \E_{\T^d}\l[\Psi(\b{U}^{d}) \r]  \r|$, it follows from (\ref{chapter3:appendix:unstructured.single.proposition:3.3}) that
\begin{align}
    \l| \E_{\{\b{h}_{ij}\}\cup \T^d}\l[ \Psi(\b{U}^{d}(\b{h}_{ij};i,j)) \r]  - \E_{\T^d}\l[\Psi(\b{U}^{d})\r] \r| &\leq  \frac{2\theta^2 \phi^2 \nu_1^2   n_2}{|\mathcal{S}^d_1|}.  \label{chapter3:appendix:unstructured.single.proposition:3.4}
\end{align}
 We then update $\Sm^{d+1}_1 = \Sm^{d}_1 \setminus \{(i,j)\}$ and $\Sm^{d+1}_2 = \Sm^{d}_2 \cup \{(i,j)\}$. Suppose we repeat this switching operation as long as $d < \lfloor\alpha n_1 n_2\rfloor$, and let $\mathcal{K}= \Sm^{\lfloor \alpha n_1 n_2 \rfloor}_2$. That is, $\mathcal{K}$ is the set of all swapped  pairs $(i, j)$ at the end of the interpolation process, so that $|\mathcal{K}|= \lfloor \alpha n_1 n_2 \rfloor$. Let $\beta_d =\E_{\T^d}[\Psi(\b{U}^d)]$ for $d\in [0, |\mathcal{K}|]$ and note that if we switch  entry $(i, j)$  at the $d$-th step, then $\beta_d = \E_{\T^d}\l[\Psi(\b{U}^{d}([\b{W}_1]_{ij};i,j))\r]=\E_{\T^d}[\Psi(\b{U}^d)]$ and 
$\beta_{d+1} = \E_{\T^{d+1}}\l[\Psi(\b{U}^{d}(\b{h}_{ij};i,j))\r]$. Therefore, we can rewrite (\ref{chapter3:appendix:unstructured.single.proposition:3.4}) as 
\begin{align}
    |\beta_{d+1} - \beta_{d}| &\leq  \frac{2 \theta^2 \phi^2 \nu_1^2  n_2}{|\Sm^d_1|} \leq \frac{2 \theta^2 \phi^2 \nu_1^2 }{1-\alpha}  \frac{1}{n_1},\nonumber
\end{align}
where we used $|\Sm^d_1| \geq (1-\alpha)n_1 n_2$ in the last inequality. Telescoping the above inequality over $d\in [0, \lfloor \alpha n_1 n_2 \rfloor -1]$ yields
\begin{align}
    | \beta_{\lfloor \alpha n_1 n_2 \rfloor} - \beta_0| &\leq  \frac{2 \theta^2 \phi^2 \nu_1^2 }{1-\alpha}  \frac{\lfloor \alpha n_1 n_2 \rfloor}{n_1} \leq \frac{2 \alpha \theta^2 \phi^2 \nu_1^2  n_2}{1-\alpha} .\nonumber
\end{align}
Letting $\b{U} = \b{U}^{\lfloor \alpha n_1 n_2 \rfloor}$ and $\T = \T^{\lfloor \alpha n_1 n_2 \rfloor}$, it follows that
\begin{align}
    \E_{\T}[\Psi(\b{U})] &\leq \Psi(\b{W}_1) +  \frac{2 \alpha \theta^2 \phi^2 \nu_1^2  n_2}{1-\alpha}\nonumber.
\end{align}
We have $\Psi(\b{W}_1) \leq \sigma_1^2 \E[\|\b{x}-\b{x}'\|^2]$. Therefore, 
\begin{align}
     \E_{\T}[\Psi(\b{U})] &\leq \sigma_1^2\E\l[\|\b{x} - \b{x}'\|^2\r] + \frac{2 \alpha \theta^2 \phi^2 \nu_1^2  n_2}{1-\alpha}. \label{chapter3:appendix:unstructured.single.proposition:3.5}
\end{align}
We next show
\begin{align}
    \E_{\T}\l[\Pi^{\kappa}(\b{U})\r] &\leq \sigma_1^2\E\l[\|\b{x} - \b{x}'\|^2\r] + \frac{2 \alpha \theta^2 \phi^2 \nu_1^2  n_2}{1-\alpha}. \label{chapter3:appendix:unstructured.single.proposition:3.6}
\end{align}
By (\ref{chapter3:appendix:unstructured.single.proposition:3.5}), it suffices to prove that
\begin{align}
    \E_{\T}\l[\Pi^{\kappa}(\b{U})\r] &\leq \E_{\T}\l[\Psi(\b{U})\r],  \label{chapter3:proposition:onerlayer.7}
\end{align}
which we show next. Fix a pair of vectors $(\b{x}, \b{x}') \in \mathbb{B}^{n_1}_{2}(\rho) \times \mathbb{B}^{n_1}_{2}(\theta)$, and denote by $\b{z}, \b{z}'$ the output vectors $\varphi_1(\b{W}_1 \b{x}),  \varphi_1(\b{U} \b{x}')$. As $\|\b{x}\|\leq \rho$  it follows from (\ref{chapter3:assumption:1.1}) and (\ref{chapter3:assumption:1.2}) that $  \|\b{z}\| = \|\varphi_1(\b{W}_1 \b{x})\| \leq \sigma_1 \rho $. Since $\|\b{z}\|\leq \sigma_1 \rho $ and $\sigma_1\rho \leq \kappa$ from the assumptions of Proposition \ref{chapter3:appendix:unstructured.single.proposition:3}, we have using Lemma \ref{chapter3:lemma:1} that 
\begin{align}
    \|\b{z} - [\b{z}']_{\kappa}\|  &\leq \|\b{z} - \b{z}'\|\nonumber \\
    &= \|\varphi_1(\b{W}_1 \b{x}) - \varphi_1(\b{U}\b{x}')\|\nonumber \\
    &\leq \|\b{W}_1 \b{x} - \b{U} \b{x}'\|, \nonumber
\end{align}
where we used (\ref{chapter3:assumption:1.2}) in the last line. Therefore,
\begin{align}
    \E_{\T}\l[\Pi^{\kappa}(\b{U})\r] &= \E_{\T}\E_{(\b{x}, \b{x}')}\l[\|\b{z} - [\b{z}']_{\kappa}\|^2\r]\nonumber\\
    &\leq  \E_{\T}\E_{(\b{x}, \b{x}')}\l[\|\b{W}_1 \b{x} - \b{U} \b{x}'\|^2\r]\nonumber\\
    &= \E_{\T}\l[\Psi(\b{U})\r]. \nonumber
\end{align}
This concludes the proof of (\ref{chapter3:appendix:unstructured.single.proposition:3.6}), and also shows (\ref{chapter3:appendix:unstructured.single.proposition:3.2}). Note  that by construction
\begin{align}
    \b{U} =(\b{1}\b{1}^\top-\b{M}) \odot \b{h} + \b{M} \odot \b{W}_1,\nonumber
\end{align}
where $\b{M}_{ij} = 1_{(i,j)\not \in \mathcal{K}}$ is a mask matrix satisfying $\|\b{M}\|_{0} = n_1 n_2 - |\mathcal{K}|= n_1 n_2 - \lfloor \alpha n_1 n_2 \rfloor$. This shows (\ref{chapter3:appendix:unstructured.single.proposition:3.1}), and  ends the proof of Proposition  \ref{chapter3:appendix:unstructured.single.proposition:3}.    
\end{proof}
%%%%%%%%%%%%%%%%%%%%%%%%%%%%%%%%%%%%%%%%%%%%%%%%%%%%%%%%%%%%%%%%%%%%%%%%%%%%
%%%%%%%%%%%%%%%%%%%%%%%%%%%%%%%%%%%%%%%%%%%%%%%%%%%%%%%%%%%%%%%%%%%%%%%%%%%%
%%%%%%%%%%%%%%%%%%%%%%%%%%%%%%%%%%%%%%%%%%%%%%%%%%%%%%%%%%%%%%%%%%%%%%%%%%%%
\subsubsection{Proof of Proposition \ref{chapter3:appendix:unstructured.single.proposition:1}}
\begin{proof}[Proof of Proposition \ref{chapter3:appendix:unstructured.single.proposition:1}]
Let $\b{h}_{ij} = q\l([\b{W}_1]_{ij}; \nu_1, k\r)$ for $(i, j) \in [n_2] \times [n_1]$. We then have by Lemma \ref{chapter3:lemma:0} that $\E[\b{h}_{ij}]=[\b{W}_1]_{ij}$, and ${\rm Var}(\b{h}_{ij}) \leq \frac{\nu_1^2}{k^2} $. Using Proposition \ref{chapter3:appendix:unstructured.single.proposition:3} with $\phi^2 = 1/k^2$, there exists a mask matrix $\b{M}\in \{0, 1\}^{n_2 \times n_1}$ satisfying 
\begin{align}
    \frac{\|\b{M}\|_0}{n_1 n_2} &= 1 - \frac{\lfloor \alpha n_1 n_2 \rfloor}{n_1 n_2},\nonumber
\end{align}
and
\begin{align}
    \E_{\b{h}}\l[\Pi^{\kappa}\l((\b{1}\b{1}^\top-\b{M}) \odot \b{h} + \b{M} \odot \b{W}_1  \r)\r] &\leq \sigma_1^2 \E\l[\|\b{x} - \b{x}'\|^2\r] + \frac{2 \alpha \theta^2  \nu_1^2  n_2}{k^2(1-\alpha)}. \nonumber
\end{align}
Since the above inequality holds in expectation over $\b{h}$, it follows that there exists a (nonrandom) realization $\b{Q} \in\l\{ \pm \frac{\ell \nu_1}{k} \big| \ell \in [k]\r\}^{n_2 \times n_1}$ of $\b{h}$ such that
\begin{align}
    \Pi^{\kappa}\l((\b{1}\b{1}^\top-\b{M}) \odot \b{Q} + \b{M} \odot \b{W}_1   \r) &\leq \E_{\b{h}}\l[\Pi^{\kappa}\l((\b{1}\b{1}^\top-\b{M}) \odot \b{h} + \b{M} \odot \b{W}_1   \r)\r] \nonumber\\
    &\leq \sigma_1^2 \E\l[\|\b{x} - \b{x}'\|^2\r] + \frac{2 \alpha \theta^2 \nu_1^2  n_2}{k^2(1-\alpha)}. \nonumber
\end{align}
This ends the proof of Proposition \ref{chapter3:appendix:unstructured.single.proposition:1}.
\end{proof}
%%%%%%%%%%%%%%%%%%%%%%%%%%%%%%%%%%%%%%%%%%%%%%%%%%%%%%%%%%%%%%%%%%%%%%%%%%%%
%%%%%%%%%%%%%%%%%%%%%%%%%%%%%%%%%%%%%%%%%%%%%%%%%%%%%%%%%%%%%%%%%%%%%%%%%%%%
%%%%%%%%%%%%%%%%%%%%%%%%%%%%%%%%%%%%%%%%%%%%%%%%%%%%%%%%%%%%%%%%%%%%%%%%%%%%
\subsubsection{Proof of Proposition \ref{chapter3:appendix:unstructured.single.proposition:2}}

%%%%%%%%%%%%%%%%%%%%%%%%%%%%%%%%%%%%%%%%%%%%%%%%%%%%%%%%%%%%%%%%%%%%%%%%%%%%
%%%%%%%%%%%%%%%%%%%%%%%%%%%%%%%%%%%%%%%%%%%%%%%%%%%%%%%%%%%%%%%%%%%%%%%%%%%%
%%%%%%%%%%%%%%%%%%%%%%%%%%%%%%%%%%%%%%%%%%%%%%%%%%%%%%%%%%%%%%%%%%%%%%%%%%%%
\begin{proof}[Proof of Proposition \ref{chapter3:appendix:unstructured.single.proposition:2}]
Let $\b{h} = \frac{\b{Ber}^{n_2, n_1}(p)}{p} \odot \b{W}_1$. We have $\E[\b{h}_{ij}]=[\b{W}_1]_{ij}$, and ${\rm Var}(\b{h}_{ij}) = \frac{1-p}{p} ([\b{W}_1]_{ij})^2\leq \frac{(1-p) \nu_1^2}{p}$. Using Proposition \ref{chapter3:appendix:unstructured.single.proposition:3} with $\phi^2 = (1-p)/p$, there exists a mask matrix $\b{M}\in \{0, 1\}^{n_2 \times n_1}$ satisfying 
\begin{align}
    \frac{\|\b{M}\|_0}{n_1 n_2} &= 1 - \frac{\lfloor \alpha n_1 n_2 \rfloor}{n_1 n_2},
\end{align}
and
\begin{align}
    \E_{\b{h}}\l[\Pi^{\kappa}\l((\b{1}\b{1}^\top-\b{M}) \odot \b{h} + \b{M} \odot \b{W}_1  \r)\r] &\leq \sigma_1^2 \E\l[\|\b{x} - \b{x}'\|^2\r] + \frac{2\alpha \theta^2 (1-p)\nu_1^2  n_2}{p(1-\alpha)}. \label{chapter3:appendix:unstructured.single.proposition:2.3}
\end{align}
Let $\b{U} = (\b{1}\b{1}^\top-\b{M}) \odot \b{h} + \b{M} \odot \b{W}_1 $ and $\b{b} = (\b{1} \b{1}^{\top}  - \b{M}) \odot \b{h}$. In particular, the term $\b{b}$ depends on at most $\lfloor \alpha n_1 n_2 \rfloor $ independent ${\rm Bernoulli}(p)$ random variables. Using Lemma {\ref{chapter3:lemma:4}}, it follows that there exists a positive constant $n_0=n_0(\gamma, p, \alpha, \varepsilon)$ such that if $\lfloor \alpha n_1 n_2\rfloor \geq n_0 $, then there exists a nonrandom realization $\hat{\b{b}}$ of $\b{b}$ such that $\|\hat{\b{b}}\|_0 \leq (1+\gamma)p\lfloor \alpha n_1 n_2\rfloor$ and
\begin{align*}
    \Pi^{\kappa}\l(\hat{\b{U}} \r) &\leq (1+\varepsilon)\E_{\b{h}}\l[\Pi^{\kappa}\l((\b{1}\b{1}^\top-\b{M}) \odot \b{h} + \b{M} \odot \b{W}_1  \r)\r]\\
    &\leq (1+\varepsilon)\sigma_1^2 \E\l[\|\b{x} - \b{x}'\|^2\r] + (1+\varepsilon)\frac{2\alpha \theta^2 (1-p)\nu_1^2  n_2}{p(1-\alpha)},
\end{align*}
where $\hat{\b{U}}$ is given by $\hat{\b{U}} = \hat{\b{b}} + \b{M} \odot \b{W}_1$. The above shows (\ref{chapter3:appendix:unstructured.single.proposition:2.2}). Moreover, we have by construction of $\hat{\b{U}}$ that
    \begin{align}
        \|\hat{\b{U}}\|_0 &\leq \|\b{M}\|_0 + \|\hat{\b{b}}\|_0\nonumber\\
        &=  n_1 n_2 -\lfloor \alpha n_1 n_2 \rfloor + \|\hat{\b{b}}\|_0. \nonumber\\
        &\leq n_1 n_2 - \lfloor \alpha n_1 n_2 \rfloor + (1+\gamma) p \lfloor \alpha n_1 n_2 \rfloor.
    \end{align}
Note that $\hat{\b{b}}$ is the entrywise product of a mask with entries in $\l\{0, \frac{1}{p}\r\}$ and the matrix $\b{W}_1$. Therefore, $\hat{\b{U}}$ is the product of a matrix with entries in $\l\{0, 1, \frac{1}{p}\r\}$ and $\b{W}_1$. Combining the latter with the previous bound readily shows (\ref{chapter3:appendix:unstructured.single.proposition:2.1}). This ends the proof of Proposition \ref{chapter3:appendix:unstructured.single.proposition:2}.
\end{proof}
%%%%%%%%%%%%%%%%%%%%%%%%%%%%%%%%%%%%%%%%%%%%%%%%%%%%%%%%%%%%%%%%%%%%%%%%%%%%
%%%%%%%%%%%%%%%%%%%%%%%%%%%%%%%%%%%%%%%%%%%%%%%%%%%%%%%%%%%%%%%%%%%%%%%%%%%%
%%%%%%%%%%%%%%%%%%%%%%%%%%%%%%%%%%%%%%%%%%%%%%%%%%%%%%%%%%%%%%%%%%%%%%%%%%%%

\subsection{Two-Layer Perceptron}\label{chapter3:appendix:unstructured.two.layers}
In this section we consider a two-layer network. Namely, we let $m=2$  in (\ref{chapter3:fcnn}).  Following the same notation as in Section \ref{chapter3:appendix:unstructured.single.layer} let $\rho, \theta \in \R_{>0}$ and $\mathcal{Q}_{\rho, \theta}$ be a joint distribution over $\mathbb{B}_{2}^{n_1}(\rho)\times \mathbb{B}^{n_1}_{2}(\theta)$, and introduce
\begin{align}
    \Psi: \quad \R^{n_2\times n_1}\to \R, \quad \b{A} \mapsto \E\l[\|\b{W}_2\varphi_1(\b{W}_1 \b{x}) - \b{W}_2\varphi_1(\b{A} \b{x}')\|^2\r],\nonumber
\end{align}
where the expectation is taken over $(\b{x}, \b{x}') \sim \mathcal{Q}_{\rho, \theta}$. Given $\kappa \in (0, \infty)$, we also introduce
\begin{align}
    \Pi^{\kappa}: \quad \R^{n_2\times n_1}\to \R, \quad \b{A} \mapsto \E\l[\|\varphi_2(\b{W}_2\varphi_1(\b{W}_1 \b{x})) - [\varphi_2(\b{W}_2\varphi_1(\b{A} \b{x}'))]_{\kappa}\|^2\r].\nonumber
\end{align}
Similarly to section \ref{chapter3:appendix:unstructured.single.layer}, we consider both quantization and pruning. We next state this section's result for quantization. 
\begin{proposition}\label{chapter3:appendix:unstructured.two.layers.proposition:1}
    Suppose the activations $\varphi_1, \varphi_2$ satisfy (\ref{chapter3:assumption:1.1}) and (\ref{chapter3:assumption:1.2}), and $\varphi_1$ also satisfies (\ref{chapter3:assumption:1.3}) in Assumption \ref{chapter3:assumption:1}. Let $\alpha \in \l(\frac{1}{n_1 n_2}, 1-\frac{1}{n_1 n_2}\r], \sigma_i = \|\b{W}_i\|,  \nu_i = \|\b{W}_i\|_{\infty}$ for $i=1,2$, and $\kappa\in [\sigma_1 \sigma_2 \rho, \infty)$. Given a quantization parameter $k \in \mathbb{Z}_{\geq 1}$, there exists a mask matrix $\b{M} \in \l\{0, 1\r\}^{n_2\times n_1}$ satisfying
    \begin{align}
        \frac{\|\b{M}\|_0}{n_1 n_2} &= 1-\frac{\lfloor \alpha n_1 n_2 \rfloor}{n_1 n_2},  \label{chapter3:appendix:unstructured.two.layers.proposition:1.1}
    \end{align}
     and a matrix $\b{Q} \in \l\{ \pm \frac{\ell \nu_1 }{k} \big| \ell \in [k]\r\}^{n_2 \times n_1}$ such that 
    \begin{align}
        \Pi^{\kappa}\l( (\b{1} \b{1}^\top - \b{M})\odot \b{Q} + \b{M} \odot \b{W}_1\r) &\leq   (1 + f\l(\Lambda \r))(\sigma_1 \sigma_2)^2 \E\l[\|\b{x} - \b{x}'\|^2\r] +f\l(\Lambda \r) \omega,  \label{chapter3:appendix:unstructured.two.layers.proposition:1.2}
    \end{align}
 where $\Lambda = \frac{2 \alpha \theta \sigma_2 \nu_1^2  \sqrt{n_2} }{k^2(1-\alpha)}$,  $\omega =  \theta^2+\frac{2\theta \nu_2^2 n_3 \sqrt{n_2}}{\sigma_2} + \frac{2\theta^2  \nu_1 \nu_2 \sqrt{n_2 n_3}}{k} $, and $f(x) = xe^x$.
\end{proposition}
Next, we state this section's result for pruning.
\begin{proposition}\label{chapter3:appendix:unstructured.two.layers.proposition:2}
    Suppose the activations $\varphi_1, \varphi_2$ satisfy (\ref{chapter3:assumption:1.1}) and (\ref{chapter3:assumption:1.2}), and $\varphi_1$ also satisfies (\ref{chapter3:assumption:1.3}) in Assumption \ref{chapter3:assumption:1}. Let $p,\gamma, \varepsilon \in (0, 1)$, $\alpha \in \l(\frac{1}{n_1 n_2}, 1-\frac{1}{n_1 n_2}\r], \sigma_i = \|\b{W}_i\|, \nu_i = \|\b{W}_i\|_{\infty}$ for $i=1,2$,  and $\kappa\in [\sigma_1 \sigma_2 \rho, \infty)$. There exists a constant $n_0 = n_0(\gamma, p, \alpha, \varepsilon)$ such that if $n_1\vee n_2 \geq n_0$, then there exists a mask matrix $\b{M} \in \l\{0, 1, \frac{1}{p}\r\}^{n_2\times n_1}$ satisfying
    \begin{align}
        \frac{\|\b{M}\|_0}{n_1 n_2} \leq 1- \frac{\lfloor \alpha n_1 n_2 \rfloor}{n_1 n_2} + \frac{\lfloor \alpha n_1 n_2 \rfloor}{n_1 n_2}(1+\gamma)p,  \label{chapter3:appendix:unstructured.two.layers.proposition:2.1}
    \end{align}
     such that 
    \begin{align}
        \Pi^{\kappa}(\b{M} \odot \b{W}_1) &\leq   (1+\varepsilon)(1 + f\l(\Lambda \r))(\sigma_1 \sigma_2)^2 \E\l[\|\b{x} - \b{x}'\|^2\r]  +(1+\varepsilon)f\l(\Lambda \r) \omega,   \label{chapter3:appendix:unstructured.two.layers.proposition:2.2}
    \end{align}
 where $\Lambda = \frac{2 \alpha \theta (1-p) \sigma_2 \nu_1^2 \sqrt{n_2}}{p(1-\alpha)}$,  $\omega =  \theta^2+\frac{2\theta \nu_2^2 n_3 \sqrt{n_2}}{\sigma_2} +\frac{2\theta^2(p\vee (1-p)) \nu_1 \nu_2 \sqrt{n_2 n_3}}{p} $, and $f(x) = xe^x$.
\end{proposition}

\subsubsection{Proof of Propositions \ref{chapter3:appendix:unstructured.two.layers.proposition:1} and \ref{chapter3:appendix:unstructured.two.layers.proposition:2}}
We first show the following proposition.
\begin{proposition}\label{chapter3:appendix:unstructured.two.layers.proposition:3}
    Suppose the activations $\varphi_1, \varphi_2$ satisfy (\ref{chapter3:assumption:1.1}) and (\ref{chapter3:assumption:1.2}), and $\varphi_1$ also satisfies (\ref{chapter3:assumption:1.3}) in Assumption \ref{chapter3:assumption:1}. Let $\alpha \in \l(\frac{1}{n_1 n_2}, 1-\frac{1}{n_1 n_2}\r],  \sigma_i = \|\b{W}_i\|, \nu_i = \|\b{W}_i\|_{\infty}$ for $i=1,2$,  and $\kappa\in [\sigma_1 \sigma_2 \rho, \infty)$. Suppose $\b{h} \in \R^{n_2 \times n_1}$ is a random matrix with independent entries $\b{h}_{ij}$ satisfying 
    \begin{align}
        \forall (i, j) \in [n_2] \times [n_1], \quad  \E[\b{h}_{ij}] &= [\b{W}_1]_{ij}, \quad \text{and } \quad |\b{h}_{ij} - [\b{W}_1]_{ij}| \leq \nu_1  (\tau_1 1_{  \b{h}_{ij} \neq 0} + \tau_0 1_{  \b{h}_{ij} = 0}),\label{chapter3:appendix:unstructured.two.layers.proposition:3.1}
    \end{align}
    where $ \tau_1, \tau_0 \in \mathbb{R}_{\geq 0}$. Furthermore, let $\Delta = \max_{(i,j)\in[n_2]\times [n_1]} (\tau_0^2 \p(\b{h}_{ij}=0) + \tau_1^2 \p(\b{h}_{ij}\neq 0))$. Then, there exists a (nonrandom) mask matrix $\b{M} \in \l\{0, 1\r\}^{n_2\times n_1}$ satisfying
    \begin{align}
        \frac{\|\b{M}\|_0}{n_1 n_2} = 1- \frac{\lfloor \alpha n_1 n_2 \rfloor}{n_1 n_2}, \label{chapter3:appendix:unstructured.two.layers.proposition:3.2}
    \end{align}
     such that 
    \begin{align}
        \E_{\b{h}}\l[\Pi^{\kappa}\l( (\b{1} \b{1}^\top - \b{M})\odot \b{h} + \b{M} \odot \b{W}_1\r) \r] &\leq   (1 + f\l(\Lambda\r)) (\sigma_1 \sigma_2)^2 \E\l[\|\b{x} - \b{x}'\|^2\r] + f\l(\Lambda \r) \omega, \label{chapter3:appendix:unstructured.two.layers.proposition:3.3}
    \end{align}
    where $\Lambda = \frac{2 \alpha \theta \Delta \sigma_2 \nu_1^2  \sqrt{n_2}}{1-\alpha}$,  $\omega =  \theta^2+\frac{2\theta \nu_2^2 n_3 \sqrt{n_2}}{\sigma_2} + 2\theta^2 \tau \nu_1 \nu_2 \sqrt{n_2 n_3} $, $\tau = \tau_0\vee \tau_1$ and $f(x) = xe^x$.
\end{proposition}
\begin{proof}
 Similarly to the proof of Proposition \ref{chapter3:appendix:unstructured.single.proposition:3}, we use an interpolation argument whereby we iteratively pick \textit{suitable} indices $(i, j) \in [n_2]\times [n_1]$ and switch the  $(i,j)$-th weight  from $[\b{W}_1]_{ij}$ to $\b{h}_{ij}$. Reusing the same notation, we have by Taylor's
\begin{align}
    \Psi(\b{U}^{d}(\b{h}_{ij};i,j)) - \Psi(\b{U}^{d}) &= \Psi(\b{U}^{d}(\b{h}_{ij};i,j)) - \Psi(\b{U}^{d}([\b{W}_1]_{ij};i, j)) \nonumber \\
    &= (\b{h}_{ij}-[\b{W}_1]_{ij})\frac{\partial \Psi(\b{U}^{d}([\b{W}_1]_{ij};i,j))}{\partial t}  \nonumber\\
    &+ \frac{(\b{h}_{ij}  - [\b{W}_1]_{ij})^2}{2}\frac{\partial^2 \Psi(\b{U}^{d}(\hat{t}_{ij};i,j))}{\partial t^2}, \nonumber
\end{align}
 where $\hat{t}_{ij}$ are random variables in $([\b{W}_1]_{ij} \wedge \b{h}_{ij}, [\b{W}_1]_{ij}\vee\b{h}_{ij})$. Let  $u=\p(\b{h}_{ij}=0)$ and suppose that $u\in (0, 1)$ (the cases $u\in \{0, 1\}$ are identical in treatment, and thus are skipped). Taking the expectation of the Taylor expansion over $\{\b{h}_{ij}\}\cup \mathcal{T}^d$ and using (\ref{chapter3:appendix:unstructured.two.layers.proposition:3.1}) we obtain
\begin{align}
    &\l|\E_{\{\b{h}_{ij}\}\cup \T^d}\l[ \Psi(\b{U}^{d}(\b{h}_{ij};i,j)) \r]- \E_{\T^d}\l[\Psi(\b{U}^{d}) \r]  \r| \nonumber \\
    &\leq \frac{ u \tau_0^2 \nu_1^2 }{2} \E_{\{\b{h}_{ij}\} \cup \T^d}\l[ \l| \frac{\partial^2 \Psi(\b{U}^{d}(\hat{t}_{ij};i,j))}{\partial t^2} \r| \biggm| \b{h}_{ij}=0 \r]\nonumber\\
    &+ \frac{ (1-u)\tau_1^2 \nu_1^2 }{2} \E_{\{\b{h}_{ij}\} \cup \T^d}\l[ \l| \frac{\partial^2 \Psi(\b{U}^{d}(\hat{t}_{ij};i,j))}{\partial t^2} \r| \biggm| \b{h}_{ij}\neq 0 \r].\label{chapter3:appendix:unstructured.two.layers.proposition:3.4}
\end{align}
Let $\eta \in \R_{\geq 1}$ and introduce
\begin{alignat}{2}
    \mathcal{C}_0 &\triangleq \Bigg\{ (i, j) \in \mathcal{S}^d_1 \Biggm| \quad \E_{\{\b{h}_{ij}\} \cup \T^d}&&\l[ \l|\frac{\partial^2 \Psi(\b{U}^{d}(\hat{t}_{ij};i,j))}{\partial t^2} \r| \Biggm| \b{h}_{ij}=0  \r] \leq  \frac{2\eta  \theta^2 \sigma_2 \sqrt{n_2}}{|\Sm^d_1|} \nonumber\\
    & &&\times \l(  \sqrt{\E_{\T^d} \l[ \Psi(\b{U}^d) \r]} + \frac{\nu_2^2 n_3\sqrt{n_2}}{\sigma_2}  + \theta \tau_0  \nu_1 \nu_2 \sqrt{n_2 n_3}\r)\Bigg\},\nonumber\\
    \mathcal{C}_1 &\triangleq \Bigg\{ (i, j) \in \mathcal{S}^d_1 \Biggm| \quad \E_{\{\b{h}_{ij}\} \cup \T^d}&&\l[ \l|\frac{\partial^2 \Psi(\b{U}^{d}(\hat{t}_{ij};i,j))}{\partial t^2} \r| \Biggm| \b{h}_{ij}\neq 0  \r] \leq  \frac{2\eta  \theta^2 \sigma_2 \sqrt{n_2}}{|\Sm^d_1|}  \nonumber\\
    & && \times\l( \sqrt{\E_{\T^d} \l[ \Psi(\b{U}^d) \r]} + \frac{\nu_2^2 n_3\sqrt{n_2}}{\sigma_2}  + \theta \tau_1  \nu_1 \nu_2 \sqrt{n_2 n_3}\r)\Bigg\} \nonumber.
\end{alignat}
Using Lemma \ref{chapter3:lemma:7}, it follows that  $|\mathcal{C}_0|, |\mathcal{C}_1| \geq \l(1- \frac{1}{\eta}\r)|\mathcal{S}^d_1|$.  Since $\alpha\leq1-\frac{1}{n_1 n_2}$, we have $|\Sm^d_1|\geq n_1 n_2 - \lfloor \alpha n_1 n_2 \rfloor + 1 \geq 2$. Setting $\eta=4$, it follows that $|\mathcal{C}_0 \cap \mathcal{C}_1| \geq |\mathcal{C}_0|+|\mathcal{C}_1|-|\Sm^d_1| \geq \l(1- \frac{2}{\eta}\r)|\Sm^d_1|  \geq (1-2/\eta)2=1$, and thus $\mathcal{C}_0 \cap \mathcal{C}_1  \neq \emptyset$. Henceforth, if we pick the pair $(i,j) \in \Sm_1^d$ with the smallest \textit{score} given by $\l|\E_{\{\b{h}_{ij}\}\cup \T^d}\l[ \Psi(\b{U}^{d}(\b{h}_{ij};i,j)) \r] - \E_{\T^d}\l[\Psi(\b{U}^{d}) \r]  \r|$, it follows from (\ref{chapter3:appendix:unstructured.two.layers.proposition:3.4}) that
\begin{align}
    &\l| \E_{\{\b{h}_{ij}\}\cup \T^d}\l[ \Psi(\b{U}^{d}(\b{h}_{ij};i,j)) \r]  - \E_{\T^d}\l[\Psi(\b{U}^{d})\r] \r| \nonumber\\
    &\leq \frac{4\theta^2   (u \tau_0^2 + (1-u)\tau_1^2) \sigma_2 \nu_1^2 \sqrt{n_2}}{|\Sm^d_1| }  \l(  \sqrt{\E_{\T^d} \l[ \Psi(\b{U}^d) \r]} + \frac{\nu_2^2 n_3\sqrt{n_2}}{\sigma_2}  + \theta \tau  \nu_1 \nu_2 \sqrt{n_2 n_3}\r)\nonumber\\
    &\leq \frac{4 \theta^2 \Delta  \sigma_2 \nu_1^2   \sqrt{n_2 }}{|\Sm^d_1|} \l(  \frac{\E_{\T^d} \l[ \Psi(\b{U}^d) \r]}{2\theta} + \frac{\theta}{2} + \frac{\nu_2^2 n_3\sqrt{n_2}}{\sigma_2}  + \theta \tau  \nu_1 \nu_2 \sqrt{n_2 n_3}\r)\label{chapter3:appendix:unstructured.two.layers.proposition:3.5} \\
    &\leq \frac{2 \theta \Delta \sigma_2 \nu_1^2 \sqrt{n_2 }}{|\Sm^d_1|} \l(  \E_{\T^d} \l[ \Psi(\b{U}^d) \r] +  \theta^2+\frac{2\theta \nu_2^2 n_3 \sqrt{n_2}}{\sigma_2} + 2\theta^2 \tau \nu_1 \nu_2 \sqrt{n_2 n_3} \r), \label{chapter3:appendix:unstructured.two.layers.proposition:3.6}
\end{align}
where we used the arithmetic-geometric inequality in (\ref{chapter3:appendix:unstructured.two.layers.proposition:3.5}). Suppose we repeat this switching operation as long as $|\mathcal{S}^d_1|> n_1 n_2 - \lfloor\alpha n_1 n_2\rfloor$, and let $\mathcal{K}= \Sm^{\lfloor \alpha n_1 n_2 \rfloor}_2$. That is, $\mathcal{K}$ is the set of all swapped pairs $(i, j)$ at the end of the interpolation process, so that $|\mathcal{K}|= \lfloor \alpha n_1 n_2 \rfloor$. Let $\beta_d =\E_{\T^d}[\Psi(\b{U}^d)]$ for $d\in [0, \lfloor \alpha n_1 n_2 \rfloor]$ and note that if we switch entry $(i, j)$  at the $d$-th step, then $\beta_d = \E_{\T^d}\l[\Psi(\b{U}^{d}([\b{W}_1]_{ij};i,j))\r]=\E_{\T^d}\l[\Psi(\b{U}^{d})\r]$ and 
$\beta_{d+1} = \E_{\T^{d+1}}\l[\Psi(\b{U}^{d}(\b{h}_{ij};i,j))\r]$. Therefore, we have using (\ref{chapter3:appendix:unstructured.two.layers.proposition:3.6}) that
\begin{align}
    |\beta_{d+1} - \beta_{d}| &\leq \lambda \l(  \beta_d + \omega \r),\nonumber
\end{align}
where  $\lambda = \frac{2\theta \Delta \sigma_2 \nu_1^2  }{(1-\alpha)n_1 \sqrt{n_2}}$,  $\omega =  \theta^2+\frac{2\theta \nu_2^2 n_3 \sqrt{n_2}}{\sigma_2} + 2\theta^2 \tau \nu_1 \nu_2 \sqrt{n_2 n_3} $, and we used $|\Sm^d_1|\geq (1-\alpha)n_1 n_2$. Using Lemma \ref{chapter3:lemma:3}, we obtain
\begin{align}
    |\beta_{d+1} - \beta_d| &\leq \exp\l(\lambda \lfloor \alpha n_1 n_2 \rfloor\r)\lambda (\beta_0 + \omega).  \nonumber
\end{align}
Summing the above over $d\in [0, \lfloor \alpha n_1 n_2 \rfloor -1]$ we obtain
\begin{align}
    |\beta_{\lfloor \alpha n_1 n_2 \rfloor} - \beta_0| &\leq \exp\l( \lambda \lfloor \alpha n_1 n_2 \rfloor\r) \lambda \lfloor \alpha n_1 n_2 \rfloor (\beta_0 + \omega) \nonumber \\
    &\leq f\l(\lambda \alpha n_1 n_2 \r)\l(\beta_0 + \omega\r) \nonumber \\
    &= f(\Lambda)(\beta_0 + \omega), \nonumber
\end{align}
where $\Lambda = \lambda\alpha n_1 n_2=\frac{2 \alpha \theta \Delta \sigma_2 \nu_1^2  \sqrt{n_2}}{1-\alpha}$. Let $\b{U} = \b{U}^{\lfloor \alpha n_1 n_2 \rfloor}$, and $\T = \T^{\lfloor \alpha n_1 n_2 \rfloor}$. Since $\|\b{\varphi}_1\|_{\rm Lip} \leq 1$ by (\ref{chapter3:assumption:1.2}), we have that  $\beta_0 = \Psi(\b{W}_1) \leq (\sigma_1 \sigma_2)^2 \E\l[\|\b{x} - \b{x}'\|^2\r]$.  Therefore,
\begin{align}
    \E_{\T}[\Psi(\b{U})]  &\leq  (1 + f\l(\Lambda \r)) (\sigma_1 \sigma_2)^2 \E\l[\|\b{x} - \b{x}'\|^2\r] + f\l(\Lambda \r) \omega \label{chapter3:appendix:unstructured.two.layers.proposition:3.7}.
\end{align}
We next show
\begin{align}
    \E_{\T}\l[\Pi^{\kappa}(\b{U})\r] &\leq  (1 + f\l(\Lambda \r)) (\sigma_1 \sigma_2)^2 \E\l[\|\b{x} - \b{x}'\|^2\r] + f\l(\Lambda \r) \omega. \label{chapter3:appendix:unstructured.two.layers.proposition:3.8}
\end{align}
By (\ref{chapter3:appendix:unstructured.two.layers.proposition:3.7}), it suffices to prove that
\begin{align}
    \E_{\T}\l[\Pi^{\kappa}(\b{U})\r] &\leq \E_{\T}\l[\Psi(\b{U})\r],  \label{chapter3:appendix:unstructured.two.layers.proposition:3.9}
\end{align}
which we show next. Fix a pair of vectors $(\b{x}, \b{x}') \in \mathbb{B}^{n_1}_{2}(\rho) \times \mathbb{B}^{n_1}_{2}(\theta)$, and denote by $\b{z}, \b{z}'$ the output vectors $\varphi_2(\b{W}_2 \varphi_1(\b{W}_1 \b{x})),  \varphi_2(\b{W}_2 \varphi_1(\b{U} \b{x}'))$. As $\|\b{x}\|\leq \rho$   it follows from (\ref{chapter3:assumption:1.1}) and (\ref{chapter3:assumption:1.2}) that $  \|\b{z}\| = \|\varphi_2(\b{W}_2\varphi_1(\b{W}_1 \b{x}))\| \leq \sigma_1 \sigma_2 \rho $. Since  $\|\b{z}\|\leq \sigma_1 \sigma_2 \rho $ and $\sigma_1 \sigma_2 \rho \leq \kappa$ from the assumptions of Proposition \ref{chapter3:appendix:unstructured.two.layers.proposition:3}, we have using Lemma \ref{chapter3:lemma:1} that 
\begin{align}
    \|\b{z} - [\b{z}']_{\kappa}\|  &\leq \|\b{z} - \b{z}'\|\nonumber \\
    &= \|\varphi_2(\b{W}_2\varphi_1(\b{W}_1 \b{x})) - \varphi_2(\b{W}_2\varphi_1(\b{U}\b{x}'))\| \nonumber \\
    &\leq \|\b{W}_2 \varphi_1(\b{W}_1\b{x}) - \b{W}_2 \varphi_1(\b{U} \b{x}')\|, \nonumber
\end{align}
where we used  $\|\varphi_2\|_{\rm Lip}\leq 1$ in the last line. Therefore,
\begin{align}
    \E_{\T}\l[\Pi^{\kappa}(\b{U})\r] &= \E_{\T}\E_{(\b{x}, \b{x}')}\l[\|\b{z} - [\b{z}']_{\kappa}\|^2\r]\nonumber\\
    &\leq  \E_{\T}\E_{(\b{x}, \b{x}')}\l[\|\b{W}_2 \varphi_1(\b{W}_1\b{x}) - \b{W}_2 \varphi_1(\b{U} \b{x}')\|^2\r]\nonumber\\
    &= \E_{\T}\l[\Psi(\b{U})\r]. \nonumber
\end{align}
The latter completes the proof of (\ref{chapter3:appendix:unstructured.two.layers.proposition:3.8}). Note that by construction
\begin{align}
    \b{U} =(\b{1}\b{1}^\top-\b{M}) \odot \b{h} + \b{M} \odot \b{W}_1 ,  \nonumber
\end{align}
where $\b{M}_{ij} = 1_{(i,j)\not \in \mathcal{K}}$ is a mask matrix satisfying $\|\b{M}\|_{0} = n_1 n_2 - |\mathcal{K}|= n_1 n_2 - \lfloor \alpha n_1 n_2 \rfloor $. The latter combined with (\ref{chapter3:appendix:unstructured.two.layers.proposition:3.8}) ends the proof of Proposition \ref{chapter3:appendix:unstructured.two.layers.proposition:3}. 
\end{proof}

\subsubsection{Proof of Proposition \ref{chapter3:appendix:unstructured.two.layers.proposition:1}}
\begin{proof}[Proof of Proposition \ref{chapter3:appendix:unstructured.two.layers.proposition:1}]
Let $\b{h}_{ij} = q\l([\b{W}_1]_{ij}; \nu_1, k\r)$ for $(i, j) \in [n_2] \times [n_1]$. We then have by Lemma \ref{chapter3:lemma:0} that $\E[\b{h}_{ij}]=[\b{W}_1]_{ij}$, and $\tau_1=\tau_0 = \frac{1}{k}$. Therefore $\Delta = \frac{1}{k^2}$. Using Proposition \ref{chapter3:appendix:unstructured.two.layers.proposition:3}, there exists a mask matrix $\b{M}\in \{0, 1\}^{n_2 \times n_1}$ satisfying 
\begin{align}
    \frac{\|\b{M}\|_0}{n_1 n_2} &= 1 - \frac{\lfloor \alpha n_1 n_2 \rfloor}{n_1 n_2},\nonumber
\end{align}
and
    \begin{align}
        \E_{\b{h}}\l[\Pi^{\kappa}\l( (\b{1} \b{1}^\top - \b{M})\odot \b{h} + \b{M} \odot \b{W}_1\r) \r] &\leq   (1 + f\l(\Lambda \r)) (\sigma_1 \sigma_2)^2 \E\l[\|\b{x} - \b{x}'\|^2\r] + f\l(\Lambda \r) \omega, \label{chapter3:proposition:twolayer.2}
    \end{align}
    where $\Lambda = \frac{2 \alpha\theta  \sigma_2 \nu_1^2 \sqrt{n_2}}{k^2(1-\alpha)}$ and  $\omega =  \theta^2+\frac{2\theta \nu_2^2 n_3 \sqrt{n_2}}{\sigma_2} + \frac{2\theta^2  \nu_1 \nu_2 \sqrt{n_2 n_3}}{k} $. Since the above inequality holds in expectation over $\b{h}$, it follows that there exists a (nonrandom) realization $\b{Q} \in\l\{ \pm \frac{\ell \nu_1 }{k} \big| \ell \in [k]\r\}^{n_2 \times n_1}$ of $\b{h}$ such that
\begin{align}
    \Pi^{\kappa}\l((\b{1}\b{1}^\top-\b{M}) \odot \b{Q} + \b{M} \odot \b{W}_1  \r) &\leq \E\l[\Pi^{\kappa}\l((\b{1}\b{1}^\top-\b{M}) \odot \b{h} + \b{M} \odot \b{W}_1    \r)\r]\nonumber \\
    &\leq (1 + f\l(\Lambda \r))(\sigma_1 \sigma_2)^2 \E\l[\|\b{x} - \b{x}'\|^2\r] + f\l(\Lambda \r) \omega.\nonumber
\end{align}
This ends the proof of Proposition \ref{chapter3:appendix:unstructured.two.layers.proposition:1}.
\end{proof}

\subsubsection{Proof of Proposition \ref{chapter3:appendix:unstructured.two.layers.proposition:2}} 
\begin{proof}[Proof of Proposition \ref{chapter3:appendix:unstructured.two.layers.proposition:2}]
Let $\b{h} = \frac{\b{Ber}^{n_2, n_1}(p)}{p} \odot \b{W}_1$. We have $\E[\b{h}_{ij}]=[\b{W}_1]_{ij}$. Furthermore, if $[\b{W}_1]_{ij}\neq 0$ then $\tau_0=1, \tau_1=\frac{1-p}{p}$ which yields $\p(\b{h}_{ij}=0) \tau_0^2 + (1-\p(\b{h}_{ij}=0)) \tau_1^2 =1-p + \frac{(1-p)^2}{p} = \frac{1-p}{p}$ and $\tau = \frac{p\vee (1-p)}{p}$. If $[\b{W}_1]_{ij}=0$, then $\b{h}_{ij}=0$ and thus $\tau_0=\tau=0$. Combining both cases, we have $\Delta \leq \frac{1-p}{p}$. Using Proposition \ref{chapter3:appendix:unstructured.two.layers.proposition:3}, there exists a mask matrix $\b{M}\in \{0, 1\}^{n_2 \times n_1}$ satisfying 
\begin{align}
    \frac{\|\b{M}\|_0}{n_1 n_2} &= 1 - \frac{\lfloor \alpha n_1 n_2 \rfloor}{n_1 n_2},  \label{chapter3:appendix:unstructured.two.layers.proposition:2.3}
\end{align}
and
\begin{align}
    \E\l[\Pi^{\kappa}\l((\b{1}\b{1}^\top-\b{M}) \odot \b{h} + \b{M} \odot \b{W}_1  \r)\r] &\leq (1 + f\l(\Lambda \r))(\sigma_1 \sigma_2)^2 \E\l[\|\b{x} - \b{x}'\|^2\r] + f\l(\Lambda \r) \omega, \nonumber
\end{align}
where $\Lambda = \frac{2 \alpha \theta (1-p) \sigma_2 \nu_1^2 \sqrt{n_2}}{p(1-\alpha)}$ and  $\omega =  \theta^2+\frac{2\theta \nu_2^2 n_3 \sqrt{n_2}}{\sigma_2} + \frac{2\theta^2 (p\vee (1-p)) \nu_1 \nu_2 \sqrt{n_2 n_3}}{p} $. Let $\b{U} = (\b{1}\b{1}^\top-\b{M}) \odot \b{h} + \b{M} \odot \b{W}_1 $ and $\b{b} = (\b{1} \b{1}^{\top}  - \b{M}) \odot \b{h}$. Using Lemma \ref{chapter3:lemma:4} and following the same  arguments in the proof of Proposition \ref{chapter3:appendix:unstructured.single.proposition:2}, there exists a positive constant $n_0=n_0(\gamma, p, \alpha, \varepsilon)$ such that if $n_1\vee n_2 \geq n_0$, then there exists a (nonrandom) realization $\hat{\b{b}}, \hat{\b{U}}, \hat{\b{M}}$ of $\b{b}, \b{U}, \b{M}$ such that $\|\hat{\b{b}}\|_0 \leq (1+\gamma)p\lfloor \alpha n_1 n_2 \rfloor $, and
    \begin{align}
        \Pi^{\kappa}(\hat{\b{U}}) &\leq  (1+\varepsilon)(1 + f\l(\Lambda \r))(\sigma_1 \sigma_2)^2 \E\l[\|\b{x} - \b{x}'\|^2\r]+(1+\varepsilon)f\l(\Lambda \r) \omega. 
    \end{align}
    This shows (\ref{chapter3:appendix:unstructured.two.layers.proposition:2.2}). Furthermore, we have using (\ref{chapter3:appendix:unstructured.two.layers.proposition:2.3})
    \begin{align}
        \|\hat{\b{U}}\|_0 &\leq \|\hat{\b{M}}\|_0 + \|\hat{\b{b}}\|_0\nonumber\\
        &\leq \l(1-\frac{\lfloor \alpha n_1 n_2 \rfloor}{n_1 n_2}\r) n_1 n_2 + (1+\gamma)p\lfloor \alpha n_1 n_2 \rfloor.\nonumber
    \end{align}
    Note that $\hat{\b{b}}$ is the entrywise product of a mask with entries in $\l\{0,  \frac{1}{p}\r\}$ and the matrix $\b{W}_1$. Therefore, $\hat{\b{U}}$ is the product of a matrix with entries in $\l\{0, 1, \frac{1}{p}\r\}$ and $\b{W}_1$. Combining the latter with the previous bound readily shows (\ref{chapter3:appendix:unstructured.two.layers.proposition:2.1}) and concludes the proof of the proposition.
\end{proof}

\subsection{Multilayer Perceptron}\label{chapter3:appendix:unstructured.deep}

We show Proposition \ref{chapter3:proposition:1} below. The proof of Proposition \ref{chapter3:proposition:2} is identical to the proof of Proposition \ref{chapter3:proposition:1} with the only difference being the use of Propositions \ref{chapter3:appendix:unstructured.single.proposition:1} and \ref{chapter3:appendix:unstructured.two.layers.proposition:1} instead of Propositions \ref{chapter3:appendix:unstructured.single.proposition:2} and \ref{chapter3:appendix:unstructured.two.layers.proposition:2}. We thus  skip its proof.
\begin{proof}[Proof of Proposition \ref{chapter3:proposition:1}]
    We will construct sparse matrices $\hat{\b{W}}_\ell$ recursively and verify conditions (\ref{chapter3:lemma:9.1}), (\ref{chapter3:lemma:9.2}), and (\ref{chapter3:lemma:9.3}) from Lemma \ref{chapter3:lemma:9}, which would then readily yield (\ref{chapter3:proposition:1.2}) for small enough $\delta$ by an application of (\ref{chapter3:lemma:9.4}). Following the notation of Lemma \ref{chapter3:lemma:9}, let $\b{z}^\ell, \hat{\b{z}}^\ell$ be given recursively by $\b{z}^0 = \hat{\b{z}}^0 = \b{x}$, and for $\ell \in [m]$
    \begin{alignat}{2}
        &\b{z}^{\ell} &&= \varphi_{\ell}(\b{W}_\ell \b{z}^{\ell-1}),\nonumber\\
        &\hat{\b{z}}^{\ell} &&= \varphi_{\ell}(\hat{\b{W}}_\ell [\hat{\b{z}}^{\ell-1}]_{\kappa_{\ell-1}}),\nonumber
    \end{alignat}
    where $\kappa_\ell = c_1^\ell$ and $\hat{\b{W}}_{\ell}$ denotes the $\ell$-th weight matrix of the pruned network $\hat{\Phi}$. Namely, the vectors $\b{z}^\ell, \hat{\b{z}}^\ell$ track the outputs of the dense/pruned networks for $\ell\in [ m]$. Introduce $\sigma_\ell = \|\b{W}_\ell\|$, $\nu_\ell = \|\b{W}_\ell\|_{\infty}$ for $\ell \in [m]$. Let $\alpha\approx 0.99,\gamma \approx 0.01$, and $\varepsilon=\xi/2$. Denote by $n_0^1$ the constant $n_0=n_0(\gamma, p,\alpha, \varepsilon)=n_0(p,\xi)$ in Proposition \ref{chapter3:appendix:unstructured.single.proposition:2}. Similarly, denote by $n_0^2$ the constant $n_0=n_0(\gamma, p,\alpha, \varepsilon)=n_0(p,\xi)$ in Proposition \ref{chapter3:appendix:unstructured.two.layers.proposition:2}. Set $n_0 = n_0(p, \xi) = n_0^1 \vee n_0^2 \vee  \frac{1}{1-\alpha}$. we consider $3$ cases.
    \begin{enumerate}
        \item Case 1. $\ell \not \in \mathcal{W} \cup \mathcal{B} $. In this case, we set $\hat{\b{W}}_\ell = \b{W}_\ell$, which satisfies (\ref{chapter3:lemma:9.1}) in Lemma \ref{chapter3:lemma:9}.

        \item Case 2. $\ell\in \mathcal{W}$. Note that $(\b{z}^{\ell-1}, [\hat{\b{z}}^{\ell-1}]_{\kappa_{\ell-1}}) \in \mathbb{B}^{n_{\ell}}_2(\kappa_{\ell-1}) \times \mathbb{B}^{n_{\ell}}_{2}(\kappa_{\ell-1})$. Applying Proposition \ref{chapter3:appendix:unstructured.single.proposition:2} with $(\b{x}, \b{x}') = (\b{z}^{\ell-1}, [\hat{\b{z}}^{\ell-1}]_{\kappa_{\ell-1}})$,  $(\rho, \theta) = (\kappa_{\ell-1}, \kappa_{\ell-1}),\kappa=\kappa_\ell=c_1 \rho$ and $\alpha\approx0.99, \gamma\approx 0.01$ (note that by construction of $n_0$, the inequality $\alpha \leq 1- 1/n_{\ell}n_{\ell+1}$ holds in the statement of  Proposition \ref{chapter3:appendix:unstructured.single.proposition:2}), it follows that there exists a mask matrix $\b{M}\in \l\{0, 1, \frac{1}{p}\r\}^{n_{\ell+1} \times n_{\ell}}$, such that  $\frac{\|\b{M}\|_0}{n_{\ell} n_{\ell+1}} \leq 0.01 + 1.01p$, and 
            \begin{align}
        \Pi^{\kappa_{\ell}}(\b{M} \odot \b{W}_{\ell}) &\leq \l(1+\frac{\xi}{2}\r)\sigma_\ell^2  \E\l[\|\b{z}^{\ell-1} - [\hat{\b{z}}^{\ell-1}]_{\kappa_{\ell-1}}\|^2\r]  \nonumber \\
        &+  \l(1+\frac{\xi}{2}\r)\underbrace{\frac{2 \alpha \kappa_{\ell-1}^2 (1-p)\nu_{\ell}^2  n_{\ell+1}}{p(1-\alpha)}}_{t} \label{chapter3:proposition:1.case2.1}. 
        \end{align}
        Set $\hat{\b{W}}_{\ell} = \b{M} \odot \b{W}_{\ell} $ in $\hat{\Phi}$. It follows that
        \begin{align}
            \frac{\|\hat{\b{W}}_{\ell}\|_0}{n_{\ell}n_{\ell+1}} &\leq 0.01 + 1.01p. \label{chapter3:proposition:1.case2.2}
        \end{align}
        We now bound $t$. We have using $\nu_{\ell} \leq c_2/\sqrt{n_{\ell}}$ from (\ref{chapter3:assumption:1.5}) 
        \begin{align*}
            t &\leq   \frac{2c_2^2   \kappa_{\ell-1}^2 \alpha (1-p) }{p(1-\alpha)} \frac{n_{\ell+1}}{n_{\ell}}\\
            &=\frac{2c_2^2 c_1^{2(\ell-1)} \alpha(1-p) }{p(1-\alpha)} \frac{n_{\ell+1}}{n_{\ell} }\\
            &\leq 2c_1^{2\ell} \delta,
        \end{align*}
        where we used (\ref{chapter3:proposition:1.w}) in the last line. Set $\varepsilon_1 = \frac{\xi}{2}$ and $\varepsilon_2 = 2\delta\l(1+\frac{\xi}{2}\r)$. It follows that 
        \begin{align}
            \E\l[\|\b{z}^{\ell}-[\hat{\b{z}}^{\ell}]_{\kappa_\ell}\|^2\r]  &=\Pi^{\kappa_{\ell}}(\hat{\b{W}}_{\ell})\nonumber \\ &\leq \l(1+\varepsilon_1\r) c_1^2\E\l[\|\b{z}^{\ell-1} - [\hat{\b{z}}^{\ell-1}]_{\kappa_{\ell-1}}\|^2\r]  + c_1^{2\ell}\varepsilon_2 
        \end{align}
        which satisfies (\ref{chapter3:lemma:9.2}) in Lemma \ref{chapter3:lemma:9} with $\sigma = c_1$.

        \item Case 3. $\ell \in \mathcal{B}$. Note that $(\b{z}^{\ell-1}, [\hat{\b{z}}^{\ell-1}]_{\kappa_{\ell-1}}) \in \mathbb{B}^{n_{\ell}}_2(\kappa_{\ell-1}) \times \mathbb{B}^{n_{\ell}}_{2}(\kappa_{\ell-1})$. Applying Proposition \ref{chapter3:appendix:unstructured.two.layers.proposition:2} with $(\b{x}, \b{x}') = (\b{z}^{\ell-1}, [\hat{\b{z}}^{\ell-1}]_{\kappa_{\ell-1}})$,  $(\rho, \theta) = (\kappa_{\ell-1}, \kappa_{\ell-1}),\kappa=\kappa_{\ell+1}=c_1^2 \rho$ and $\alpha\approx0.99, \gamma\approx 0.01$ (note that by construction of $n_0$, the inequality $\alpha \leq 1- 1/n_{\ell}n_{\ell+1}$ holds in the statement of Proposition \ref{chapter3:appendix:unstructured.two.layers.proposition:2}), it follows that there exists a  mask  matrix $\b{M} \in \l\{0, 1, \frac{1}{p}\r\}^{n_{\ell+1} \times n_{\ell}}$ such that $\frac{\|\b{M}\|_0}{n_{\ell} n_{\ell+1}} \leq 0.01 + 1.01p$, and
        \begin{align}
            \Pi^{\kappa_{\ell+1}}(\b{M} \odot \b{W}_{\ell}) &\leq\l(1+\frac{\xi}{2}\r)(1 + f\l( \Lambda \r))(\sigma_{\ell} \sigma_{\ell+1})^2 \E\l[\|\b{z}^{\ell-1} - [\hat{\b{z}}^{\ell-1}]_{\kappa_{\ell-1}}\|^2\r] \nonumber \\
            &+\l(1+\frac{\xi}{2}\r)f\l( \Lambda \r) \omega,   \label{chapter3:proposition:1.case3.1}
        \end{align}
        where 
        \begin{align*}
            \Lambda &= \frac{2\alpha \kappa_{\ell-1} (1-p) \sigma_{\ell+1}\nu_{\ell}^2\sqrt{n_{\ell+1}}}{p(1-\alpha)},\\
            \omega &=  \kappa_{\ell-1}^2+\frac{2\kappa_{\ell-1} \nu_{\ell+1}^2 n_{\ell+2} \sqrt{n_{\ell+1}}}{\sigma_{\ell+1}} + \frac{2\kappa_{\ell-1}^2(p\vee (1-p))  \nu_{\ell} \nu_{\ell+1} \sqrt{n_{\ell+1} n_{\ell+2}}}{p}.
        \end{align*}
        Set $\hat{\b{W}}_{\ell} = \b{M} \odot \b{W}_{\ell} $ in $\hat{\Phi}$. It follows that
        \begin{align}
            \frac{\|\hat{\b{W}}_{\ell}\|_0}{n_{\ell}n_{\ell+1}} &\leq 0.01 + 1.01p. \label{chapter3:proposition:1.case3.2}
        \end{align}
        We next bound $\Lambda$ and $\omega$. We have using $\nu_{\ell} \leq c_2/\sqrt{n_{\ell+1}}$ from (\ref{chapter3:assumption:1.5}) and (\ref{chapter3:proposition:1.b.1})
        \begin{align}
            \Lambda &\leq \frac{2c_2^2 c_1^{\ell} \alpha(1-p)}{p(1-\alpha)} \frac{1}{\sqrt{n_{\ell+1}}} \leq 2\delta \label{chapter3:proposition:1.case3.3}.
        \end{align}
        Similarly, we have
        \begin{align}
            f(\Lambda)\omega &= e^{\Lambda} \Lambda \l(\kappa_{\ell-1}^2+\frac{2\kappa_{\ell-1} \nu_{\ell+1}^2 n_{\ell+2} \sqrt{n_{\ell+1}}}{\sigma_{\ell+1}} + \frac{2(p\vee (1-p))\kappa_{\ell-1}^2  \nu_{\ell} \nu_{\ell+1} \sqrt{n_{\ell+1} n_{\ell+2}}}{p}\r) \nonumber \\
            &= e^{\Lambda}\Bigg( \frac{2\alpha \kappa_{\ell-1}^3 (1-p)\sigma_{\ell+1}\nu_\ell^2 \sqrt{n_{\ell+1}}}{p(1-\alpha)} + \frac{4\alpha(1-p) \kappa_{\ell-1}^2 \nu_{\ell}^2 \nu_{\ell+1}^2 n_{\ell+1} n_{\ell+2}}{p(1-\alpha)} \\&+ \frac{4\alpha(p(1-p)\vee (1-p)^2) \kappa_{\ell-1}^3 \sigma_{\ell+1} \nu_{\ell}^3 \nu_{\ell+1}n_{\ell+1}\sqrt{n_{\ell+2}}}{p^2(1-\alpha)}\Bigg) \nonumber \\
            &\leq \frac{\alpha(1-p)e^{\Lambda}}{p(1-\alpha)} \l( 2c_2^2 c_1^{3\ell-2}\frac{1}{\sqrt{n_{\ell+1}}} + 4c_2^4 c_1^{2(\ell -1)} \frac{ n_{\ell+2}}{n_{\ell+1}} + \frac{4c_2^4c_1^{3\ell -2} (p\vee (1-p))}{p} \frac{1}{\sqrt{n_{\ell+1}}}\r) \label{chapter3:proposition:1.case3.4}\\
            &\leq \frac{\alpha(1-p)e^{\Lambda}}{p(1-\alpha)} \l( 2c_1^{3\ell-2}\l(c_2^2 + 2c_2^4\l(1\vee \frac{1-p}{p}\r)\r) \frac{1}{\sqrt{n_{\ell+1}}}  + 4c_2^4 c_1^{2(\ell -1)} \frac{ n_{\ell+2}}{n_{\ell+1}} \r) \nonumber  \\
            &\leq c_1^{2(\ell+1)}e^{2\delta} \l(6\delta + 4 \delta\r) \label{chapter3:proposition:1.case3.5} \\
            &=c_1^{2(\ell+1)} 5f(2\delta), \nonumber
        \end{align}
        where we used $\sigma_{\ell+1}\leq c_1$ and $\nu_{\ell}, \nu_{\ell+1}\leq c_2/\sqrt{n_{\ell+1} \vee n_{\ell+2}}$ in line (\ref{chapter3:proposition:1.case3.4}), and (\ref{chapter3:proposition:1.b.1}) and (\ref{chapter3:proposition:1.b.2}) in line (\ref{chapter3:proposition:1.case3.5}).  Setting $\varepsilon_3 = \frac{\xi}{2} + \l(1 + \frac{\xi}{2}\r)f(2\delta)$ and $\varepsilon_4 = \l(1+\frac{\xi}{2}\r)5f(2\delta)$, it follows that 
        \begin{align}
            \E\l[\|\b{z}^{\ell}-[\hat{\b{z}}^{\ell}]_{\kappa_\ell}\|^2\r]  &=\Pi^{\kappa_{\ell}}(\hat{\b{W}}_{\ell})\nonumber \\ &\leq (1+\varepsilon_3)c_1^4 \E\l[\|\b{z}^{\ell-1} - [\hat{\b{z}}^{\ell-1}]_{\kappa_{\ell-1}}\|^2\r]  + c_1^{2(\ell+1)}\varepsilon_4,  \nonumber
        \end{align}
        which satisfies (\ref{chapter3:lemma:9.3}) in Lemma \ref{chapter3:lemma:9}.
    \end{enumerate}
    Using Lemma \ref{chapter3:lemma:9}, it follows that (\ref{chapter3:proposition:1.2}) holds for small enough $\delta=\delta(\xi)$. Moreover, (\ref{chapter3:proposition:1.1}) readily holds from (\ref{chapter3:proposition:1.case2.2}) and (\ref{chapter3:proposition:1.case3.2}). This concludes the proof of (\ref{chapter3:proposition:1.1}) and (\ref{chapter3:proposition:1.2}). Note then that
\begin{align}
    \mathcal{L}(\hat{\Phi}; \mathcal{D}) &= \E_{(\b{x}, \b{y})\sim \mathcal{D}} \l[\|\hat{\Phi}(\b{x}) - \b{y}\|^2 \r]\nonumber\\
    &= \E_{(\b{x}, \b{y}) \sim \mathcal{D}} \l[\|\hat{\Phi}(\b{x}) - \Phi(\b{x}) + \Phi(\b{x})  - \b{y}\|^2 \r]\nonumber\\
    &= \mathcal{L}(\Phi; \mathcal{D}) + 2 \E[\langle \hat{\Phi}(\b{x}) - \Phi(\b{x}),\Phi(\b{x})  - \b{y} \rangle ] + \E_{\b{x}} [\| \hat{\Phi}(\b{x}) - \Phi(\b{x})\|^2]\nonumber\\
    &\leq \mathcal{L}(\Phi; \mathcal{D}) + 2 \sqrt{\mathcal{L}(\Phi; \mathcal{D})} \sqrt{\E_{\b{x}} [\| \hat{\Phi}(\b{x}) - \Phi(\b{x})\|^2]} + \E_{\b{x}} [\| \hat{\Phi}(\b{x}) - \Phi(\b{x})\|^2]\nonumber\\
    &\leq \mathcal{L}(\Phi; \mathcal{D}) + 2 c_1^m \sqrt{\varepsilon \mathcal{L}(\Phi; \mathcal{D})} + c_1^{2m} \varepsilon,\nonumber
\end{align}
where we used (\ref{chapter3:proposition:1.2}) in the last line and $\varepsilon = (1+\xi)^m \xi$. This ends the proof of (\ref{chapter3:proposition:1.3}), and concludes the proof of Proposition \ref{chapter3:proposition:1}.
\end{proof}

%\newpage

%\newpage
\section{Proofs for Structured Pruning of Multilayer Perceptrons}\label{chapter3:appendix:structured.shallow}

\subsection{Single-Layer Perceptron}\label{chapter3:appendix:structured.single.layer}
In this section we consider structured pruning of a single-layer perceptron and adopt the same  introduced in Section \ref{chapter3:appendix:unstructured.single.layer} for $\Psi, \Pi^{\kappa}, \b{W}_1, \varphi_1$ and $\mathcal{Q}_{\rho, \theta}$. We now state this section's result for pruning.

\begin{proposition}\label{chapter3:appendix:structured.single.proposition}
    Suppose the activation $\varphi_1$ satisfies (\ref{chapter3:assumption:1.1}) and (\ref{chapter3:assumption:1.2}) in Assumption \ref{chapter3:assumption:1}. Let $k | \gcd(n_1, n_2)$, $p,\gamma, \varepsilon \in (0, 1)$, $\alpha \in \l(\frac{k}{n_1}, 1-\frac{k}{n_1}\r]$, $\sigma_1=\|\b{W}_1\|$, and $\kappa\in [\sigma_1 \rho, \infty)$. Finally, let  $\EE_\ell = \{k(\ell-1) + q \mid q\in [k]\}$ for $\ell\in \mathbb{Z}_{\geq 1}$, and $\nu_1=\max_{(i,j)} \|[\b{W}_1]_{\EE_i,\EE_j}\|$. There exists a constant $n_0 = n_0(\gamma, p, \alpha, \varepsilon)$ such that if $n_1/k \geq n_0$, then  there exists a diagonal mask matrix $\b{D} \in \l\{0, 1, \frac{1}{p}\r\}^{n_1 \times n_1}$ satisfying
    \begin{align}
        \l|\l\{\EE_\ell \mid \ell\in[n_1/k],  \b{D}_{\EE_\ell, \EE_\ell} \neq \b{0}_{k\times k} \r\}\r| \leq (1+\gamma)p \l\lfloor  \frac{\alpha n_1}{k} \r\rfloor + \frac{n_1}{k} - \l\lfloor \frac{\alpha n_1}{k} \r\rfloor, \label{chapter3:appendix:structured.single.proposition.1}
    \end{align}
     such that 
    \begin{align}
        \Pi^{\kappa}(\b{W}_1 \b{D}) &\leq (1+\varepsilon)\l(\sigma_1^2\E\l[\|\b{x} - \b{x}'\|^2\r] + \frac{2\alpha\theta^2(1-p) }{p(1-\alpha)} \frac{ \nu_1^2 n_2 }{k} \r). \label{chapter3:appendix:structured.single.proposition.2}
    \end{align}
\end{proposition}
\begin{proof}
Let $(h_\ell)_{\ell \in [n_1/k]}$ be a sequence of independent random variables with distribution ${\rm Bernoulli}(p) / p$. We use an interpolation argument whereby we iteratively pick a \textit{suitable} set $\mathcal{E}_\ell$ indexed by $\ell \in  [n_1/k]$ and switch the $\ell$-th block-column  from $[\b{W}_1]_{:, \mathcal{E}_{\ell}}$ to $h_{\ell} [\b{W}_1]_{:, \mathcal{E}_\ell}$. We control the resulting error from each switch, and  make $\lfloor \alpha n_1/k \rfloor$ switches. Let $\b{U}^d$ be the weight matrix at the end of the $d$-th  interpolation step where $d \in [0, \lfloor \alpha n_1/k \rfloor - 1]$, and let $\mathcal{S}^d_1 \cup \mathcal{S}^d_2 =  [n_1/k]$ track the interpolation process where $\mathcal{S}^d_1$ contains indices of  unswitched block-columns, and $\mathcal{S}^d_2$ contains indices of switched ones. Namely
\begin{alignat}{2}
    \b{U}^d_{ij} &= [\b{W}_1]_{ij}, \quad &&\text{if } j \in \bigcup_{\ell \in \mathcal{S}^d_1}\mathcal{E}_\ell,\nonumber\\
    \b{U}^d_{ij} &= h_{\ell} [\b{W}_1]_{ij} \quad &&\text{if } j \in \mathcal{E}_\ell \text{ for some } \ell \in \mathcal{S}^d_2.\nonumber
\end{alignat}
for all $(i, j) \in [n_2] \times [n_1]$. In particular $\b{U}^0 = \b{W}_1$ (i.e. $\mathcal{S}^0_1= [n_1/k], \mathcal{S}^0_2=\emptyset$). Finally, we let $\T^d = \{h_{\ell}\mid \ell \in \Sm_2^d\}$, with $\T^0 = \emptyset$. Suppose we are at step $d< \lfloor \alpha n_1/k \rfloor$ so that  $|\mathcal{S}^d_2| < \lfloor \alpha n_1/k \rfloor$ (otherwise the interpolation is over). Let $\ell \in \mathcal{S}^d_1$ (note that $\mathcal{S}^d_1 \neq \emptyset$ as $|\mathcal{S}^d_1|\geq  n_1/k  - \lfloor \alpha  n_1/k \rfloor + 1 \geq 1 $). Since $t\mapsto \Psi(\b{U}^d \b{G}(t; \mathcal{E}_\ell))$ is a quadratic function,  we have by Taylor's
\begin{align}
    \Psi(\b{U}^d \b{G}(h_{\ell}; \EE_\ell)) - \Psi(\b{U}^d) &= \Psi(\b{U}^d \b{G}(h_{\ell}; \EE_\ell)) - \Psi(\b{U}^d \b{G}(1; \EE_\ell)) \nonumber \\ 
    &= (h_{\ell}-1)\frac{\partial \Psi(\b{U}^{d} \b{G}(1; \EE_\ell))}{\partial t}  + \frac{(h_{\ell}  - 1)^2}{2}\frac{\partial^2 \Psi(\b{U}^{d}\b{G}(1; \EE_\ell))}{\partial t^2}. \nonumber
\end{align}
Taking the expectation of the above over  $\{h_{\ell}\}\cup \T^d$  we obtain
\begin{align}
    \l|\E_{\{h_\ell\} \cup \T^{d}}\l[ \Psi(\b{U}^{d}\b{G}(h_{\ell};\EE_\ell)) \r] - \E_{\T^d}\l[\Psi(\b{U}^{d}) \r]  \r|  &\leq \frac{1-p}{2p} \E_{\T^d}\l[ \l| \frac{\partial^2 \Psi(\b{U}^{d}\b{G}(1;\EE_\ell))}{\partial t^2} \r| \r].   \label{chapter3:appendix:structured.single.proposition.3}
\end{align}
Let $\eta \in \R_{\geq 1}$ and introduce
\begin{alignat}{2}
    \mathcal{C} &\triangleq \Bigg\{ \ell \in \mathcal{S}^d_1 \Biggm| \quad \E_{ \T^d}&&\l[ \l|\frac{\partial^2 \Psi(\b{U}^{d}\b{G}(1;\EE_\ell))}{\partial t^2} \r| \r] \leq  \frac{2\eta \theta^2}{|\Sm^d_1|} \frac{\nu_1^2 n_2}{k} \Bigg\}.\nonumber
\end{alignat}
Using part (\ref{chapter3:lemma:8.1}) from Lemma \ref{chapter3:lemma:8}, it follows that $|\mathcal{C}| \geq \l(1- \frac{1}{\eta}\r)|\mathcal{S}^d_1|$.  Since $\alpha\leq 1-\frac{k}{n_1}$, we have $|\Sm^d_1|\geq n_1/k - \lfloor \alpha n_1/k\rfloor + 1 \geq (1-\alpha)n_1/k  + 1 \geq 2 $. Setting $\eta=2$, it follows that $|\mathcal{C}| \geq (1-1/\eta)2=1$, and thus $\mathcal{C}  \neq \emptyset$. Henceforth, if we pick $\ell\in \Sm_1^d$ with the smallest \textit{score} given by $\l|\E_{\{h_{\ell}\}\cup \T^d}\l[ \Psi(\b{U}^{d}\b{G}(h_{\ell};\EE_\ell)) \r] - \E_{\T^d}\l[\Psi(\b{U}^{d}) \r]  \r|$, it follows from (\ref{chapter3:appendix:structured.single.proposition.3}) that
\begin{align}
    \l| \E_{\{h_{\ell}\}\cup \T^d}\l[ \Psi(\b{U}^{d}\b{G}(h_{\ell};\EE_\ell)) \r]  - \E_{\T^d}\l[\Psi(\b{U}^{d})\r] \r| &\leq \frac{2\theta^2(1-p) }{p|\Sm^d_1|} \frac{\nu_1^2 n_2}{k}. \label{chapter3:appendix:structured.single.proposition.4}
\end{align}
 We then update $\Sm^{d+1}_1 = \Sm^{d}_1 \setminus \{\ell\}$ and $\Sm^{d+1}_2 = \Sm^{d}_2 \cup \{\ell\}$. Suppose we repeat this switching operation as long as $d < \lfloor\alpha n_1/k \rfloor$, and let $\mathcal{K}= \Sm^{\lfloor \alpha n_1/k \rfloor}_2$. That is, $\mathcal{K}$ is the set of all swapped  indices $\ell$ at the end of the interpolation process, so that $|\mathcal{K}|= \lfloor \alpha n_1/k \rfloor$. Let $\beta_d =\E_{\T^d}[\Psi(\b{U}^d)]$ for $d\in [0, |\mathcal{K}|]$ and note that if we switch the index $\ell$  at the $d$-th step, then $\beta_d = \E_{\T^d}\l[\Psi(\b{U}^{d} \b{G}(1;\EE_\ell))\r]=\E_{\T^d}[\Psi(\b{U}^d)]$ and 
$\beta_{d+1} = \E_{\T^{d+1}}\l[\Psi(\b{U}^{d}\b{G}(h_{\ell};\EE_\ell))\r]$. Therefore, we can rewrite (\ref{chapter3:appendix:structured.single.proposition.4}) as 
\begin{align}
    |\beta_{d+1} - \beta_{d}| &\leq  \frac{2\theta^2(1-p) }{p|\Sm^d_1|} \frac{\nu_1^2 n_2}{k} \nonumber \\
    &\leq \frac{2\theta^2(1-p) }{p(1-\alpha)} \frac{\nu_1^2 n_2}{n_1}  ,\nonumber
\end{align}
where we used $|\Sm^d_1| \geq (1-\alpha)n_1/k$ in the last line. Telescoping the above inequality over $d\in [0, \lfloor \alpha n_1/k \rfloor -1]$ yields
\begin{align}
    \l| \beta_{\l\lfloor \frac{\alpha n_1}{k} \r\rfloor} - \beta_0 \r| &\leq  \frac{2\theta^2(1-p) }{p(1-\alpha)} \frac{ \nu_1^2 \lfloor \alpha n_1/k \rfloor n_2 }{n_1 }\nonumber\\
    &\leq  \frac{2\theta^2\alpha(1-p) }{p(1-\alpha)} \frac{ \nu_1^2 n_2 }{k}.\nonumber
\end{align}
Letting $\b{U} = \b{U}^{\l\lfloor \frac{\alpha n_1}{k} \r\rfloor}$ and $\T = \T^{\l\lfloor \frac{\alpha n_1}{k} \r\rfloor}$ and noting that $\beta_{\l\lfloor \frac{\alpha n_1}{k} \r\rfloor} = \E_{\T}[\Psi(\b{U})], \beta_0 = \Psi(\b{W}_1)$, it follows that
\begin{align}
    \E_{\T}[\Psi(\b{U})] &\leq \Psi(\b{W}_1) + \frac{2\alpha\theta^2(1-p) }{p(1-\alpha)} \frac{ \nu_1^2 n_2 }{k}\nonumber.
\end{align}
Since  $\|\b{W}_1\|=\sigma_1$ and $\|\varphi_1\|_{\rm Lip}\leq 1$, we have that $\Psi(\b{W}_1) \leq \sigma_1^2 \E[\|\b{x}-\b{x}'\|^2]$. Therefore 
\begin{align}
     \E_{\T}[\Psi(\b{U})] &\leq \sigma_1^2\E\l[\|\b{x} - \b{x}'\|^2\r] +\frac{2\alpha\theta^2(1-p) }{p(1-\alpha)} \frac{ \nu_1^2 n_2 }{k} \label{chapter3:appendix:structured.single.proposition.5}.
\end{align}
We next show
\begin{align}
    \E_{\T}\l[\Pi^{\kappa}(\b{U})\r] &\leq \sigma_1^2\E\l[\|\b{x} - \b{x}'\|^2\r] + \frac{2\alpha\theta^2(1-p)}{p(1-\alpha)} \frac{ \nu_1^2 n_2 }{k}.  \label{chapter3:appendix:structured.single.proposition.6}
\end{align}
By (\ref{chapter3:appendix:structured.single.proposition.5}), it suffices to prove that
\begin{align}
    \E_{\T}\l[\Pi^{\kappa}(\b{U})\r] &\leq \E_{\T}\l[\Psi(\b{U})\r], 
\end{align}
which we show next. Fix a pair of vectors $(\b{x}, \b{x}') \in \mathbb{B}^{n_1}_{2}(\rho) \times \mathbb{B}^{n_1}_{2}(\theta)$, and denote by $\b{z}, \b{z}'$ the output vectors $\varphi_1(\b{W}_1 \b{x}),  \varphi_1(\b{U} \b{x}')$. As $\|\b{x}\|\leq \rho$  it follows from (\ref{chapter3:assumption:1.1}) and (\ref{chapter3:assumption:1.2}) that $  \|\b{z}\| = \|\varphi_1(\b{W}_1 \b{x})\| \leq \sigma_1 \rho $. Since $\|\b{z}\|\leq \sigma_1 \rho $ and $\sigma_1\rho \leq \kappa$ from the assumptions of Proposition \ref{chapter3:appendix:structured.single.proposition}, we have using Lemma \ref{chapter3:lemma:1} that 
\begin{align}
    \|\b{z} - [\b{z}']_{\kappa}\|  &\leq \|\b{z} - \b{z}'\|\nonumber \\
    &= \|\varphi_1(\b{W}_1 \b{x}) - \varphi_1(\b{U}\b{x}')\|\\
    &\leq \|\b{W}_1 \b{x} - \b{U} \b{x}'\|.
\end{align}
Therefore,\begin{align}
    \E_{\T}\l[\Pi^{\kappa}(\b{U})\r] &= \E_{\T}\E_{(\b{x}, \b{x}')}\l[\|\b{z} - [\b{z}']_{\kappa}\|^2\r]\nonumber\\
    &\leq  \E_{\T}\E_{(\b{x}, \b{x}')}\l[\|\b{W}_1 \b{x} - \b{U} \b{x}'\|^2\r]\nonumber\\
    &= \E_{\T}\l[\Psi(\b{U})\r], \nonumber
\end{align}
where we used  $\|\varphi_1\|_{\rm Lip}\leq 1$ in the last line. This concludes the proof of (\ref{chapter3:appendix:structured.single.proposition.6}). Note  that by construction $\b{U} =\b{W}_1 \b{D}$, where $\b{D}$ is a diagonal matrix given by 
\begin{alignat}{2}
    \b{D}_{i,i} &= h_{\ell}, \quad &&\text{if } i \in  \EE_\ell \text{ and } \ell\in\mathcal{K},\nonumber\\
    \b{D}_{i,i} &= 1\quad &&\text{otherwise}.\nonumber
\end{alignat}
Let $\b{b}\in \R^{|\mathcal{K}|}$ be the vector obtained by stacking the variables $ph_\ell$ for $\ell \in \mathcal{K}$. Clearly, the entries of $\b{b}$ are independent with distribution ${\rm Bernoulli}(p)$. Moreover, $\Pi^\kappa(\b{U})$ only depends on $h_\ell$ through $\b{b}$, and
\begin{align*}
    \frac{\|\b{D}\|_0}{k} &= \sum_{\ell \in \mathcal{K}} 1_{h_\ell \neq 0} +  \frac{n_1}{k} - |\mathcal{K}| \\
    &\leq \|\b{b}\|_0 + \frac{n_1}{k} - \l\lfloor \frac{\alpha n_1}{k} \r\rfloor. 
\end{align*}
It follows from Lemma~\ref{chapter3:lemma:4} that there exists a positive constant  $n_0=n_0(\gamma , p, \varepsilon)$ such that if $\alpha n_1/k \geq n_0$, then there exist  nonrandom realizations $\hat{\b{b}}$ of $\b{b}$, and $\hat{h}_\ell,\hat{\b{U}}, \hat{\b{D}}$ of $h_\ell, \b{U}, \b{D}$ satisfying
    \begin{align}
        \Pi^{\kappa}(\hat{\b{U}}) &\leq (1+\varepsilon)  \E_{\T}\l[\Pi^{\kappa}(\b{U})\r] \nonumber \\
        &\leq (1+\varepsilon) \l( \sigma_1^2\E\l[\|\b{x} - \b{x}'\|^2\r] + \frac{2\alpha\theta^2(1-p) }{p(1-\alpha)} \frac{ \nu_1^2 n_2 }{k} \r), \label{chapter3:appendix:structured.single.proposition.7}
    \end{align}
    and $\|\hat{\b{b}}\|_0 \leq (1+\gamma)p \lfloor \frac{\alpha n_1}{k} \rfloor $. The latter implies 
    \begin{align}
        \frac{\|\hat{\b{D}}\|_0}{k} \leq (1+\gamma)p\l\lfloor \frac{\alpha n_1}{k} \r\rfloor + \frac{n_1}{k} - \l\lfloor \frac{\alpha n_1}{k} \r\rfloor \label{chapter3:appendix:structured.single.proposition.8}.
    \end{align}
    Finally, Note that by construction each null $\hat{\b{b}}_i$ leads to a null submatrix $\hat{\b{D}}_{\EE_i,\EE_i}$. Hence, (\ref{chapter3:appendix:structured.single.proposition.7}) and (\ref{chapter3:appendix:structured.single.proposition.8}) readily yield the result of the proposition.
\end{proof}

%%%%%%%%%%%%%%%%%%%%%%%%%%%%%%%%%%%%%%%%%%%%%%%%%%%%%%%%%%%%%%%%%%%%%%%%%%%%%%%%%%%%%%%%%%%%%%%%%%%%%%%%%%%%%%%%%%%
%%%%%%%%%%%%%%%%%%%%%%%%%%%%%%%%%%%%%%%%%%%%%%%%%%%%%%%%%%%%%%%%%%%%%%%%%%%%%%%%%%%%%%%%%%%%%%%%%%%%%%%%%%%%%%%%%%%
%%%%%%%%%%%%%%%%%%%%%%%%%%%%%%%%%%%%%%%%%%%%%%%%%%%%%%%%%%%%%%%%%%%%%%%%%%%%%%%%%%%%%%%%%%%%%%%%%%%%%%%%%%%%%%%%%%%
\subsection{Two-Layer Perceptron}\label{chapter3:appendix:structured.two.layers}
In this section we consider structured pruning of a block of two layers and adopt the same notation introduced in Section \ref{chapter3:appendix:unstructured.two.layers} for $\Psi, \Pi^\kappa, \b{W}_1, \b{W}_2, \varphi_1, \varphi_2$.  We now state this section's result for pruning.
\begin{proposition}\label{chapter3:appendix:structured.two.layers.proposition}
 Suppose the activations $\varphi_1, \varphi_2$ satisfy (\ref{chapter3:assumption:1.1}) and (\ref{chapter3:assumption:1.2}) , and $\varphi_1$ also satisfies $(\ref{chapter3:assumption:1.3})$  in Assumption \ref{chapter3:assumption:1}. Let $k | \gcd(n_1, n_2, n_3)$, $p,\gamma, \varepsilon \in (0, 1)$, $\alpha \in \l(\frac{k}{n_2}, 1-\frac{k}{n_2}\r]$, $\sigma_i=\|\b{W}_i\|$ for $i=1,2$, and $\kappa\in [\sigma_1\sigma_2 \rho, \infty)$. Finally, let  $\mathcal{E}_\ell = \{k(\ell-1) + q \mid q\in [k]\}$ for $\ell \in \mathbb{Z}_{\geq 1}$, and $\nu_i=\max_{(a,b)} \|[\b{W}_i]_{\EE_a,\EE_b}\|$ for $i=1,2$. There exists a constant $n_0 = n_0(\gamma, p, \alpha, \varepsilon)$ such that if $n_2/k \geq n_0$, then  there exists a diagonal mask matrix $\b{D} \in \l\{0, 1, \frac{1}{p}\r\}^{n_2 \times n_2}$ satisfying
    \begin{align}
        \l|\l\{\EE_\ell \mid \ell\in[n_2/k],  \b{D}_{\EE_\ell, \EE_\ell} \neq \b{0}_{k\times k} \r\}\r| \leq (1+\gamma)p \l\lfloor \frac{\alpha n_2}{k} \r\rfloor + \frac{n_2}{k} - \l\lfloor \frac{\alpha n_2}{k} \r\rfloor, \label{chapter3:appendix:structured.two.layers.proposition.1}
    \end{align}
     such that 
    \begin{align}
        \Pi^{\kappa}(\b{D}\b{W}_1) &\leq (1+\varepsilon) \bigg(\l(1 + f\l(\Lambda\r) \r)(\sigma_1\sigma_2)^2\E\l[\|\b{x} - \b{x}'\|^2\r] + f\l(\Lambda \r)\omega \bigg). \label{chapter3:appendix:structured.two.layers.proposition.2}
    \end{align}
    where $\Lambda =  \frac{2\alpha \theta (1-p)}{p(1-\alpha)} \frac{\sigma_1^2 \nu_2 \sqrt{n_3}}{\sqrt{k}}$, $\omega=\frac{2\theta\nu_2 \sqrt{n_3} + 2\tau \theta^2 \sigma_2 \nu_1 \sqrt{n_1}}{\sqrt{k}} + \theta^2$, $\tau = 1\vee \frac{1-p}{p}$, and $f(x)=xe^x$.
\end{proposition}
\begin{proof}
Let $(h_{\ell})_{\ell \in [n_2/k]}$ be a sequence of independent random variables with distribution ${\rm Bernoulli}(p)/p$. Similarly to Proposition \ref{chapter3:appendix:structured.single.proposition}, we use an interpolation argument whereby we iteratively pick a \textit{suitable} set $\mathcal{E}_\ell$ indexed by $\ell \in  [n_2/k]$ and switch the $\ell$-th block-row of $\b{W}_1$ from $[\b{W}_1]_{\EE_\ell, :}$ to  $h_{\ell} [\b{W}_1]_{\EE_\ell, :} $. We control the resulting error from each switch, and  make $\lfloor \alpha n_2/k \rfloor$ switches. Let $\b{U}^d$ be the first layer weight matrix at the end of the $d$-th  interpolation step where $d \in [0, \lfloor \alpha n_2/k \rfloor - 1]$, and let $\mathcal{S}^d_1 \cup \mathcal{S}^d_2 =  [n_2/k]$ track the interpolation process similarly to the proof of Proposition \ref{chapter3:appendix:structured.single.proposition}. Suppose we are at step $d< \lfloor \alpha n_2/k \rfloor$ so that  $|\mathcal{S}^d_2| < \lfloor \alpha n_2/k \rfloor$ (otherwise the interpolation is over). Let $\ell \in \mathcal{S}^d_1$ (note that $\mathcal{S}^d_1 \neq \emptyset$ as $|\mathcal{S}^d_1|\geq  n_2/k  - \lfloor \alpha  n_2/k \rfloor + 1 \geq 1 $). Using Taylor's expansion on $t\mapsto \Psi(\b{G}(t; \mathcal{E}_\ell) \b{U}^d )$, we obtain
\begin{align}
    \Psi( \b{G}(h_{\ell}; \EE_\ell) \b{U}^d) - \Psi(\b{U}^d) &=\Psi( \b{G}(h_{\ell}; \EE_\ell) \b{U}^d) - \Psi(\b{G}(1; \EE_\ell) \b{U}^d) \nonumber \\
    &= (h_{\ell}-1)\frac{\partial \Psi(\b{G}(1; \EE_\ell) \b{U}^d)}{\partial t} 
    + \frac{(h_{\ell}  - 1)^2}{2}\frac{\partial^2 \Psi(\b{G}(t_{\ell}; \EE_\ell) \b{U}^d)}{\partial t^2}. \nonumber
\end{align}
where $t_{\ell} \in (1\wedge h_{\ell}, 1\vee h_{\ell})$ is a random variable. Note in particular that $|1-t_\ell| \leq \tau$ almost surely. Taking the expectation of the above over  $\{h_{\ell}\}\cup \T^d$  we obtain
\begin{align}
    &\l|\E_{\{h_{\ell}\}\cup \T^d}\l[ \Psi(\b{U}^{d}\b{G}(h_{\ell};\EE_\ell)) \r] - \E_{\T^d}\l[\Psi(\b{U}^{d}) \r]  \r| \nonumber \\
    &\leq \frac{1-p}{2}\E_{\{t_{\ell}\} \cup \T^d}\l[ \l| \frac{\partial^2 \Psi(\b{G}(t_{\ell};\EE_\ell)\b{U}^{d})}{\partial t^2} \r| \Biggm| h_{\ell}=0\r] \nonumber \\
    &+ \frac{(1-p)^2}{2p}\E_{\{t_{\ell}\} \cup \T^d}\l[ \l| \frac{\partial^2 \Psi(\b{G}(t_{\ell};\EE_\ell)\b{U}^{d})}{\partial t^2} \r| \Biggm| h_{\ell}=\frac{1}{p}\r]. \label{chapter3:appendix:structured.two.layers.proposition.3}
\end{align}
Let $\eta \in \R_{\geq 1}$ and introduce
\begin{alignat}{2}
    \mathcal{C}_0 &\triangleq \Bigg\{ \ell \in \mathcal{S}^d_1 \Biggm|  \E_{\{t_{\ell}\} \cup \T^d}&&\l[ \l|\frac{\partial^2 \Psi(\b{G}(t_\ell; \EE_\ell)\b{U}^{d})}{\partial t^2} \r| \Biggm| h_\ell=0  \r] \leq  \frac{2\eta \theta^2 }{|\Sm_1^d|}\frac{\sigma_1^2 \nu_2 \sqrt{n_3} }{\sqrt{k}}  \nonumber\\
    & &&\times \l(\frac{\nu_2 \sqrt{n_3} + \tau \theta \sigma_2  \nu_1 \sqrt{n_1}}{\sqrt{k}} + \sqrt{\E\l[\Psi(\b{U}^d)\r]} \r) \Bigg\},\nonumber\\
    \mathcal{C}_1 &\triangleq \Bigg\{ \ell \in \mathcal{S}^d_1 \Biggm|  \E_{\{t_{\ell}\} \cup \T^d}&&\l[ \l|\frac{\partial^2 \Psi(\b{G}(t_\ell; \EE_\ell)\b{U}^{d})}{\partial t^2} \r| \Biggm| h_\ell=\frac{1}{p}  \r] \leq  \frac{2\eta \theta^2 }{|\Sm_1^d|}\frac{\sigma_1^2 \nu_2 \sqrt{n_3} }{\sqrt{k}}   \nonumber\\
    & &&\times \l(\frac{\nu_2 \sqrt{n_3} + \tau \theta \sigma_2  \nu_1 \sqrt{n_1}}{\sqrt{k}} + \sqrt{\E\l[\Psi(\b{U}^d)\r]} \r) \Bigg\},\nonumber
\end{alignat}
Using (\ref{chapter3:lemma:8.2}) of Lemma \ref{chapter3:lemma:8}, it follows that $|\mathcal{C}_0|, |\mathcal{C}_1| \geq \l(1- \frac{1}{\eta}\r)|\mathcal{S}^d_1|$.  Since $\alpha\leq 1-\frac{1}{n_2/k}$, we have $|\Sm^d_1|\geq n_2/k - \lfloor \alpha n_2/k\rfloor + 1 \geq (1-\alpha)n_2/k  + 1 \geq 2 $. Setting $\eta=4$, it follows that $|\mathcal{C}_0\cap \mathcal{C}_1| \geq |\mathcal{C}_0| + |\mathcal{C}_1| - |\Sm_1^d| \geq (1-2/\eta)|\Sm_1^d|\geq (1-2/\eta)2=1$, and thus $\mathcal{C}_0\cap \mathcal{C}_1  \neq \emptyset$. Henceforth, if we pick $\ell\in \Sm_1^d$ with the smallest \textit{score} given by $\l|\E_{\{h_{\ell}\}\cup \T^d}\l[ \Psi(\b{G}(h_{\ell};\EE_\ell) \b{U}^{d}) \r] - \E_{\T^d}\l[\Psi(\b{U}^{d}) \r]  \r|$, it follows from (\ref{chapter3:appendix:structured.two.layers.proposition.3}) that
\begin{align}
    \l| \E_{\{h_{\ell}\}\cup \T^d}\l[ \Psi(\b{G}(h_{\ell};\EE_\ell)\b{U}^{d}) \r]  - \E_{\T^d}\l[\Psi(\b{U}^{d})\r] \r| &\leq \l(\frac{1-p}{2} + \frac{(1-p)^2}{2p}\r)  \frac{8 \theta^2 }{|\Sm_1^d|}\frac{\sigma_1^2\nu_2 \sqrt{n_3} }{\sqrt{k}}\nonumber \\
    &\times \l(\frac{\nu_2 \sqrt{n_3} + \tau \theta \sigma_2 \nu_1  \sqrt{n_1}}{\sqrt{k}} + \sqrt{\E\l[\Psi(\b{U}^d)\r]} \r)\nonumber \\
    &\leq \frac{4\theta^2(1-p)  }{p(1-\alpha)}\frac{\sigma_1^2 \nu_2 \sqrt{k n_3} }{n_2} \nonumber\\
    &\times \l(\frac{\nu_2 \sqrt{n_3} + \tau \theta \sigma_2 \nu_1  \sqrt{n_1}}{\sqrt{k}} + \frac{\E\l[\Psi(\b{U}^d)\r]}{2\theta} + \frac{\theta}{2} \r), \label{chapter3:appendix:structured.two.layers.proposition.4}
\end{align}
where we used $|\Sm_1^d| \geq (1-\alpha) n_2 /k$ and the arithmetic-geometric inequality in the last line. We then update $\Sm^{d+1}_1 = \Sm^{d}_1 \setminus \{\ell\}$ and $\Sm^{d+1}_2 = \Sm^{d}_2 \cup \{\ell\}$. We repeat this switching operation as long as $d < \lfloor\alpha n_2/k \rfloor$, and let $\mathcal{K}= \Sm^{\lfloor \alpha n_2/k \rfloor}_2$. That is, $\mathcal{K}$ is the set of all swapped  indices $\ell$ at the end of the interpolation process, so that $|\mathcal{K}|= \lfloor \alpha n_2/k \rfloor$. Let $\beta_d =\E_{\T^d}[\Psi(\b{U}^d)]$ for $d\in [0, |\mathcal{K}|]$ and note that if we switch the index $\ell$  at the $d$-th step, then $\beta_d = \E_{\T^d}\l[\Psi(\b{G}(1;\EE_\ell)\b{U}^{d})\r]=\E_{\T^d}[\Psi(\b{U}^d)]$ and 
$\beta_{d+1} = \E_{\T^{d+1}}\l[\Psi(\b{G}(t_{\ell};\EE_\ell)\b{U}^{d})\r]$. Therefore, we can rewrite (\ref{chapter3:appendix:structured.two.layers.proposition.4}) as 
\begin{align}
    |\beta_{d+1} - \beta_{d}| &\leq  \lambda (\beta_d + \omega) ,\nonumber
\end{align}
where $\lambda = \frac{2\theta (1-p)  }{p(1-\alpha)}\frac{\sigma_1^2 \nu_2 \sqrt{k n_3} }{n_2}$ and $\omega=\frac{2\theta\nu_2 \sqrt{n_3} + 2\tau \theta^2 \sigma_2 \nu_1 \sqrt{n_1}}{\sqrt{k}} + \theta^2$. Using Lemma \ref{chapter3:lemma:3}, we obtain
\begin{align*}
    |\beta_{d+1} - \beta_d| &\leq \exp\l(\lambda \l\lfloor \frac{\alpha n_2}{k} \r\rfloor\r)\lambda\l(\beta_0 + \omega\r),
\end{align*}
summing the above over $d\in \l\lfloor \frac{\alpha n_2}{k} \r\rfloor$, we have
\begin{align*}
    \l|\beta_{\l\lfloor \frac{\alpha n_2}{k} \r\rfloor} - \beta_0\r| &\leq \exp\l(\lambda \l\lfloor \frac{\alpha n_2}{k} \r\rfloor\r)\lambda \l\lfloor \frac{\alpha n_2}{k} \r\rfloor\l(\beta_0 + \omega\r)\\
    &\leq f\l(\frac{\lambda \alpha n_2}{k}\r) \l(\beta_0 + \omega\r)\\
    &= f(\Lambda)(\beta_0 + \omega),
\end{align*}
where $\Lambda = \frac{\lambda \alpha n_2}{k} = \frac{2\theta \alpha(1-p)}{p(1-\alpha)} \frac{\sigma_1^2 \nu_2 \sqrt{n_3}}{\sqrt{k}}$. It follows that
\begin{align}
    \E_{\T}[\Psi(\b{U})] &=   \beta_{\l\lfloor \frac{\alpha n_2}{k} \r\rfloor} \leq \l(1 + f\l(\Lambda \r)\r) \Psi(\b{W}_1) + f\l(\Lambda\r)\omega. \nonumber
\end{align}
Since  $\|\b{W}_1\|=\sigma_1, \|\b{W}_2\|=\sigma_2$ and $\|\varphi_1\|_{\rm Lip}\leq 1$ by (\ref{chapter3:assumption:1.2}), we have that $\Psi(\b{W}_1) \leq (\sigma_1 \sigma_2)^2 \E[\|\b{x}-\b{x}'\|^2]$. Therefore 
\begin{align}
     \E_{\T}[\Psi(\b{U})] &\leq \l(1 + f\l(\Lambda\r) \r)(\sigma_1\sigma_2)^2\E\l[\|\b{x} - \b{x}'\|^2\r] + f\l(\Lambda\r)\omega. \label{chapter3:appendix:structured.two.layers.proposition.5}
\end{align}
We next show
\begin{align}
    \E_{\T}\l[\Pi^{\kappa}(\b{U})\r] &\leq \l(1 + f\l(\Lambda\r) \r)(\sigma_1\sigma_2)^2\E\l[\|\b{x} - \b{x}'\|^2\r] + f\l(\Lambda\r)\omega. \label{chapter3:appendix:structured.two.layers.proposition.6}
\end{align}
By (\ref{chapter3:appendix:structured.two.layers.proposition.5}), it suffices to prove that
\begin{align}
    \E_{\T}\l[\Pi^{\kappa}(\b{U})\r] &\leq \E_{\T}\l[\Psi(\b{U})\r], \label{chapter3:appendix:structured.two.layers.proposition.7}
\end{align}
which we show next. Fix a pair of vectors $(\b{x}, \b{x}') \in \mathbb{B}^{n_1}_{2}(\rho) \times \mathbb{B}^{n_1}_{2}(\theta)$, and denote by $\b{z}, \b{z}'$ the output vectors $\varphi_2(\b{W}_2 \varphi_1(\b{W}_1 \b{x})),  \varphi_2(\b{W}_2 \varphi_1(\b{U} \b{x}'))$. As $\|\b{x}\|\leq \rho$   it follows from (\ref{chapter3:assumption:1.1}) and (\ref{chapter3:assumption:1.2}) that $  \|\b{z}\| = \|\varphi_2(\b{W}_2\varphi_1(\b{W}_1 \b{x}))\| \leq \sigma_1 \sigma_2 \rho $. Since $\|\b{z}\|\leq \sigma_1 \sigma_2 \rho $ and $\sigma_1 \sigma_2\rho \leq \kappa$ from the assumptions of Proposition \ref{chapter3:appendix:structured.two.layers.proposition}, we have using Lemma \ref{chapter3:lemma:1} that 
\begin{align}
    \|\b{z} - [\b{z}']_{\kappa}\|  &\leq \|\b{z} - \b{z}'\|\nonumber \nonumber \\
    &= \|\varphi_2(\b{W}_2\varphi_1(\b{W}_1 \b{x})) - \varphi_2(\b{W}_2\varphi_1(\b{U}\b{x}'))\| \nonumber \\
    &\leq \|\b{W}_2 \varphi_1(\b{W}_1\b{x}) - \b{W}_2 \varphi_1(\b{U} \b{x}')\|, \nonumber
\end{align}
where we used $\|\varphi_2\|_{\rm Lip}\leq 1$  in the last line. Therefore
\begin{align}
    \E_{\T}\l[\Pi^{\kappa}(\b{U})\r] &= \E_{\T}\E_{(\b{x}, \b{x}')}\l[\|\b{z} - [\b{z}']_{\kappa}\|^2\r]\nonumber\\
    &\leq  \E_{\T}\E_{(\b{x}, \b{x}')}\l[\|\b{W}_2 \varphi_1(\b{W}_1\b{x}) - \b{W}_2 \varphi_1(\b{U} \b{x}')\|^2\r]\nonumber\\
    &= \E_{\T}\l[\Psi(\b{U})\r]. \nonumber
\end{align}
This concludes the proof of (\ref{chapter3:appendix:structured.two.layers.proposition.6}). Note that by construction
\begin{align}
    \b{U} = \b{D} \b{W}_1, \nonumber
\end{align}
we then conclude similarly to the proof of Proposition \ref{chapter3:appendix:structured.single.proposition} that there exists a constant $n_0=n_0(\gamma , p, \alpha, \varepsilon)$, such that if $n_2/k \geq n_0$, then there exists a nonrandom realization $(\hat{h}_{\ell})_{\ell \in \mathcal{K}}$ of $(h_{\ell})_{\ell \in \mathcal{K}}$, and $\hat{\b{U}}, \hat{\b{D}}$ of $\b{U}, \b{D}$ such that
    \begin{align}
        \Pi^{\kappa}(\hat{\b{U}}) &\leq (1+\varepsilon) \bigg( \l(1 + f\l(\Lambda\r) \r)(\sigma_1\sigma_2)^2\E\l[\|\b{x} - \b{x}'\|^2\r] + f\l(\Lambda\r)\omega \bigg), \label{chapter3:appendix:structured.two.layers.proposition.8}
    \end{align}
    and 
    \begin{align}
        \frac{\|\hat{\b{D}}\|_0}{k} \leq (1+\gamma)p\l\lfloor \frac{\alpha n_2}{k} \r\rfloor + \frac{n_2}{k} - \l\lfloor \frac{\alpha n_2}{k} \r\rfloor. \label{chapter3:appendix:structured.two.layers.proposition.9}
    \end{align}
    Finally, note that by construction each null $\hat{h}_\ell$ leads to a null submatrix $\hat{\b{D}}_{\EE_\ell,\EE_\ell}$. Hence (\ref{chapter3:appendix:structured.two.layers.proposition.8}) and (\ref{chapter3:appendix:structured.two.layers.proposition.9}) yield the result of the proposition.
\end{proof}

\subsection{Multilayer Perceptron}\label{chapter3:appendix:structured.deep.fcnn}

\begin{proposition}\label{chapter3:appendix:structured.deep}
    Let $\mathcal{R}$ be a distribution over $\mathbb{B}^{n_1}_{2}(1)$, $\b{x} \sim \mathcal{R}$, $\xi, p\in(0,1)$ and $\alpha \approx 0.99$. Suppose $\Phi$ is an $m$-layer MLP with layers $\b{W}_\ell \in \R^{n_{\ell+1} \times n_{\ell}}, \ell \in[m]$. Furthermore, let $k|\gcd(\{n_\ell \mid \ell\in[m]\})$,  and $\sigma\geq \max_{\ell \in [m]} \|\b{W}_\ell\|, \nu_{\ell} =   \max_{(i, j) \in [n_{\ell+1}/k] \times [n_{\ell}/k]} \| [\b{W}_{\ell}]_{\EE_i, \EE_j}\|, \forall \ell \in [m]$, where $\mathcal{E}_{\ell} = \{k(\ell-1)+q \mid q\in[k]\}$. Then there exists $\delta = \delta(\xi)$ and $n_0 = n_0(\xi, p)$ such that if $\min_{\ell\in \mathcal{W}\cup (\mathcal{B}+1)} n_{\ell} / k \geq n_0$ and
    \begin{alignat}{3}
        &\forall \ell\in \mathcal{W}, \quad && \frac{\nu_{\ell}^2 n_{\ell+1}}{k} &&\leq \frac{ p(1-\alpha)\delta}{\alpha(1-p)} \sigma^2, \label{chapter3:appendix:structured.deep.w}\\
        &\forall \ell\in \mathcal{B}, \quad && \frac{\nu_{\ell+1}^2 n_{\ell+2}}{k} &&\leq \frac{p^2(1-p)^2\delta^2}{\alpha^2 (1-p)^2}\frac{(1 \wedge \sigma^{8})}{\sigma^{2(\ell+1)}}  \wedge  \frac{ p(1-\alpha)\delta}{\alpha(1-p)} \sigma^2, \label{chapter3:appendix:structured.deep.b.1}\\
        &\forall \ell\in \mathcal{B}, \quad  &&\frac{\nu_\ell \nu_{\ell+1} \sqrt{n_\ell n_{\ell+2}}}{k} &&\leq \frac{p(1-\alpha)\delta}{\alpha(1-p)} \frac{1}{\sigma^{\ell-2}\l(1\vee \frac{1-p}{p}\r)}  \label{chapter3:appendix:structured.deep.b.2},
    \end{alignat}
    then there exists a network $\hat{\Phi}$ given  by $\hat{\Phi}(\b{x}) = [\varphi_m(\hat{\b{W}}_m[\varphi_{m-1}(\dots \hat{\b{W}}_1\b{x})]_{\kappa_{m-1}})]_{\kappa_m} $ with $\kappa_\ell=\sigma^\ell$, such that
    \begin{enumerate}
        \item \label{chapter3:appendix:structured.deep.1} $\E_{\b{x}}\l[ \|\Phi(\b{x}) - \hat{\Phi}(\b{x})\|^{2} \r] \leq \sigma^{2m} (1+\xi)^{m}\xi$.
        \item \label{chapter3:appendix:structured.deep.2} $\forall \ell \in \mathcal{W}$, the matrix $\hat{\b{W}}_\ell$ has at most $1.01p$ fraction of its block-columns $\hat{\b{W}}_{:, \EE_j}$ not set to zero.
        \item \label{chapter3:appendix:structured.deep.3} $\forall \ell \in \mathcal{B}$, the matrix  $\hat{\b{W}}_\ell$ has at most $1.01p$ fraction of its block-rows $\hat{\b{W}}_{\EE_i,:}$ not set to zero.
    \end{enumerate}
    Furthermore, if $\mathcal{D}$ is a data distribution over $(\b{x}, \b{y}) \in \mathbb{B}^{n_1}_2(1) \times \R^{n_{m+1}}$, and $\varepsilon = (1+\xi)^m \xi$ then
    \begin{align}
        \mathcal{L}(\hat{\Phi}; \mathcal{D}) &\leq \mathcal{L}(\Phi; \mathcal{D}) + 2\sigma^{m} \sqrt{\varepsilon \mathcal{L}(\Phi; \mathcal{D})} + \sigma^{2m} \varepsilon.  \label{chapter3:appendix:structured.deep.4}
    \end{align}
\end{proposition}

\begin{proof}[Proof of Proposition \ref{chapter3:appendix:structured.deep}]
    We will construct structurally sparse matrices $\hat{\b{W}}_\ell$ recursively and verify conditions (\ref{chapter3:lemma:9.1}), (\ref{chapter3:lemma:9.2}) and (\ref{chapter3:lemma:9.3}) from Lemma \ref{chapter3:lemma:9}, which would then readily yield the result of the proposition for small enough $\delta$. Following the notation of Lemma \ref{chapter3:lemma:9}, let $\b{z}^\ell, \hat{\b{z}}^\ell$ be given recursively by $\b{z}^0 = \hat{\b{z}}^0 = \b{x}$, and for $\ell \in [m]$
    \begin{alignat}{2}
        &\b{z}^{\ell} &&= \varphi_{\ell}(\b{W}_\ell \b{z}^{\ell-1}),\nonumber\\
        &\hat{\b{z}}^{\ell} &&= \varphi_{\ell}(\hat{\b{W}}_\ell [\hat{\b{z}}^{\ell-1}]_{\kappa_{\ell-1}}),\nonumber
    \end{alignat}
    where $\kappa_\ell = \sigma^{\ell}$. Let $\alpha\approx 0.99,\gamma \approx 0.01$, and $\varepsilon=\xi/2$. Denote by $n_0^1$ the constant $n_0=n_0(\gamma, p,\alpha, \varepsilon) = n_0(p, \xi)$ in Proposition \ref{chapter3:appendix:structured.single.proposition}. Similarly, denote by $n_0^2$ the constant $n_0=n_0(\gamma, p,\alpha, \varepsilon)=n_0(p, \xi)$ in Proposition \ref{chapter3:appendix:structured.two.layers.proposition}. Set $n_0 = n_0^1 \vee n_0^2 \vee  \frac{1}{1-\alpha}$. we consider three cases.
    \begin{enumerate}
        \item Case 1. $\ell \not \in \mathcal{W} \cup \mathcal{B}$. In this case we set $\hat{\b{W}}_\ell = \b{W}_\ell$, which satisfies (\ref{chapter3:lemma:9.1}) in Lemma \ref{chapter3:lemma:9}.
        \item Case 2. $\ell\in \mathcal{W}$. Note that $(\b{z}^{\ell-1}, [\hat{\b{z}}^{\ell-1}]_{\kappa_{\ell-1}}) \in \mathbb{B}^{n_{\ell}}_2(\kappa_{\ell-1}) \times \mathbb{B}^{n_{\ell}}_{2}(\kappa_{\ell-1})$. Applying Proposition \ref{chapter3:appendix:structured.single.proposition} with $(\b{x}, \b{x}') = (\b{z}^{\ell-1}, [\hat{\b{z}}^{\ell-1}]_{\kappa_{\ell-1}})$,  $(\rho, \theta) = (\kappa_{\ell-1}, \kappa_{\ell-1}),\kappa=\kappa_\ell=\sigma \rho$ and $\alpha\approx0.99, \gamma\approx 0.01$ (note that by construction of $n_0$, the inequality $\alpha \leq 1- k/n_{\ell} $ holds in the statement of  Proposition \ref{chapter3:appendix:structured.single.proposition}), it follows that there exists a diagonal mask matrix $\b{D}\in \l\{0, 1, \frac{1}{p}\r\}^{n_{\ell} \times n_{\ell}}$, such that 
        \begin{align}
            \frac{\l|\l\{\EE_\ell \mid \ell\in[n_\ell/k],  \b{D}_{\EE_\ell, \EE_\ell} \neq \b{0}_{k\times k} \r\}\r|}{n_\ell/k} &\leq 0.01 + 1.01p, \label{chapter3:appendix:structured.deep.case2.1}
        \end{align}
        and
        \begin{align}
        \Pi^{\kappa_{\ell}}(\b{W}_{\ell} \b{D}) &\leq \l(1+\frac{\xi}{2}\r)\sigma^2  \E\l[\|\b{z}^{\ell-1} - [\hat{\b{z}}^{\ell-1}]_{\kappa_{\ell-1}}\|^2\r]  \nonumber\\
        &+  \l(1+\frac{\xi}{2}\r)\frac{2\alpha \kappa_{\ell-1}^2(1-p)}{p(1-\alpha)} \frac{\nu_\ell^2 n_{\ell+1}}{k} \label{chapter3:appendix:structured.deep.case2.2}. 
            \end{align}
        Set $\hat{\b{W}}_{\ell} = \b{W}_{\ell} \b{D}$ in $\hat{\Phi}$. We then have by (\ref{chapter3:appendix:structured.deep.case2.1}) that at most $0.01 + 1.01p$ fraction of the block-columns $[\hat{\b{W}}_\ell]_{:, \mathcal{E}_i}, i\in [n_\ell/k]$ in $\hat{\b{W}}_\ell$ are nonzero. Moreover, 
        \begin{align*}
            \l(1+\frac{\xi}{2}\r)\frac{2\alpha  \kappa_{\ell-1}^2 (1-p) }{p(1-\alpha)} \frac{\nu_\ell^2 n_{\ell+1}}{k} &\leq 4 \sigma^{2(\ell-1)}   \frac{\alpha(1-p)}{p(1-\alpha)} \frac{\nu_\ell^2 n_{\ell+1}}{k} \\
            &\leq 4 \sigma^{2\ell}  \delta,
        \end{align*}
        where we used (\ref{chapter3:appendix:structured.deep.w}) in the last line.    Let $\varepsilon_1 = \frac{\xi}{2}$ and $\varepsilon_2 = 4\delta$. Combining the above with (\ref{chapter3:appendix:structured.deep.case2.2}), it follows that
        \begin{align*}
            \E\l[ \| \b{z}^{\ell} - [\hat{\b{z}}^\ell]_{\kappa_\ell} \|^2\r] &= \Pi^{\kappa_\ell}(\hat{\b{W}}_\ell)\\
            &\leq (1+\varepsilon_1)\sigma^2 \E\l[\|\b{z}^{\ell-1} - [\hat{\b{z}}^{\ell-1}]_{\kappa_{\ell-1}}\|^2\r] + \sigma^{2\ell} \varepsilon_2,
        \end{align*}
        which shows (\ref{chapter3:lemma:9.2}).        
        \item Case 3. $\ell \in \mathcal{B}$. Note that $(\b{z}^{\ell-1}, [\hat{\b{z}}^{\ell-1}]_{\kappa_{\ell-1}}) \in \mathbb{B}^{n_{\ell}}_2(\kappa_{\ell-1}) \times \mathbb{B}^{n_{\ell}}_{2}(\kappa_{\ell-1})$. Applying Proposition \ref{chapter3:appendix:structured.two.layers} with $(\b{x}, \b{x}') = (\b{z}^{\ell-1}, [\hat{\b{z}}^{\ell-1}]_{\kappa_{\ell-1}})$,  $(\rho, \theta) = (\kappa_{\ell-1}, \kappa_{\ell-1}),\kappa=\kappa_{\ell+1}=\sigma^2 \rho$ and $\alpha\approx0.99, \gamma\approx 0.01$ (note that by construction of $n_0$, the inequality $\alpha \leq 1- k/n_{\ell+1}$ holds in the statement of Proposition \ref{chapter3:appendix:structured.two.layers}), it follows that there exists a diagonal mask  matrix $\b{D} \in \l\{0, 1, \frac{1}{p}\r\}^{n_{\ell+1} \times n_{\ell+1}}$ such that
        \begin{align}
            \frac{ \l|\l\{\EE_\ell \mid \ell\in[n_{\ell+1}/k],  \b{D}_{\EE_\ell, \EE_\ell} \neq \b{0}_{k\times k} \r\}\r|}{n_{\ell+1}/k} &\leq 0.01 + 1.01p,\label{chapter3:appendix:structured.deep.case3.1}
        \end{align}
        and
        \begin{align}
        \Pi^{\kappa_{\ell+1}}(\b{D} \b{W}_{\ell}) &\leq \l(1+ \frac{\xi}{2}\r) \l(1 + f(\Lambda) \r) \sigma^4\E\l[\|\b{z}^{\ell-1}-[\hat{\b{z}}^{\ell-1}]_{\kappa_{\ell-1}}\|^2\r] + \l(1 + \frac{\xi}{2}\r) f(\Lambda) \omega\label{chapter3:appendix:structured.deep.case3.2} ,
        \end{align}
        where $\Lambda =  \frac{2\alpha\kappa_{\ell-1}(1-p)  }{p(1-\alpha)}\frac{\sigma^2 \nu_{\ell+1} \sqrt{n_{\ell+2}} }{\sqrt{k}}$, $\omega = \frac{2\kappa_{\ell-1}\nu_{\ell+1} \sqrt{n_{\ell+2}} + 2\tau (\kappa_{\ell-1})^2\sigma \nu_{\ell} \sqrt{n_{\ell}}}{\sqrt{k}} + (\kappa_{\ell-1})^2$, and $\tau = 1\vee \frac{1-p}{p}$. We first bound the terms in (\ref{chapter3:appendix:structured.deep.case3.2}). We have
        \begin{align}
            \Lambda & = \frac{2\alpha\sigma^{\ell+1}(1-p)}{p(1-\alpha)} \frac{\nu_{\ell+1} \sqrt{n_{\ell+2}}}{\sqrt{k}} \leq 2\delta \label{chapter3:appendix:structured.deep.case3.3}, 
        \end{align}
        where we used (\ref{chapter3:appendix:structured.deep.b.1}) in the last inequality. Similarly, we have
        \begin{align}
            f(\Lambda)\omega  &= e^{\Lambda} \Lambda\l(\frac{2\sigma^{\ell-1} \nu_{\ell+1} \sqrt{n_{\ell+2}}}{\sqrt{k}} + \frac{2\sigma^{2\ell-1} \tau \nu_{\ell}\sqrt{n_{\ell}}}{\sqrt{k}} + \sigma^{2(\ell-1)}\r) \nonumber \\
            &= \frac{2\alpha (1-p)e^{\Lambda}}{p(1-\alpha)} \l( \frac{2\sigma^{2\ell} (\nu_{\ell+1})^2 n_{\ell+2}}{k} + \frac{2 \sigma^{3\ell} \tau \nu_{\ell} \nu_{\ell+1} \sqrt{n_{\ell} n_{\ell+2}} }{k}  + \frac{\sigma^{3\ell-1} \nu_{\ell+1} \sqrt{n_{\ell+2}}}{\sqrt{k}}\r) \nonumber \\
            &\leq 2e^{2\delta}\sigma^{2(\ell+1)} \l( 2\delta + 2 \delta + \delta\r) \label{chapter3:appendix:structured.deep.case3.4}\\
            &= 10\delta e^{2\delta} \sigma^{2(\ell+1)} \nonumber \\
            &= 5f(2\delta) \sigma^{2(\ell+1)}, \label{chapter3:appendix:structured.deep.case3.5}
        \end{align}
        where we used (\ref{chapter3:appendix:structured.deep.b.1}) and (\ref{chapter3:appendix:structured.deep.b.2}) in line (\ref{chapter3:appendix:structured.deep.case3.3}).   Set $\hat{\b{W}}_\ell = \b{D}\b{W}_\ell$, $\varepsilon_3 = \frac{\xi}{2}+\l(1+\frac{\xi}{2}\r)f(2\delta)$, and $\varepsilon_4 = \l( 1 + \frac{\xi}{2}\r)5f\l(2\delta\r)$. We then have that at most $0.01 + 1.01p$ fraction of the rows of $\hat{\b{W}}_\ell$ are nonzero. Moreover, we have from (\ref{chapter3:appendix:structured.deep.case3.3}) and  (\ref{chapter3:appendix:structured.deep.case3.4})
        \begin{align*}
             \E\l[ \| \b{z}^{\ell+1} - [\hat{\b{z}}^{\ell+1}]_{\kappa_{\ell+1}} \|^2\r] &= \Pi^{\kappa_{\ell+1}}(\hat{\b{W}}_{\ell})\\
             &\leq  (1+\varepsilon_3) \sigma^4 \E\l[ \| \b{z}^{\ell-1} - [\hat{\b{z}}^{\ell-1}]_{\kappa_{\ell-1}} \|^2\r] + \sigma^{2(\ell+1)} \varepsilon_4,
        \end{align*}
        which shows (\ref{chapter3:lemma:9.3}).
    \end{enumerate}
    Using Lemma \ref{chapter3:lemma:9}, it follows that item \ref{chapter3:appendix:structured.deep.1} in Proposition \ref{chapter3:appendix:structured.deep} holds for small enough $\delta =\delta(\xi)$. Moreover, items \ref{chapter3:appendix:structured.deep.2} and \ref{chapter3:appendix:structured.deep.3} readily hold from (\ref{chapter3:appendix:structured.deep.case2.1}) and (\ref{chapter3:appendix:structured.deep.case3.1}). It remains to show (\ref{chapter3:appendix:structured.deep.4}). We have
\begin{align}
    \mathcal{L}(\hat{\Phi}; \mathcal{D}) &= \E_{(\b{x}, \b{y})\sim \mathcal{D}} \l[\|\hat{\Phi}(\b{x}) - \b{y}\|^2 \r]\nonumber\\
    &= \E_{(\b{x}, \b{y}) \sim \mathcal{D}} \l[\|\hat{\Phi}(\b{x}) - \Phi(\b{x}) + \Phi(\b{x})  - \b{y}\|^2 \r]\nonumber\\
    &= \mathcal{L}(\Phi; \mathcal{D}) + 2 \E[\langle \hat{\Phi}(\b{x}) - \Phi(\b{x}),\Phi(\b{x})  - \b{y} \rangle ] + \E_{\b{x}} [\| \hat{\Phi}(\b{x}) - \Phi(\b{x})\|^2]\nonumber\\
    &\leq \mathcal{L}(\Phi; \mathcal{D}) + 2 \sqrt{\mathcal{L}(\Phi; \mathcal{D})} \sqrt{\E_{\b{x}} [\| \hat{\Phi}(\b{x}) - \Phi(\b{x})\|^2]} + \E_{\b{x}} [\| \hat{\Phi}(\b{x}) - \Phi(\b{x})\|^2]\nonumber\\
    &\leq \mathcal{L}(\Phi; \mathcal{D}) + 2 \sigma^m \sqrt{\varepsilon \mathcal{L}(\Phi; \mathcal{D})} + \sigma^{2m} \varepsilon,\nonumber
\end{align}
where we used (\ref{chapter3:appendix:structured.deep.1}) in the last line and $\varepsilon = (1+\xi)^m \xi$. This ends the proof of (\ref{chapter3:appendix:structured.deep.4}), and concludes the proof of Proposition \ref{chapter3:appendix:structured.deep}.
\end{proof}

\subsubsection{Proof of Proposition \ref{chapter3:proposition:3}}
\begin{proof}[Proof of Proposition \ref{chapter3:proposition:3}]
We apply the result of Proposition \ref{chapter3:appendix:structured.deep}. Let  $k=1$, and $\sigma=c_1 \geq \max_{\ell \in [m]}\|\b{W}_\ell\|$ in the setting of Proposition \ref{chapter3:appendix:structured.deep}. We have by (\ref{chapter3:assumption:1.5}) from Assumption \ref{chapter3:assumption:1} that $\nu_\ell \leq c_2 / \sqrt{n_\ell \vee n_{\ell+1}}$. Using (\ref{chapter3:proposition:3.w}), we   have for $\ell \in \mathcal{W}$
\begin{align*}
    \frac{\nu_\ell^2 n_{\ell+1}}{k} &\leq c_2^2 \frac{n_{\ell+1}}{n_{\ell}} \leq \frac{p(1-\alpha)\delta}{\alpha (1-p)}\sigma^2,
\end{align*}
which satisfies (\ref{chapter3:appendix:structured.deep.w}). Similarly, we have from (\ref{chapter3:proposition:3.b.1}) for $\ell\in \mathcal{B}$
\begin{align*}
    \frac{\nu_{\ell+1}^2 n_{\ell+2}}{k} &\leq c_2^2 \frac{n_{\ell+2}}{n_{\ell+1}} \leq \frac{p^2(1-p)^2\delta^2}{\alpha^2 (1-p)^2}\frac{(1 \wedge \sigma^{8})}{\sigma^{2(\ell+1)}}  \wedge  \frac{ p(1-\alpha)\delta}{\alpha(1-p)} \sigma^2,
\end{align*}
which verifies (\ref{chapter3:appendix:structured.deep.b.1}). Finally, we have  from (\ref{chapter3:proposition:3.b.2})
\begin{align*}
    \frac{\nu_\ell \nu_{\ell+1} \sqrt{n_\ell n_{\ell+2}}}{k} &\leq c_2^2 \frac{\sqrt{n_\ell n_{\ell+2}}}{n_{\ell+1}} \leq \frac{p(1-\alpha)\delta}{\alpha(1-p)} \frac{1}{\sigma^{\ell-2}\l(1\vee \frac{1-p}{p}\r)} ,
\end{align*}
which verifies (\ref{chapter3:appendix:structured.deep.b.2}) in Proposition \ref{chapter3:appendix:structured.deep}. Therefore, we conclude from Proposition \ref{chapter3:appendix:structured.deep} that there exists $\delta =\delta(\xi)$ and $n_0=n_0(\xi, p)$, such that if $\min_{\ell \in \mathcal{W} \cup (\mathcal{B}+1)} n_\ell \geq n_0$, then there exists a network $\hat{\Phi}$ satisfying item \ref{chapter3:proposition:3.1} in Proposition \ref{chapter3:proposition:3}, with layers $\hat{\b{W}}$ satisfying items \ref{chapter3:appendix:structured.deep.2} and \ref{chapter3:appendix:structured.deep.3} in Proposition \ref{chapter3:appendix:structured.deep}, which readily yields \ref{chapter3:proposition:3.2} and \ref{chapter3:proposition:3.3} in Proposition \ref{chapter3:proposition:3}. Finally, (\ref{chapter3:proposition:3.4}) follows from (\ref{chapter3:appendix:structured.deep.4}). This concludes the proof of Proposition \ref{chapter3:proposition:3}.

\end{proof}

\subsubsection{Proof of Proposition \ref{chapter3:proposition:4}}
\begin{proof}[Proof of Proposition \ref{chapter3:proposition:4}]
In the setting of Proposition \ref{chapter3:appendix:structured.deep}, let $k=r^2$, and $\sigma=c_1 \geq \max_{\ell \in [m]}\|\b{W}_\ell\|$. We first derive bounds for $\nu_\ell$. Let $\b{K}_{\ell} \in \R^{d_{\ell+1} \times d_{\ell} \times r \times r}$ be the four-dimensional tensor representing the $\ell$-th convolutional layer. Namely, $[\b{W}_\ell]_{\EE_o, \EE_i} = \mathcal{C}([\b{K}_\ell]_{o, i}), \forall (o, i) \in [d_{\ell+1}]\times [d_\ell]$. By construction $\|\b{K}_\ell\|_{\infty} = \|\b{W}_\ell\|_{\infty} \leq \frac{c_2}{q\sqrt{d_\ell \vee d_{\ell+1}}}$. Thus, we have using Lemma \ref{chapter3:lemma:6} for $(o, i) \in [d_{\ell+1}] \times [d_\ell]$
\begin{align}
    \|[\b{W}_\ell]_{\EE_o, \EE_i}\| &= \| \mathcal{C}([\b{K}_\ell]_{o, i}) \| \\
                                    &\leq \|[\b{K}_\ell]_{o, i}\|_{\infty} \|[\b{K}_\ell]_{o, i}\|_0\\
                                    &\leq q^2 \|[\b{K}_\ell]_{o, i}\|_{\infty}, \label{chapter3:proposition:4.6}
\end{align}
where we used the fact that $[\b{K}_\ell]_{o, i}$ was obtained by padding a $q\times q$ kernel matrix with zeros. Note that by construction, the set of entries in $[\b{W}_\ell]_{\EE_o, \EE_i}$ is identical to the set of entries in $[\b{K}_\ell]_{o, i}$. Thus 
\begin{align*}
    \|[\b{K}_\ell]_{o, i}\|_{\infty}&=\|[\b{W}_\ell]_{\EE_o, \EE_i}\|_{\infty}\\
    &\leq \|\b{W}_\ell\|_{\infty}\\
    &\leq \frac{c_2}{q\sqrt{d_{\ell} \vee d_{\ell+1}}},
\end{align*}
where we used (\ref{chapter3:proposition:4.0}) in the last line. Combining the above with (\ref{chapter3:proposition:4.6}), we obtain
\begin{align*}
    \forall (o, i) \in [d_{\ell+1}] \times [d_\ell], \quad \|[\b{W}_\ell]_{\EE_o, \EE_i}\| &\leq \frac{c_2 q}{\sqrt{d_{\ell} \vee d_{\ell+1}}},
\end{align*}
therefore
\begin{align}
    \forall \ell \in [m], \quad \nu_\ell \leq \frac{c_2 q}{\sqrt{d_\ell \vee d_{\ell+1}}}. \label{chapter3:proposition:4.7}
\end{align}
We now apply the result of Proposition \ref{chapter3:appendix:structured.deep}. We have by (\ref{chapter3:proposition:4.7}) for $\ell\in \mathcal{W}$
\begin{align*}
    \frac{\nu_{\ell}^2 n_{\ell+1}}{k} &\leq \l(\frac{c_2 q}{\sqrt{d_{\ell} \vee d_{\ell+1}}}\r)^2\frac{d_{\ell+1} r^2}{r^2}\\
    &= \frac{c_2^2 q^2 d_{\ell+1}}{d_{\ell}}\\
    &\leq \frac{p(1-\alpha)\delta}{\alpha (1-p)} \sigma^2,
\end{align*}
where we used (\ref{chapter3:proposition:4.w}) in the last line. The latter satisfies (\ref{chapter3:appendix:structured.deep.w}).  Similarly, we have from (\ref{chapter3:proposition:4.b.1}) for $\ell \in \mathcal{B}$
\begin{align*}
    \frac{\nu_{\ell+1}^2 n_{\ell+2}}{k} &\leq \frac{c_2^2 q^2 d_{\ell+2}}{d_{\ell+1}}\\
    &\leq \frac{p^2(1-p)^2\delta^2}{\alpha^2 (1-p)^2}\frac{(1 \wedge \sigma^{8})}{\sigma^{2(\ell+1)}}  \wedge  \frac{ p(1-\alpha)\delta}{\alpha(1-p)} \sigma^2,
\end{align*}
which verifies (\ref{chapter3:appendix:structured.deep.b.1}). Finally, we have from (\ref{chapter3:proposition:4.b.2})
\begin{align*}
    \frac{\nu_{\ell} \nu_{\ell+1} \sqrt{n_{\ell} n_{\ell+2}}}{k} &\leq c_2^2 q^2 \frac{\sqrt{d_{\ell} d_{\ell+2}}}{d_{\ell+1}} \\
    &\leq \frac{p(1-\alpha)\delta}{\alpha(1-p)} \frac{1}{\sigma^{\ell-2}\l(1\vee \frac{1-p}{p}\r)},
\end{align*}
which verifies (\ref{chapter3:appendix:structured.deep.b.2}). Therefore, we conclude from Proposition \ref{chapter3:appendix:structured.deep} that there exists $\delta =\delta(\xi)$ and $n_0=n_0(\xi, p)$, such that if $\min_{\ell \in \mathcal{W} \cup (\mathcal{B}+1)} d_\ell \geq n_0$, then there exists a network $\hat{\Phi}$ satisfying Item \ref{chapter3:proposition:4.1} in Proposition \ref{chapter3:proposition:4}, with layers $\hat{\b{W}}$ satisfying Items \ref{chapter3:appendix:structured.deep.2} and \ref{chapter3:appendix:structured.deep.3} from  Proposition \ref{chapter3:appendix:structured.deep}. Note that the input and output filters of the $\ell$-th layer are given by the block matrices $[\hat{\b{W}}_\ell]_{:, \mathcal{E}_i}$ and $[\hat{\b{W}}_\ell]_{\mathcal{E}_o, :}$ respectively. Hence,  items \ref{chapter3:appendix:structured.deep.2} and \ref{chapter3:appendix:structured.deep.3} from Proposition \ref{chapter3:appendix:structured.deep} readily yields parts \ref{chapter3:proposition:4.2} and  \ref{chapter3:proposition:4.3}  in Proposition \ref{chapter3:proposition:4}. Finally, (\ref{chapter3:proposition:4.4}) follows from (\ref{chapter3:appendix:structured.deep.4}). This concludes the proof of Proposition \ref{chapter3:proposition:4}.

\end{proof}

%\newpage

%\section{Conclusions}
%\input{Conclusions}

%\newpage
\bibliographystyle{plain}
\bibliography{References}

@inproceedings{glorot2010understanding,
  title={Understanding the difficulty of training deep feedforward neural networks},
  author={Glorot, Xavier and Bengio, Yoshua},
  booktitle={Proceedings of the thirteenth international conference on artificial intelligence and statistics},
  pages={249--256},
  year={2010},
  organization={JMLR Workshop and Conference Proceedings},
}

@article{niehaus2024weight,
  title={Weight Rescaling: Applying Initialization Strategies During Training},
  author={Niehaus, Lukas and Krumnack, Ulf and Heidemann, Gunther},
  journal={Swedish Artificial Intelligence Society},
  pages={83--92},
  year={2024},
}

@inproceedings{kumar2024no,
  title={No Free Prune: Information-Theoretic Barriers to Pruning at Initialization},
  author={Kumar, Tanishq and Luo, Kevin and Sellke, Mark},
  booktitle={Forty-first International Conference on Machine Learning},
  year={2024}
}

@article{frankle2020pruning,
  title={Pruning neural networks at initialization: Why are we missing the mark?},
  author={Frankle, Jonathan and Dziugaite, Gintare Karolina and Roy, Daniel M and Carbin, Michael},
  journal={arXiv preprint arXiv:2009.08576},
  year={2020}
}

@article{latala2005some,
  title={Some estimates of norms of random matrices},
  author={Latala, Rafal},
  journal={Proceedings of the American Mathematical Society},
  volume={133},
  number={5},
  pages={1273--1282},
  year={2005}
}

@article{lecun1989optimal,
  title={Optimal brain damage},
  author={LeCun, Yann and Denker, John and Solla, Sara},
  journal={Advances in neural information processing systems},
  volume={2},
  year={1989}
}

@article{dong2017learning,
  title={Learning to prune deep neural networks via layer-wise optimal brain surgeon},
  author={Dong, Xin and Chen, Shangyu and Pan, Sinno},
  journal={Advances in neural information processing systems},
  volume={30},
  year={2017}
}

@inproceedings{qian2021probabilistic,
  title={A probabilistic approach to neural network pruning},
  author={Qian, Xin and Klabjan, Diego},
  booktitle={International Conference on Machine Learning},
  pages={8640--8649},
  year={2021},
  organization={PMLR}
}

@article{hassibi1992second,
  title={Second order derivatives for network pruning: Optimal brain surgeon},
  author={Hassibi, Babak and Stork, David},
  journal={Advances in neural information processing systems},
  volume={5},
  year={1992}
}

@article{frantar2022optimal,
  title={Optimal brain compression: A framework for accurate post-training quantization and pruning},
  author={Frantar, Elias and Alistarh, Dan},
  journal={Advances in Neural Information Processing Systems},
  volume={35},
  pages={4475--4488},
  year={2022}
}

@article{kurtic2022optimal,
  title={The optimal bert surgeon: Scalable and accurate second-order pruning for large language models},
  author={Kurtic, Eldar and Campos, Daniel and Nguyen, Tuan and Frantar, Elias and Kurtz, Mark and Fineran, Benjamin and Goin, Michael and Alistarh, Dan},
  journal={arXiv preprint arXiv:2203.07259},
  year={2022}
}

@inproceedings{benbaki2023fast,
  title={Fast as chita: Neural network pruning with combinatorial optimization},
  author={Benbaki, Riade and Chen, Wenyu and Meng, Xiang and Hazimeh, Hussein and Ponomareva, Natalia and Zhao, Zhe and Mazumder, Rahul},
  booktitle={International Conference on Machine Learning},
  pages={2031--2049},
  year={2023},
  organization={PMLR}
}

@inproceedings{yu2022combinatorial,
  title={The combinatorial brain surgeon: Pruning weights that cancel one another in neural networks},
  author={Yu, Xin and Serra, Thiago and Ramalingam, Srikumar and Zhe, Shandian},
  booktitle={International Conference on Machine Learning},
  pages={25668--25683},
  year={2022},
  organization={PMLR}
}

@article{han2015deep,
  title={Deep compression: Compressing deep neural networks with pruning, trained quantization and huffman coding},
  author={Han, Song and Mao, Huizi and Dally, William J},
  journal={arXiv preprint arXiv:1510.00149},
  year={2015}
}

@inproceedings{frankle2018lottery,
  title={The Lottery Ticket Hypothesis: Finding Sparse, Trainable Neural Networks},
  author={Frankle, Jonathan and Carbin, Michael},
  booktitle={International Conference on Learning Representations},
  year={2018}
}

@inproceedings{ramanujan2020s,
  title={What's hidden in a randomly weighted neural network?},
  author={Ramanujan, Vivek and Wortsman, Mitchell and Kembhavi, Aniruddha and Farhadi, Ali and Rastegari, Mohammad},
  booktitle={Proceedings of the IEEE/CVF conference on computer vision and pattern recognition},
  pages={11893--11902},
  year={2020}
}

@inproceedings{malach2020proving,
  title={Proving the lottery ticket hypothesis: Pruning is all you need},
  author={Malach, Eran and Yehudai, Gilad and Shalev-Schwartz, Shai and Shamir, Ohad},
  booktitle={International Conference on Machine Learning},
  pages={6682--6691},
  year={2020},
  organization={PMLR}
}

@article{orseau2020logarithmic,
  title={Logarithmic pruning is all you need},
  author={Orseau, Laurent and Hutter, Marcus and Rivasplata, Omar},
  journal={Advances in Neural Information Processing Systems},
  volume={33},
  pages={2925--2934},
  year={2020}
}

@article{pensia2020optimal,
  title={Optimal lottery tickets via subset sum: Logarithmic over-parameterization is sufficient},
  author={Pensia, Ankit and Rajput, Shashank and Nagle, Alliot and Vishwakarma, Harit and Papailiopoulos, Dimitris},
  journal={Advances in neural information processing systems},
  volume={33},
  pages={2599--2610},
  year={2020}
}

@inproceedings{ye2020good,
  title={Good subnetworks provably exist: Pruning via greedy forward selection},
  author={Ye, Mao and Gong, Chengyue and Nie, Lizhen and Zhou, Denny and Klivans, Adam and Liu, Qiang},
  booktitle={International Conference on Machine Learning},
  pages={10820--10830},
  year={2020},
  organization={PMLR}
}

@inproceedings{zhou2018non,
  title={Non-vacuous Generalization Bounds at the ImageNet Scale: a PAC-Bayesian Compression Approach},
  author={Zhou, Wenda and Veitch, Victor and Austern, Morgane and Adams, Ryan P and Orbanz, Peter},
  booktitle={International Conference on Learning Representations},
  year={2018}
}

@inproceedings{arora2018stronger,
  title={Stronger generalization bounds for deep nets via a compression approach},
  author={Arora, Sanjeev and Ge, Rong and Neyshabur, Behnam and Zhang, Yi},
  booktitle={International conference on machine learning},
  pages={254--263},
  year={2018},
  organization={PMLR}
}

@inproceedings{baykal2018data,
  title={Data-Dependent Coresets for Compressing Neural Networks with Applications to Generalization Bounds},
  author={Baykal, Cenk and Liebenwein, Lucas and Gilitschenski, Igor and Feldman, Dan and Rus, Daniela},
  booktitle={International Conference on Learning Representations},
  year={2018}
}

@article{liebenwein2019provable,
  title={Provable filter pruning for efficient neural networks},
  author={Liebenwein, Lucas and Baykal, Cenk and Lang, Harry and Feldman, Dan and Rus, Daniela},
  journal={arXiv preprint arXiv:1911.07412},
  year={2019}
}

@inproceedings{lym2019prunetrain,
  title={Prunetrain: fast neural network training by dynamic sparse model reconfiguration},
  author={Lym, Sangkug and Choukse, Esha and Zangeneh, Siavash and Wen, Wei and Sanghavi, Sujay and Erez, Mattan},
  booktitle={Proceedings of the International Conference for High Performance Computing, Networking, Storage and Analysis},
  pages={1--13},
  year={2019}
}

@inproceedings{
    lasby2023dynamic,
    title={Dynamic Sparse Training with Structured Sparsity},
    author={Mike Lasby and Anna Golubeva and Utku Evci and Mihai Nica and Yani Ioannou},
    booktitle={The Twelfth International Conference on Learning Representations},
    year={2024}
}

@inproceedings{lazarevich2021post,
  title={Post-training deep neural network pruning via layer-wise calibration},
  author={Lazarevich, Ivan and Kozlov, Alexander and Malinin, Nikita},
  booktitle={Proceedings of the IEEE/CVF international conference on computer vision},
  pages={798--805},
  year={2021}
}

@article{kwon2022fast,
  title={A fast post-training pruning framework for transformers},
  author={Kwon, Woosuk and Kim, Sehoon and Mahoney, Michael W and Hassoun, Joseph and Keutzer, Kurt and Gholami, Amir},
  journal={Advances in Neural Information Processing Systems},
  volume={35},
  pages={24101--24116},
  year={2022}
}

@inproceedings{zhang2024plug,
  title={Plug-and-play: An efficient post-training pruning method for large language models},
  author={Zhang, Yingtao and Bai, Haoli and Lin, Haokun and Zhao, Jialin and Hou, Lu and Cannistraci, Carlo Vittorio},
  booktitle={The Twelfth International Conference on Learning Representations},
  year={2024}
}

@article{bartlett2017spectrally,
  title={Spectrally-normalized margin bounds for neural networks},
  author={Bartlett, Peter L and Foster, Dylan J and Telgarsky, Matus J},
  journal={Advances in neural information processing systems},
  volume={30},
  year={2017}
}

@inproceedings{liuunreasonable,
  title={The Unreasonable Effectiveness of Random Pruning: Return of the Most Naive Baseline for Sparse Training},
  author={Liu, Shiwei and Chen, Tianlong and Chen, Xiaohan and Shen, Li and Mocanu, Decebal Constantin and Wang, Zhangyang and Pechenizkiy, Mykola},
  booktitle={International Conference on Learning Representations},
  year={2022}
}

@inproceedings{gadhikar2023random,
  title={Why random pruning is all we need to start sparse},
  author={Gadhikar, Advait Harshal and Mukherjee, Sohom and Burkholz, Rebekka},
  booktitle={International Conference on Machine Learning},
  pages={10542--10570},
  year={2023},
  organization={PMLR}
}

@inproceedings{chen2021adabert,
  title={AdaBERT: task-adaptive BERT compression with differentiable neural architecture search},
  author={Chen, Daoyuan and Li, Yaliang and Qiu, Minghui and Wang, Zhen and Li, Bofang and Ding, Bolin and Deng, Hongbo and Huang, Jun and Lin, Wei and Zhou, Jingren},
  booktitle={Proceedings of the Twenty-Ninth International Conference on International Joint Conferences on Artificial Intelligence},
  pages={2463--2469},
  year={2021}
}

@inproceedings{hinton2014distilling,
  title={Distilling the Knowledge in a Neural Network},
  author={Hinton, G},
  booktitle={Deep Learning and Representation Learning Workshop in Conjunction with NIPS},
  year={2014}
}

@inproceedings{ullrich2017soft,
  title={Soft Weight-Sharing for Neural Network Compression},
  author={Ullrich, Karen and Meeds, Edward and Welling, Max},
  booktitle={International Conference on Learning Representations},
  year={2017}
}

@inproceedings{wei2023joint,
  title={Joint token pruning and squeezing towards more aggressive compression of vision transformers},
  author={Wei, Siyuan and Ye, Tianzhu and Zhang, Shen and Tang, Yao and Liang, Jiajun},
  booktitle={Proceedings of the IEEE/CVF conference on computer vision and pattern recognition},
  pages={2092--2101},
  year={2023}
}

@article{lybrand2021greedy,
  title={A greedy algorithm for quantizing neural networks},
  author={Lybrand, Eric and Saab, Rayan},
  journal={Journal of Machine Learning Research},
  volume={22},
  number={156},
  pages={1--38},
  year={2021}
}

@article{zhang2024unified,
  title={Unified Stochastic Framework for Neural Network Quantization and Pruning},
  author={Zhang, Haoyu and Saab, Rayan},
  journal={arXiv preprint arXiv:2412.18184},
  year={2024}
}

@article{kuznedelev2023cap,
  title={Cap: Correlation-aware pruning for highly-accurate sparse vision models},
  author={Kuznedelev, Denis and Kurti{\'c}, Eldar and Frantar, Elias and Alistarh, Dan},
  journal={Advances in Neural Information Processing Systems},
  volume={36},
  pages={28805--28831},
  year={2023}
}

@article{zhang2023post,
  title={Post-training quantization for neural networks with provable guarantees},
  author={Zhang, Jinjie and Zhou, Yixuan and Saab, Rayan},
  journal={SIAM Journal on Mathematics of Data Science},
  volume={5},
  number={2},
  pages={373--399},
  year={2023},
  publisher={SIAM}
}

@article{he2023structured,
  title={Structured pruning for deep convolutional neural networks: A survey},
  author={He, Yang and Xiao, Lingao},
  journal={IEEE transactions on pattern analysis and machine intelligence},
  volume={46},
  number={5},
  pages={2900--2919},
  year={2023},
  publisher={IEEE}
}

@article{chatterjee2006generalization,
  title={A Generalization of the Lindeberg principle},
  author={Chatterjee, Sourav},
  journal={The Annals of Probability},
  volume={34},
  number={6},
  pages={2061--2076},
  year={2006}
}

@article{natale2024sparsity,
  title={On the sparsity of the strong lottery ticket hypothesis},
  author={Natale, Emanuele and Ferr{\'e}, Davide and Giambartolomei, Giordano and Giroire, Fr{\'e}d{\'e}ric and Mallmann-Trenn, Frederik},
  journal={Advances in Neural Information Processing Systems},
  volume={37},
  pages={40565--40592},
  year={2024}
}

@inproceedings{he2015delving,
  title={Delving deep into rectifiers: Surpassing human-level performance on imagenet classification},
  author={He, Kaiming and Zhang, Xiangyu and Ren, Shaoqing and Sun, Jian},
  booktitle={Proceedings of the IEEE international conference on computer vision},
  pages={1026--1034},
  year={2015}
}

\end{document}